\documentclass{article}

\usepackage[final]{neurips_2024}

\usepackage[utf8]{inputenc} 
\usepackage[T1]{fontenc}    
\usepackage{hyperref}       
\usepackage{url}            
\usepackage{booktabs}       
\usepackage{amsfonts}       
\usepackage{nicefrac}       
\usepackage{microtype}      
\usepackage{xcolor}         

\usepackage{amsmath, amsthm}
\usepackage{mathrsfs}
\usepackage{thmtools} 
\usepackage{thm-restate}
\usepackage{mathtools}
\usepackage{algorithm2e}
\usepackage{titletoc}
\usepackage{todonotes}
\usepackage{soul}

\newenvironment{pde}{\left\{\begin{array}{rll} } {\end{array}\right.}

\theoremstyle{plain}
\begingroup
\newtheorem{theorem}{Theorem}[]
\newtheorem*{theorem*}{Theorem}
\newtheorem*{"theorem"}{``Theorem''}

\newtheorem{lemma}[theorem]{Lemma}

\newtheorem*{theorem4}{Theorem \ref{theorem strongly convex}}

\endgroup

\theoremstyle{definition}
\begingroup

\endgroup

\theoremstyle{remark}
\begingroup
\newtheorem{remark}[theorem]{Remark}

\endgroup 

\DeclarePairedDelimiterX{\inp}[2]{\langle}{\rangle}{#1, #2}

\newcommand{\N}{\mathbb{N}}
\newcommand{\R}{\mathbb{R}}
\renewcommand{\L}{{\mathscr L}}
\let \Ra = \Rightarrow
\let \LRa = \Leftrightarrow
\DeclareMathOperator*{\argmin}{argmin} 
 
\newcommand{\F}{{\mathcal F}}
\newcommand{\E}{\mathbb E}
\newcommand{\df}{\nabla f(x'_n)}
\let\eps = \varepsilon
\newcommand{\Risk}{\mathcal{R}}
\renewcommand{\d}{\mathrm{d}}
\newcommand{\dt}{\d t}
\newcommand\norm[1]{\left\Vert#1\right\Vert}

\renewcommand{\P}{\mathbb{P}}

\RestyleAlgo{ruled}
\SetKwComment{Comment}{// }{ }
\SetKwInput{Input}{Input}
\SetKwInput{Return}{Return}

\hypersetup{
	colorlinks,
	linkcolor={blue!50!black},
	citecolor={blue!50!black},
	urlcolor={blue!80!black}
}

\title{Nesterov acceleration despite very noisy gradients}

\author{%
  Kanan Gupta\\
  Department of Mathematics\\
  University of Pittsburgh\\
  \texttt{kanan.g@pitt.edu}\\
  \And
  Jonathan W. Siegel \\
  Department of Mathematics\\
  Texas A\&M University\\
  \texttt{jwsiegel@tamu.edu} 
  \And
  Stephan Wojtowytsch \\
Department of Mathematics\\
  University of Pittsburgh\\
  \texttt{s.woj@pitt.edu} 
}

\begin{document}
\maketitle

\begin{abstract}
  We present a generalization of Nesterov's accelerated gradient descent algorithm. Our algorithm (AGNES) provably achieves acceleration for smooth convex and strongly convex minimization tasks with noisy gradient estimates if the noise intensity is proportional to the magnitude of the gradient at every point. Nesterov's method converges at an accelerated rate if the constant of proportionality is below $1$, while AGNES accommodates any signal-to-noise ratio. The noise model is motivated by applications in overparametrized machine learning. AGNES requires only two parameters in convex and three in strongly convex minimization tasks, improving on existing methods. We further provide clear geometric interpretations and heuristics for the choice of parameters.
\end{abstract}

\section{Introduction}
    
    The recent success of deep learning \citep{lecun2015deep} is built on stochastic first order optimization methods such as stochastic gradient descent \citep{lecun1998gradient} and ADAM \citep{kingma2014Adam}, which have enabled the large-scale training of neural networks. While such tasks are generally non-convex, accelerated first order methods for convex optimization have proved practically useful. Specifically, \cite{nesterov_original}'s accelerated gradient descent 
    has become a standard training method \citep{sutskever2013importance}.

    Modern neural networks tend to operate in the {\em overparametrized} regime, i.e.\ the number of model parameters exceeds the number of data points to be fit \citep{belkin2021fit}. In this setting, minibatch gradient estimates are exact (namely, exactly 0) on the set of global minimizers since data can be interpolated exactly. Motivated by such applications, \citet{vaswani2019fast} proved that \citet{nesterov2012efficiency}'s accelerated coordinate descent method (ACDM) achieves acceleration in (strongly) convex optimization with {\em multiplicative noise}, i.e.\ when assuming stochastic gradient estimates for which the noise intensity scales linearly with the magnitude of the gradient. Conversely, \citet{liu2018accelerating} show that the original version of \citet{nesterov_original}'s method generally does not achieve acceleration in this setting.
    
    Another algorithm with a similar goal is the continuized Nesterov method (CNM), which has been studied by \cite{even2021continuized, berthier2021continuized} in convex optimization (deterministic or with additive noise) and with multiplicative noise for overparametrized linear least squares regression. For a more extensive discussion of the context of our work in the literature, please see Section \ref{section lit review}.
    
    \citet{vaswani2019fast}'s algorithm is a four parameter scheme in the strongly convex case, which reduces to a three parameter scheme in the convex case. \citet{liu2018accelerating} introduce a simpler three parameter scheme, but only prove that it achieves acceleration for overparametrized {\em linear problems}. 
    In this work, we demonstrate that it is possible to achieve the same theoretical guarantees as \citet{vaswani2019fast} with  a simpler scheme, which can be considered as a reparametrized version of \citet{liu2018accelerating}'s Momentum-Added Stochastic Solver (MaSS) method. More precisely, we prove the following:

    \begin{enumerate}
        \item We show that Nesterov's accelerated gradient descent achieves an accelerated convergence rate, but \textit{only with noise which is {strictly} smaller than the gradient {in the $L^2$-sense}}. We also show numerically that when the noise is larger than the gradient, the algorithm diverges for a choice of step size for which gradient descent remains convergent.
        
        \item Motivated by this, we introduce a generalization of Nesterov’s method, which we call Accelerated Gradient descent with Noisy EStimators (AGNES), which provably achieves acceleration \textit{no matter how large the noise is relative to the gradient, both in the convex and strongly convex cases}.

        \item When moving from NAG to AGNES, the learning rate `bifurcates' to two parameters in order to accommodate stochastic gradient estimates. The extension 
        requires three hyperparameters in the strongly convex case and two in the convex case.

        \item We provide a transparent geometric interpretation of the AGNES parameters in terms of their scaling with problem parameters (Appendix \ref{appendix momentum parameters}) and the continuum limit models for various scaling  regimes (Appendix \ref{appendix continuous time}).

        \item We build strong intuition for the choice of hyperparameters for machine learning applications and empirically demonstrate that AGNES improves the training of CNNs relative to SGD with momentum and Nesterov's accelerated gradient descent.
    \end{enumerate}

\section{Literature Review}\label{section lit review}

     \paragraph{Accelerated first order methods.}
    Accelerated first order methods have been extensively studied in convex optimization. Beginning with the conjugate gradient (CG) algorithm introduced by \citet{hestenes1952methods}, the Heavy ball method of \citet{polyak1964some}, and \citet{nesterov_original}'s seminal work on accelerated gradient descent, many authors have developed and analyzed accelerated first order methods for convex problems, including \citet{beck2009fast,nesterov2012efficiency,nesterov2013gradient,chambolle2015convergence,kim2018another} to name just a few. 
    
    An important line of research is to gain an understanding of how accelerated methods work.
    After \citet{polyak1964some} derived the original Heavy ball method as a discretization of an ordinary differential equation, \citet{alvarez2002second,su2014differential,wibisono2016variational,zhang2018direct,siegel2019accelerated,shi2019acceleration,muehlebach2019dynamical,wilson2021lyapunov,shi2021understanding, suh2022continuous, attouch2022first, doi:10.1137/21M1403990, aujol2,dambrine2022stochastic} studied accelerated first order methods from the point of view of ODEs. {This perspective has facilitated the use of Lyapunov functional analysis to quantify the convergence properties. We remark that in addition to the intuition provided by differential equations, \cite{joulani2020simpler} and \cite{gasnikov2018universal} have also proposed interesting ideas for explaining and deriving accelerated first-order methods. In addition, there has been a large interest in deriving adaptive accelerated first order methods, see for instance \cite{levy2018online,cutkosky2019anytime,kavis2019unixgrad}.}

    \paragraph{Stochastic optimization.} \citet{robbins1951stochastic} first introduced optimization algorithms where gradients are only estimated by a stochastic oracle. For convex optimization, \citet{nemirovski2009robust,ghadimi2012optimal} obtained minimax-optimal convergence rates with additive stochastic noise.
    
    In deep learning, stochastic algorithms are ubiquitous in the training of deep neural networks, see \citep{lecun1998gradient,lecun2015deep,goodfellow2016deep,bottou2018optimization}. 
    Here, the additive noise assumption not usually appropriate. As \citet{wojtowytsch2021stochasticdiscrete, wualignment} show, the noise is of low rank and degenerates on the set of global minimizers. \citet{2019arXiv190704232S, 10.5555/3455716.3455953, DBLP:journals/corr/abs-1811-02564, gower2019sgd, damian2021label,wojtowytsch2021stochasticdiscrete, 10.5555/3495724.3497511} consider various non-standard noise models and \citep{wojtowytsch2021stochasticcontinuous, 10.5555/3495724.3497511, li2022what} study the continuous time limit of stochastic gradient descent. These include noise assumptions for degenerate noise due to \cite{DBLP:journals/corr/abs-1811-02564,  damian2021label, wojtowytsch2021stochasticdiscrete, wojtowytsch2021stochasticcontinuous}, low rank noise studied by \cite{damian2021label, li2022what} and noise with heavy tails explored by \cite{10.5555/3495724.3497511}. 

    \paragraph{Acceleration with stochastic gradients.} \citet{kidambi2018insufficiency} prove that there are situations in which it is impossible for any first order oracle method to improve upon SGD due to information-theoretic lower bounds. More generally, lower bounds in the stochastic first order oracle (SFO) model were presented by \cite{nemirovski2009robust} (see also \citep{ghadimi2012optimal}). 
    
    A partial improvement on the state of the art is given by \cite{jain2018accelerating}, who present an {accelerated stochastic gradient} method motivated by a particular low-dimensional and strongly convex problem. \citet{laborde2020lyapunov} obtain faster convergence of an accelerated method under an additive noise assumption by a Lyapunov function analysis. \cite{bollapragada2022fast} study an accelerated gradient method for the optimization of a strongly convex quadratic objective function with minibatch noise.

    Closest to our work are \cite{liu2018accelerating, vaswani2019fast, even2021continuized, berthier2021continuized} who study generalizations of Nesterov's method in stochastic optimization. \citet{liu2018accelerating, even2021continuized} obtain guarantees with noise of approximately multiplicative noise in overparametrized linear least squares problems and for general convex objective functions with additive noise and in deterministic optimization. \citet{vaswani2019fast} obtain comparable guarantees for the more complicated method of \citet{nesterov2013gradient}.

	\section{Algorithm and Convergence Guarantees} \label{section algorithm}
	
	\subsection{Assumptions}\label{section assumptions}
		
	In the remainder of this article, we consider the task of minimizing an objective function $f:\mathbb R ^m \rightarrow \mathbb R$ using stochastic gradient estimates $g$. We assume that $f$, $g$ and the initial condition $x_0$ satisfy:
    \begin{enumerate}
    \item The initial condition $x_0$ is a (potentially random) point such that $\E[f(x_0) + \|x_0\|^2]< \infty$.
    \item $f$ is $L-$smooth, i.e.\ $\nabla f$ is $L-$Lipschitz continuous with respect to the Euclidean norm.
    \item There exists a probability space $(\Omega,\mathcal A, \P)$ and a gradient estimator, i.e.\ a measurable function $g:\mathbb R^m\times \Omega\rightarrow \mathbb{R}^m$ such that for all $x\in\R^m$ the properties
    \begin{itemize}{
		\item $\E_\omega[g(x,\omega)] = \nabla f(x)$ (unbiased gradient oracle) and
		\item $\E_\omega\big[\Vert g(x,\omega)-\nabla f(x)\Vert^2  \big] \leq \sigma^2 \,\Vert\nabla f(x)\Vert^2$ (multiplicative noise scaling) hold.}
	\end{itemize}
	\end{enumerate}
    {A justification of the multiplicative noise scaling is given in Section \ref{section setting}. In the setting of machine learning, the space $\Omega$ is given by the random subsampling of the dataset.} A rigorous discussion of the probabilistic foundations is given in Appendix \ref{appendix auxiliary}.

    \subsection{Nesterov's Method with Multiplicative Noise}
     First we analyze \cite{nesterov_original}'s accelerated gradient descent algorithm {(NAG)} in the setting of multiplicative noise. {NAG} is given by {the initialization $x_0 = x_0'$ and the two-step iteration}
    \begin{align}
    \begin{split}
    x_{n+1} = x'_n - \eta g'_n, \qquad
    x'_{n+1} = x_{n+1} + \rho_n \big(x_{n+1}-x_n\big) = x_{n+1} + \rho_n (x_n' - \eta g'_n - x_n)
    \label{eq nesterov}
    \end{split}
    \end{align}
    where $g'_n = g(x'_n,\omega_n)$ and the variables $\omega_n$ are iid samples from the probability space $\Omega$, i.e.\ $g'_n$ is an unbiased estimate of $\nabla f(x'_n)$. We write $\rho$ instead of $\rho_n$ in cases where a dependence on $n$ is not required. We show that this scheme achieves an $O(1/n^2)$ convergence rate for convex functions but \textit{only in the case that $\sigma<1$}. To the best of our knowledge, this analysis is optimal.

    \begin{restatable}[{NAG}, convex case]{theorem}{nesterovconv}\label{theorem nesterov convex}
    Suppose that $x_n$ and $x'_n$ are generated by the time-stepping scheme (\ref{eq nesterov}), $f$ and $g$ satisfy the conditions laid out in Section \ref{section assumptions}, $f$ is convex, and $x^*$ is a point such that $f(x^*) = \inf_{x\in\R^m} f(x)$. If $\sigma < 1$ and the parameters are chosen such that
    \[ 0<\eta \leq \frac{1-\sigma^2}{L(1+\sigma^2)}, \quad \text{and} \quad \rho_n = \frac{n}{n+3},\qquad\text{
    then}\qquad
     \E[f(x_n) - f(x^*)] \leq  \frac{2\E[\norm{x_0 - x^*}^2]}{{\eta n^2}}.\] The expectation on the right hand side is over the random initialization $x_0$.
    \end{restatable}
    
    The proof of Theorem \ref{theorem nesterov convex} is given in Appendix \ref{appendix convex}. Note that the constant $1/\eta$ blows up as $\sigma\nearrow 1$ and the analysis yields no guarantees for $\sigma>1$. This mirrors numerical experiments in Section \ref{section numerical}.
    
    \begin{restatable}[{NAG}, strongly convex case]{theorem}{nesterovstrongly}\label{theorem nesterov strongly}
	In addition to the assumptions in Theorem \ref{theorem nesterov convex}, suppose that $f$ is $\mu$-strongly convex and the parameters are chosen such that
	\[ 0 < \eta \leq \frac{1-\sigma^2}{L(1+\sigma^2)} \; \text{and} \; \rho = \frac{1-\sqrt{\mu\eta}}{1+\sqrt{\mu\eta}},	\; \text{then} \; 
	 \E[f(x_n)-f(x^*)] \leq  2(1 - \sqrt{\mu\eta})^n \,\E \left[ f(x_0) - f(x^*) 
 \right]. \]
    \end{restatable}
    Just like in the convex case, the step size $\eta$ decreases to zero as $\sigma \nearrow 1$, and we fail to obtain convergence guarantees for $\sigma\geq 1$. We argue in the proof of Theorem \ref{theorem nesterov strongly}, given in Appendix \ref{appendix strongly convex}, that it is not possible to modify the Lyapunov sequence analysis to obtain a better rate of convergence. This motivates our introduction of the more general AGNES method below.

    Notably, there cannot be a diverging lower bound for NAG since gradient descent arises in the special case $\rho=0$, and gradient descent converges for small stepsize with multiplicative noise \citep{wojtowytsch2021stochasticdiscrete}. On the other hand, \citet{liu2018accelerating} show that NAG does not achieve accelerated convergence with multiplicative type noise even for quadratic strongly convex functions.

	\subsection{AGNES Descent algorithm}\label{descent}
	
	The proofs of Theorems \ref{theorem nesterov convex} and \ref{theorem nesterov strongly} suggest that the momentum step in (\ref{eq nesterov}) is quite sensitive to the step size used for the gradient step, which severely restricts the step size $\eta$.  We propose the Accelerated Gradient descent with Noisy EStimators (AGNES) scheme, which addresses this problem by introducing an additional parameter $\alpha$ in the momentum step:
	\begin{align}
        \begin{split} \label{eq agnesterov}
		x_0 &= x'_0, \qquad x_{n+1} = x'_n - \eta g'_n, \qquad
        x'_{n+1} = x_{n+1} + \rho_n \big(x'_n -\alpha g'_n -x_n\big),
        \end{split}
	\end{align} 
	where $g'_n = g(x'_n,\omega_n)$ as before. {Equivalently, AGNES can be formulated as a three-step scheme with an auxiliary {velocity} variable $v_n$}, initialized as $v_0=0$: 
	 \begin{align}
    \begin{split}
		x_n' = x_n +\alpha v_n, \qquad
		x_{n+1} = x_n' - \eta g'_n, \qquad
		v_{n+1} = \rho_n(v_n - g'_n).
    \end{split} \label{eq agnes}
	\end{align} 
	We show that the two formulations of AGNES are equivalent in Appendix \ref{appendix equiv}. However, we find \eqref{eq agnes} more intuitive (see Appendix \ref{appendix continuous time} for a continuous time interpretation) and easier to implement as an algorithm without storing past values of $x_n$. The pseudocode and a set of suggested {default} parameters are given in Algorithm \ref{descent algorithm}.

\begin{algorithm}
\caption{Accelerated Gradient descent with Noisy EStimators (AGNES)}\label{descent algorithm}
\Input{$f$ (objective{/loss} function), $x_0$ (initial point), $\alpha=10^{-3}$ (learning rate), ${\eta=10^{-2}}$ (correction step size), $\rho=0.99$ ({momentum}), $N$ (number of iterations)}
$n \gets 0$ \\
$v_0 \gets 0$ \\
\While{$n < N$}{
	$g_{n} \gets \nabla_xf(x_{n})$ \DontPrintSemicolon \Comment*[r]{gradient estimate}
	$v_{n+1} \gets \rho(v_{n} - g_{n})$ \\
	$x_{n+1} \gets x_{n} +\alpha v_{n+1} - \eta g_{n}$ \\
	$n \gets n+1$ }
$g_{N} \gets \nabla_x f(x_N)$\\
$x_N \gets x_N - \eta g_n$\\
\Return{$x_N$}
\end{algorithm}
 
	From \eqref{eq nesterov} and \eqref{eq agnesterov}, we note that NAG is AGNES with the special choice $\alpha = \eta$. Allowing $\alpha$ and $\eta$ to be different helps AGNES achieve an accelerated rate of convergence for both convex and strongly convex functions, no matter how large $\sigma$ is.
	While for gradient descent, only the product $L(1+\sigma^2)$ has to be considered, this is not the case for momentum-based schemes. We consider first the convergence rate in the convex case.
    \begin{restatable}[AGNES, convex case]{theorem}{convex}\label{theorem convex}
	Suppose that $x_n$ and $x'_n$ are generated by the time-stepping scheme (\ref{eq agnes}), $f$ and $g'_n = g(x'_n, \omega_n)$ satisfy the conditions laid out in Section \ref{section assumptions}, $f$ is convex, and $x^*$ is a point such that $f(x^*) = \inf_{x\in\R^m} f(x)$. If the parameters are chosen such that
	\begin{align*}
	0<\eta < \frac1{L(1+\sigma^2)}, \quad \alpha = \frac{\eta}{1+\sigma^2},  \quad \rho_n = \frac{n}{n+1+a_0}, \quad
	\text{for} \quad a_0 \geq \frac{2(1-\eta L)}{1-\eta L(1+\sigma^2)}, \quad \text{then} \end{align*}
	\[
	\E\big[ f(x_n) - f(x^*)\big] \leq \frac{a_0^2\,\E\big[\,\|x_0-x^*\|^2\big]}{2\,\alpha\,n^2}.
	\]
     \end{restatable}	

In particular, if $\eta \leq \frac1{L(1+2\sigma^2)}$, we may make the universal choice $a_0 = 4$, i.e.\ $\rho_n = \frac{n}{n+5}$. Only the parameters $\eta, \alpha$ depend on the specific problem.
 The proof of Theorem \ref{theorem convex} is given in Appendix \ref{appendix convex}. There, we also present an alternative version of Theorem \ref{theorem convex} for a different choice of parameters 
\[
\eta \leq \frac1{L(1+\sigma^2)}, \quad \alpha < \frac\eta{1+\sigma^2}, \quad \rho_n = \frac{n+n_0}{n+n_0+3}
\]
for a potentially large $n_0 \geq \frac{2\eta\sigma^2}{\eta - \alpha(1+\sigma^2)}\geq  2\sigma^2$. The convergence guarantees are similar in both cases.

The benefit of the accelerated scheme is an improvement from a decay rate of $O(1/n)$ to the rate $O(1/n^2)$, which is optimal under the given assumptions even in the deterministic case.
While the noise can be orders of magnitude larger than the quantity we want to estimate, it only affects the constants in the convergence, not the rate. We get an analogous result for strongly convex functions.

\begin{restatable}[AGNES, strongly convex case]{theorem}{agnesstrongly}\label{theorem strongly convex}
	In addition to the assumptions in Theorem \ref{theorem convex}, suppose that $f$ is $\mu$-strongly convex and the parameters are chosen such that
	\begin{align*}
        0<\eta \leq \frac{1}{L(1+\sigma^2)}, \qquad
        \rho = \frac{1 - \sqrt{\frac{\mu\eta}{1+\sigma^2}}} {1 + \sqrt{\frac{\mu\eta}{1+\sigma^2}}}, \qquad
        \text{and} \qquad
        \alpha = \frac{1-\sqrt{\frac\mu L}}{1-\sqrt{\frac\mu L}+\sigma^2}\,\eta \qquad \text{then}
        \end{align*}
	\[ \E[f(x_n) - f(x^*)] \leq 2\left(1 - \sqrt{\frac{\mu\eta}{1+\sigma^2}}\right)^n \E[f(x_0) - f(x^*)].
    \]
\end{restatable}
{Choosing $\eta$ too small can be interpreted as overestimating $L$ or $\sigma$. Choosing $\alpha$ too small (with respect to $\eta$) can be interpreted as overestimating $\sigma$. Since every $L$-Lipschitz function is $L'$-Lipschitz for $L'>L$, and since the multiplicative noise bound with constant $\sigma$ implies the same bound with $\sigma'>\sigma$, exponential convergence still holds at a generally slower rate.}

We note that since $|\nabla f(x) |^2 \leq 2L (f(x) - \inf f)$ (Lemma \ref{lemma mu-L bound} in Appendix \ref{appendix auxiliary}), Theorems \ref{theorem convex} and \ref{theorem strongly convex} lead to analogous convergence results for $\mathbb E[\nabla f(x_n)]$ as well. Due to the summability of the sequences $n^{-2}$ and {$r^n$ for $r<1$}, 
we get not only convergence in expectation but also almost sure convergence. The proof is given in Appendix \ref{appendix convex}.

\begin{restatable}{corollary}{almostsurely}\label{corollary almost surely}
	In the setting of Theorems \ref{theorem convex} and \ref{theorem strongly convex}, $f(x_n) \to \inf f$ with probability 1.
\end{restatable}

    In the deterministic case $\sigma = 0$, we have $\alpha = \eta$ in both Theorems \ref{theorem convex} and \ref{theorem strongly convex}. In Theorem \ref{theorem strongly convex}, the parameters coincide with the usual choice for NAG, while we opted for a simple statement in Theorem \ref{theorem convex} which does not exactly recover the standard choice $\eta = 1/L$ and $\rho_n = n/(n+3)$. The proofs below easily cover these special cases as well.
  {If $0<\sigma<1$, both AGNES and NAG converge with the same rate $n^{-2}$ in the convex case, but the constant of NAG is always larger. In the strongly convex case, even the decay rate of NAG is slower than AGNES for $\sigma\in(0,1)$ since $1-\sigma^2 < (1+\sigma^2)^{-1}$.} We see the real power of AGNES in the stochastic setting where it converges for very high values of $\sigma$ when Nesterov's method may diverge. For the optimal choice of parameters, we summarize the results in terms of the time-complexity of SGD and AGNES in Figure \ref{figure time complexity}. For the related guarantee for SGD, see Theorems \ref{theorem convex sgd} and \ref{theorem gd strongly convex} in Appendices \ref{appendix convex} and \ref{appendix strongly convex} respectively.
 
    \begin{remark}[Batching]
    Let us compare AGNES with two families of gradient estimators:
    \begin{enumerate}
        \item $g'_n = g(x_n', \omega_n)$ as studied in Theorems \ref{theorem convex} and \ref{theorem strongly convex}.
        \item A gradient estimator $g'_n:= \frac1{n_b} \sum_{j=1}^{n_b} g(x_n', \omega_{n, j})$ which averages multiple independent estimates to reduce the variance. 
    \end{enumerate}
    The second gradient estimator falls into the same framework with $\tilde \Omega = \Omega^{n_b}$ and $\tilde \sigma^2 = \sigma^2/n_b$. {Assuming vector additions cost negligible time, optimizer steps are only as expensive as gradient evaluations.} In this setting -- which is often realistic in deep learning -- it is appropriate to compare $\E[f(x_{n_bn})]$ ($n_b\cdot n$ iterations using $g'_n$) and $\E[f(X_n)]$ ($n$ iterations with $g'_n$). 
    For the strongly convex case, we note that $\left(1- \sqrt\frac\mu L\,\frac1{1+\sigma^2}\right)^{n_b} \leq 1- \sqrt\frac\mu L\,\frac1{1+\sigma^2/n_b}$ if and only if
    \begin{align*}
        n_b \geq \frac{\log\left(1- \sqrt\frac\mu L\,\frac1{1+\sigma^2/n_b}\right)}{\log\left(1- \sqrt\frac\mu L\,\frac1{1+\sigma^2}\right)}
        \approx \frac{\sqrt\frac\mu L\,\frac{1}{1+\sigma^2/n_b}}{\sqrt\frac\mu L\,\frac{1}{1+\sigma^2}}
        = \frac{1+\sigma^2}{1+\sigma^2/n_b} =\frac{1+\sigma^2}{n_b+\sigma^2}\,n_b.
    \end{align*}
    The approximation is well-justified in the important case that $\mu\ll L$. In particular, the upper bound for non-batching AGNES is {\em always} favorable compared to the batching version as $n_b \in\N_{\geq 1}$, and the two only match for the optimal batch size $n_b=1$. 
    {The optimal batch size for minimizing $f$ is the largest one that can be processed in parallel without increasing the computing time for a single step.} A similar argument holds for the convex case.
    \end{remark}
    With a slight modification, the proof of Theorem \ref{theorem convex} extends to the situation of convex objective functions which do not have minimizers. Such objectives arise for example in linear classification with the popular cross-entropy loss function and linearly separable data.
	
    \begin{restatable}[Convexity without minimizers] {theorem} {weaklyconvex}\label{theorem weakly convex}
    Let $f$ be a convex objective function satisfying the assumptions in Section \ref{section assumptions} and $x_n$ be generated by the time-stepping scheme \eqref{eq agnes}. Assume that $\eta, \alpha$ and $\rho_n$ are as in Theorem \ref{theorem convex}. Then $\liminf_{n\to\infty} \E[f(x_n)] = \inf_{x\in\R^m}f(x).$
    \end{restatable}

	The proof and more details are given in Appendix \ref{appendix convex}.
    For completeness, we consider the case of non-convex optimization in Appendix \ref{appendix non-convex}. 
	As a limitation, we note that multiplicative noise is well-motivated in machine learning for global minimizers, but not at generic critical points.

    \begin{figure}
    \begin{center}
        \centering\renewcommand{\arraystretch}{2.0}
        \begin{tabular}{r| c|c}
         Time complexity & Convex & $\mu$-strongly convex\\ \hline
         {SGD} & $O\left(\frac{L}\eps(1+\sigma^2)\right)$ & $O\left(\frac {L}\mu\,(1+\sigma^2) |\log \eps|\right)$ \\ \hline
         {AGNES} & $O\left( \sqrt{\frac{L}\eps}(1+\sigma^2) \right)$ &   $O\left(\sqrt{\frac {L}\mu}\,(1+\sigma^2) |\log \eps|\right)$
        \end{tabular}
    \end{center}
        \caption{
        The minimal $n$ for AGNES and SGD such that $\E[f(x_n) - \inf f]<\eps$ when minimizing an $L$-smooth function with multiplicative noise intensity $\sigma$ in the gradient estimates and under a convexity assumption. The SGD rate of the $\mu$-strongly convex case is achieved more generally under a PL condition with PL-constant $\mu$. 
        While SGD requires the optimal choice of one variable to achieve the optimal rate, AGNES requires three (two in the determinstic case).
        }
        \label{figure time complexity}
    \end{figure}

    \subsection{Geometric Interpretation}

    Let us briefly discuss the parameter choices in Theorem \ref{theorem strongly convex}. As we consider larger $\sigma$ for fixed $\mu$ and $L$, the decay factor $\rho$ moves closer to $1$. This slows the `forgetting' of past gradients in $v_n$, allowing us to better average out stochastic noise. The price we pay is computing with more outdated gradients, slowing convergence. Our choice balances these effects.
    
    In AGNES, $\rho$ inadvertently also governs magnitude of the momentum variable $v_n$, which scales as $(1-\rho)^{-1}$ for objective functions with constant gradient and $n\gg1$. To compensate, we choose $\alpha$ smaller compared to $\eta$ when $\sigma$ (and thus $(1-\rho)^{-1}$) is large. Nevertheless, the effect of the momentum step does not decrease. For further details, see Appendix \ref{appendix momentum parameters}.
    
    For further interpretability, we obtain several ODE and SDE continuous time descriptions of AGNES in Appendix \ref{appendix continuous time}. 

    \section{Motivation for Multiplicative Noise}\label{section setting}
    \begin{figure*}
        \centering
        \includegraphics[width = .42\textwidth]{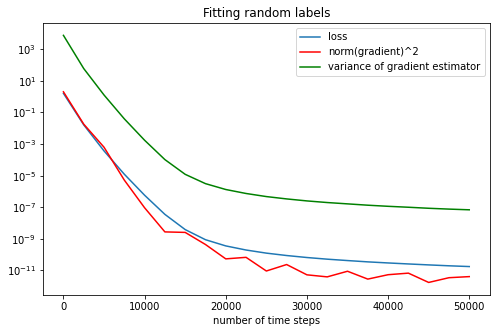}\hspace{2mm}
        \includegraphics[width = .42\textwidth]{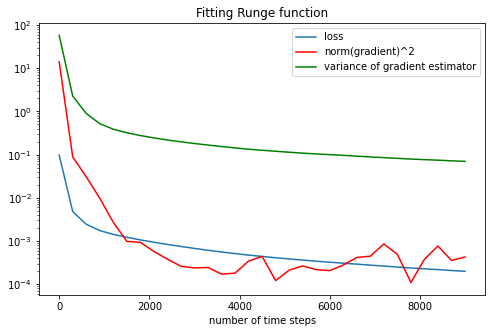}\hfill
        \caption{To be able to quantify the gradient noise exactly, we choose relatively small models and data sets. {\bf Left:} A ReLU network with four hidden layers of width 250 is trained by SGD to fit random labels $y_i$ (drawn from a 2-dimensional standard Gaussian) at $1,000$ random data points $x_i$ (drawn from a 500-dimensional standard Gaussian). The variance $\sigma^2$ of the gradient estimators is $\sim 10^5$ times larger than the loss function and $\sim 10^6$ times larger than the parameter gradient. This relationship is stable over approximately ten orders of magnitude. {\bf {Right}:} A ReLU network with two hidden layers of width 50 is trained by SGD to fit the Runge function $1/(1+x^2)$ on equispaced data samples in the interval $[-8,8]$. Also here, the variance in the gradient estimates is proportional to both the loss function and the magnitude of the gradient.}
        \label{figure noise scaling}
    \end{figure*}
    In supervised learning applications, the learning task often corresponds to minimizing a risk or loss function $\Risk(w) = {\frac1N} \sum_{i=1}^{N} \ell\big(h(w,x_i), y_i\big) =: {\frac1N}\sum_{i=1}^{N}\ell_i(w)$, where
    $
    h:\R^m \times \R^d \to \R^k$, $(w,x) \mapsto h(w,x)$ and $\ell:\R^k\times\R^k \to [0,\infty)$
    are a parametrized function of weights $w$ and data $x$ and a loss function measuring compliance between $h(w,x_i)$ and $y_i$ respectively.\footnote{\ Both $\ell$ and $\Risk$ are commonly called a `loss function' in the literature. To distinguish between the two, we will borrow the terminology of statistics and refer to $\Risk$ as the risk functional and $\ell$ as the loss function. The notation $L$, which is often used in place of $\Risk$, is reserved for the Lipschitz constant in this work.}\ \cite{DBLP:conf/icml/SafranS18,NEURIPS2018_a1afc58c, du2018gradient} show that working in the overparametrized regime $m\gg N$ simplifies the optimization process and \cite{belkin2019reconciling, belkin2020two} illustrate that it facilitates generalization to previously unseen data. \cite{cooper2019the} shows that fitting ${N}$ constraints with $m$ parameters typically leads to an ${m-N}$-dimensional submanifold $\mathcal M$ of the parameter space $\R^m$ such that all given labels $y_i$ are fit exactly by $h(w,\cdot)$ at the data points $x_i$ for $w\in\mathcal M$, i.e.\ $\Risk\equiv 0$ on the smooth set of minimizers $\mathcal M = \Risk^{-1}(\{0\})$.

    If $N$ is large, it is computationally expensive to evaluate the gradient {$\nabla\Risk(w) = \frac1N \sum_{i=1}^N \nabla \ell_i$} of the risk function $\Risk$ exactly and we commonly resort to stochastic  estimates
    \begin{align*}
    g  &= \frac1{n_b} \sum_{i\in I_b} \nabla \ell_i(w)= \frac1{n_b} \sum_{i\in I_b} \sum_{j=1}^{k} (\partial_{h_j}\ell)\big(h(w,x_i), y_i\big)\, \nabla_w h_j(w,x_i),
    \end{align*}
    where $I_b\subseteq \{1,\dots,{N}\}$ is a subsampled collection of $n_b$ data points (a batch or minibatch).
    Minibatch gradient estimates are very different from the stochasticity we encounter e.g.\ in statistical mechanics:
    \begin{enumerate}
        \item The covariance matrix $\Sigma = \frac1{{N}}\sum_{i=1}^{N}\big(\nabla \ell_i - \nabla \Risk) \otimes (\nabla \ell_i-\nabla \Risk)$ of the gradient estimators $\nabla \ell_i$ has low rank $N\ll m$.

        \item Assume specifically that $\ell$ is a loss function which satisfies $\ell(y,y) = 0$ for all $y\in\R^k$, such as the popular $\ell^2$-loss function $\ell(h,y) = \|h-y\|^2$. Then $\nabla \ell_i(w) =0$ for all $i \in \{1,\dots, {N}\}$ and all $w\in \mathcal M= \Risk^{-1}(0)$. In particular, minibatch gradient estimates are exact on $\mathcal M$.
    \end{enumerate}

    The following Lemma makes the second observation precise in the overparameterized regime and bounds the stochasticity of mini-batch estimates more generally.

	\begin{restatable}[Noise intensity]{lemma}{noise}\label{lemma noise scaling}
        Assume that $\ell(h,y) = \|h-y\|^2$ and ${h}:\R^m\times\R^d\to\R^k$ satisfies 
        $
        \|\nabla_w h(w,x_i)\|^2 \leq C\big(1+\|w\|\big)^p
        $
        for some $C, p>0$ and all $w\in \R^m$ and $i= 1,\dots, {N}$. Then for all $w\in\R^m$:
        \[
        \frac1{{N}} \sum_{i=1}^{N} \big\|\nabla \ell_i - \nabla \Risk\big\|^2 \:\leq\: 4C^2\,(1+\|w\|)^{2p}\:\Risk(w).
        \]
    \end{restatable}
    
     Lemma \ref{lemma noise scaling} is proved in Appendix \ref{appendix noise}. It is a modification of \cite[Lemma 2.14]{wojtowytsch2021stochasticdiscrete} for function models which are locally, but not globally Lipschitz-continuous in the weights $w$, such as deep neural networks with smooth activation function. The exponent $p$ may scale with network depth.

    Lemma \ref{lemma noise scaling} describes the variance of a gradient estimator which uses a random index $i\in \{1,\dots, {N}\}$ and the associated gradient $\nabla \ell_i$ is used to approximate $\nabla \Risk$. If a batch $I_b$ of $n_b$ indices is selected randomly with replacement, then the variance of the estimates scales in the usual way:
    \begin{equation}\label{eq noise scaling}
    \E_{I_b} \left[\left\|\frac1{n_b} \sum_{i\in I_b} \nabla \ell_i - \nabla \Risk\right\|^2\right] \leq \frac{4C^2(1+ \|w\|)^{2p}}{n_b}\,\Risk(w).
    \end{equation}

    As noted by \cite{wu2019global, wu2022adaloss}, $\Risk$ and $\|\nabla\Risk\|^2$ often behave similarly in overparametrized deep learning. We illustrate this in Figure \ref{figure noise scaling} together with Lemma \ref{lemma noise scaling}. Heuristically, we therefore replaced \eqref{eq noise scaling} by a more manageable assumption akin to $\E[\frac1N\sum_{i=1}^N\|\nabla \ell_i - \nabla \Risk\|^2] \leq \sigma^2\|\nabla \Risk\|^2$ in Section \ref{section assumptions}. The setting where the signal-to-noise ratio (the quotient of estimate variance and true magnitude) is $\Omega(1)$ is often referred to as `multiplicative noise', as it resembles the noise generated by estimates of the form $g = (1+\sigma Z)\nabla \Risk$, where $Z\sim \mathcal N(0,1)$. When the objective function is $L$-smooth and satisfies a PL condition (see e.g.\ \citep{karimi2016linear}), both scaling assumptions are equivalent.
    	
	\section{Numerical Experiments}\label{section numerical}

        \subsection{Convex optimization}

            \begin{figure}
    \begin{center}
    \includegraphics[width = 0.4\textwidth]{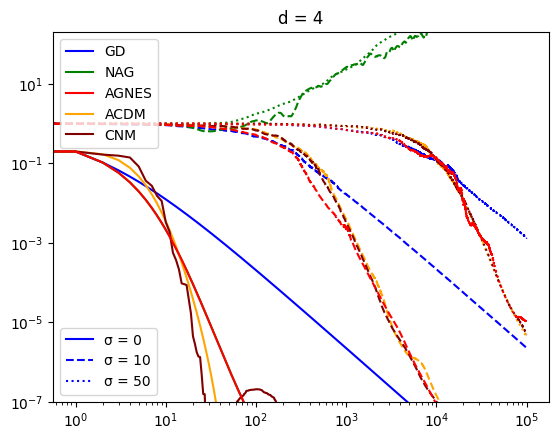}
    \includegraphics[width = 0.4\textwidth]{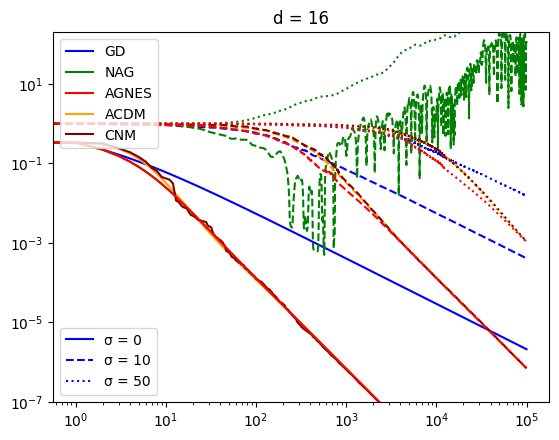}
    \end{center}    

    \caption{\label{figure convex}
    We plot $\E\big[f_{d}(x_n)\big]$ on a loglog scale for SGD (blue), AGNES (red), NAG (green), ACDM (orange) and CNM (maroon) with $d=4$ (left) and $d=16$ (right) for noise levels $\sigma=0$ (solid line), $\sigma=10$ (dashed) and $\sigma =50$ (dotted). The initial condition is $x_0=1$ in all simulations. Means are computed over 200 runs.
    After an initial plateau, AGNES, CNM and ACDM significantly outperform SGD in all settings, while NAG (green) diverges if $\sigma$ is large. 
    The length of the initial plateau increases with $\sigma$. 
    }
    \end{figure}

    We compare the optimization algorithms for the family of objective functions
    \[
    f_{d}:\R\to\R, \qquad f_{d}(x) = \begin{cases} |x|^{d} & \text{if }|x|<1\\ 1 + {d}(|x|-1) &\text{else}\end{cases}
    \]
    for $d\geq 2$ with gradient estimators $g = (1+ \sigma N) f'(x)$, where $N$ is a unit normal random variable. The functions are convex and their derivatives are Lipschitz-continuous with $L=d(d-1)$. Various trajectories are compared for different values of ${d}$ and $\sigma$ in Figure \ref{figure convex}.
    We run AGNES with the parameters $\alpha = \frac{\eta}{1+\sigma^2}$, $\eta=\frac{1}{L(1+2\sigma^2)}$, $\rho_n=\frac{n}{n+5}$ derived above and SGD with the optimal step size $\eta = \frac{1}{L(1+\sigma^2)}$ (see Lemmas \ref{lemma gd decrease} and \ref{theorem convex sgd}). For NAG, we select $\eta = \frac1{L(1+\sigma^2)}$ and $\rho_n = \frac n{n+3}$. We present a similar experiment in the strongly convex case in Appendix \ref{appendix simulations}.

    We additionally compare to two other methods of accelerated gradient descent which were recently proposed for multiplicative noise models: The ACDM method of \cite{nesterov2012efficiency, vaswani2019fast}, and the continuized Nesterov method (CNM) of \cite{even2021continuized, berthier2021continuized} with the proposed parameters. In this simple setting where all constants are known, AGNES, ACDM and CNM perform comparably in the long run and on average.

\subsection{Neural network regression}\label{section regression} 

We generated $n= 100,000$ 12-dimensional random vectors. Using a fixed, randomly initialized neural network $f^*$ (with 10 hidden layers, each with width 10, and output dimension 1), we produced labels $y_i=f^*(x_i)$. The resulting dataset was split into 90\% training and 10\% testing data. We then trained identically initialized copies of a larger neural network (15 hidden layers, each with width 15) using Adam, NAG, SGD with momentum, and AGNES to minimize the mean-squared error (MSE) loss. 

We selected the learning rate \(10^{-3}\) for Adam as it performed poorly at higher or lower rates \(10^{-2}\) and \(10^{-4}\). For AGNES, NAG, and SGD, based on initial exploratory experiments, we used a learning rate of \(10^{-4}\), a momentum value of 0.99, and for AGNES, a correction step size \(\eta = 10^{-3}\). The experiment was repeated 10 times each for batch sizes 100, 50, and 10, and run for 45,000 optimizer steps each time.
The average loss and standard deviation for each algorithm are reported in Figure \ref{figure regression}. The results show that AGNES performs better than SGD and NAG for all batch sizes. With large batch size, Adam performs well with default hyperparameters. The performance of AGNES relative to other algorithms especially improves as the batch size decreases.

\begin{figure}
    \centering
    \includegraphics[clip = true, trim = 10mm 8mm 15mm 8mm, width = 0.32\textwidth]{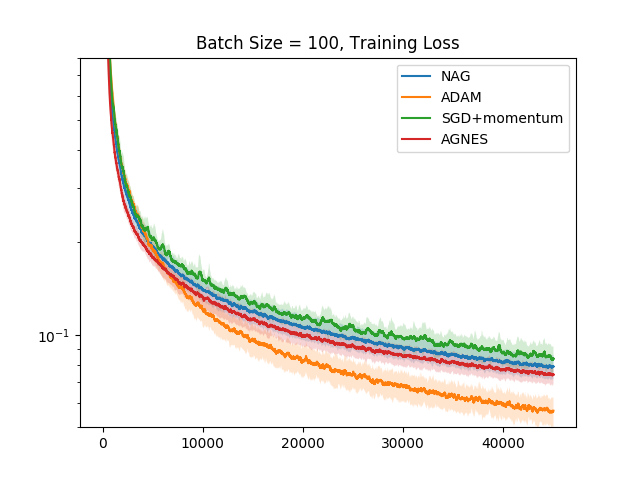}
    \includegraphics[clip = true, trim = 10mm 8mm 15mm 8mm, width = 0.32\textwidth]{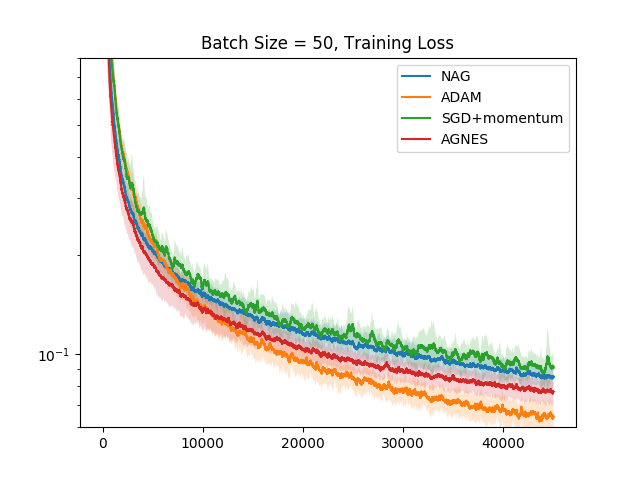}
    \includegraphics[clip = true, trim = 10mm 8mm 15mm 8mm, width = 0.32\textwidth]{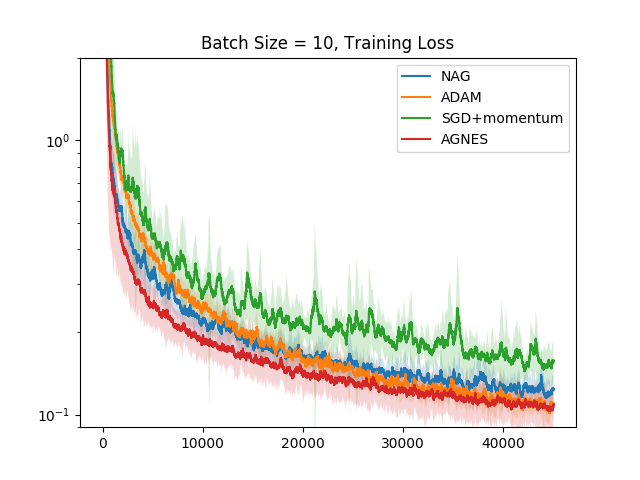}
    \includegraphics[clip = true, trim = 10mm 8mm 15mm 8mm, width = 0.32\textwidth]{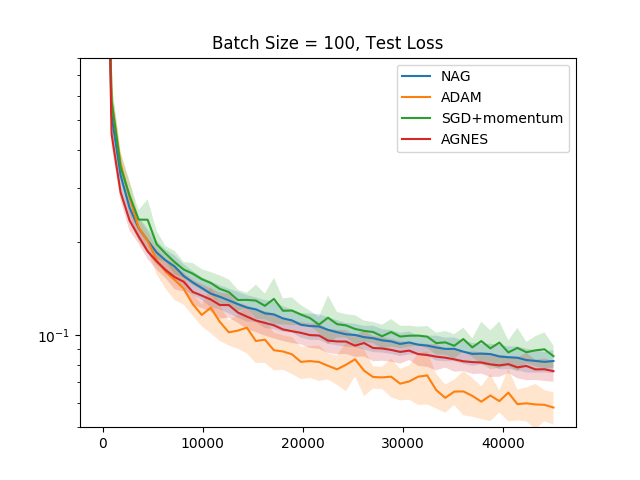}
    \includegraphics[clip = true, trim = 10mm 8mm 15mm 8mm, width = 0.32\textwidth]{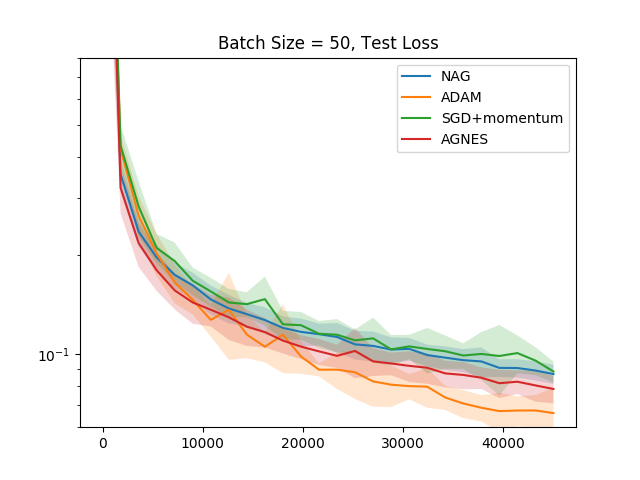}
    \includegraphics[clip = true, trim = 10mm 8mm 15mm 8mm, width = 0.32\textwidth]{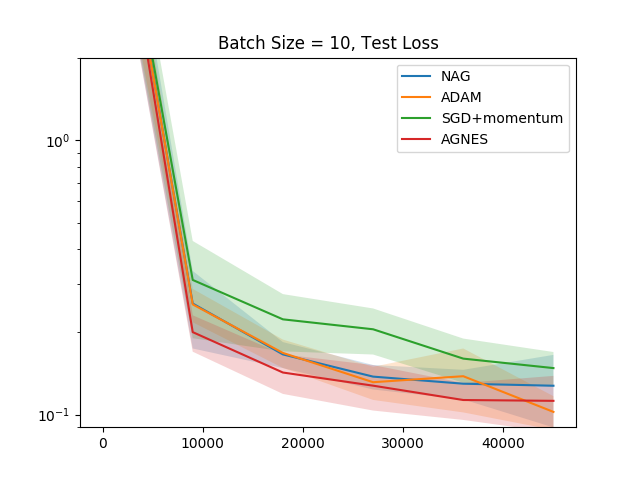}
    \caption{We report the training loss as a running average with decay rate 0.99 (top row) and test loss (bottom row) for batch sizes 100 (left column), 50 (middle column), and 10 (right column) in the setting of Section \ref{section regression}. The horizontal axis represents the number of optimizer steps. The performance gap between AGNES and other algorithms widens for smaller batch sizes, where the gradient estimates are more stochastic and the two different parameters $\alpha, \eta$ add the most benefit.}
    \label{figure regression}
\end{figure}
    
\subsection{Image classification}\label{section image classification}

    \begin{figure}
    \centering
    \includegraphics[clip = true, trim = 10mm 8mm 15mm 8mm, width = 0.32\textwidth]{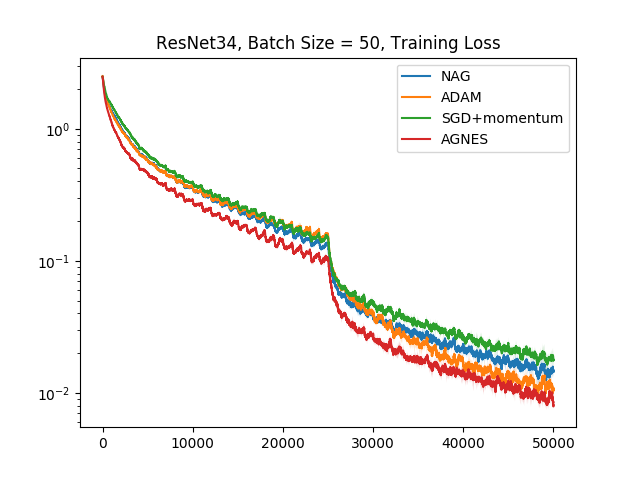}
    \includegraphics[clip = true, trim = 10mm 8mm 15mm 8mm, width = 0.32\textwidth]{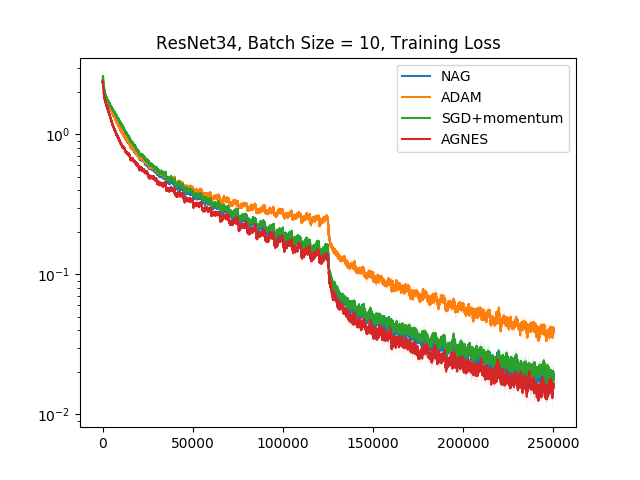}
    \includegraphics[clip = true, trim = 10mm 8mm 15mm 8mm, width = 0.32\textwidth]{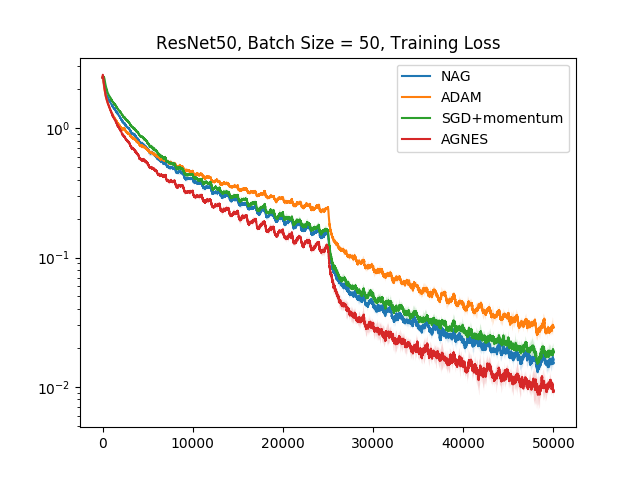}
    \includegraphics[clip = true, trim = 10mm 8mm 15mm 8mm, width = 0.32\textwidth]{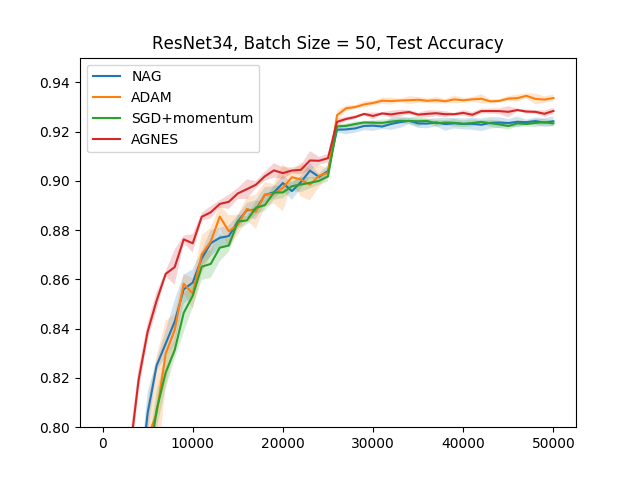}
    \includegraphics[clip = true, trim = 10mm 8mm 15mm 8mm, width = 0.32\textwidth]{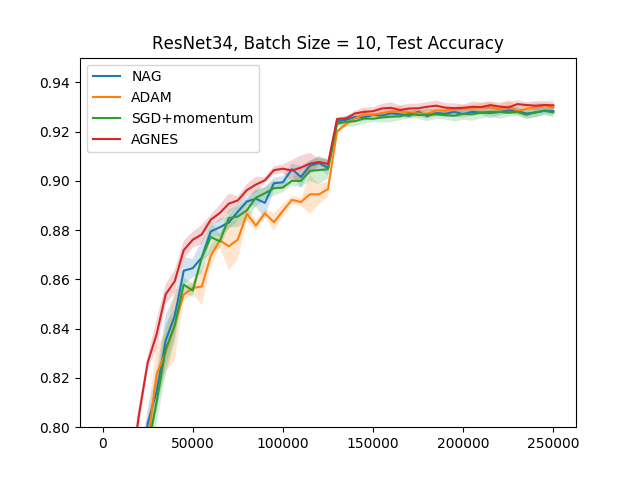}
    \includegraphics[clip = true, trim = 10mm 8mm 15mm 8mm, width = 0.32\textwidth]{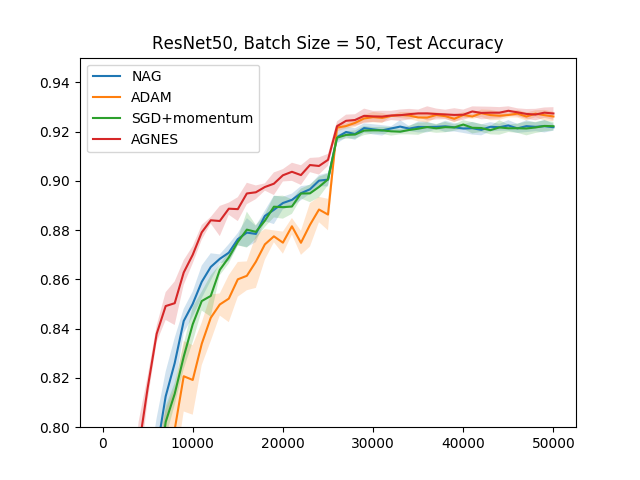}
    \caption{We report the training loss as a running average with decay rate 0.99 (top row) and test accuracy (bottom row) for ResNet-34 trained on CIFAR-10 with batch sizes 50 (left column) and 10 (middle column), and ResNet-50 trained with batch size 50 (right column). The performance of AGNES with the proposed hyperparameters is stable over the changes in model and batch size.}
    \label{figure cnn}
\end{figure}

    We trained ResNet-34 \citep{he2016deep} with batch sizes 50 and 10, and ResNet-50 with batch size 50 on the CIFAR-10 image dataset \citep{krizhevsky2009learning} with standard data augmentation (normalization, random crop, and random flip) using Adam, SGD with momentum, NAG, and AGNES. The model implementations were based on \citep{kuangliu}. Each algorithm was provided an identically initialized model and the experiment was repeated 5 times for 50 epochs each. The averages and standard deviations of training loss and test accuracy are reported in Figure \ref{figure cnn}. We used the same initial learning rate $10^{-3}$ for all the algorithms, which was dropped to $10^{-4}$ after 25 epochs. A momentum value of 0.99 was used for SGD, NAG, and AGNES and a constant correction step size $\eta = 10^{-2}$ was used for AGNES.
    
    AGNES reliably outperforms SGD and NAG both in terms of training loss and test accuracy. The gap in performance appears to increase as model size increases or batch size decreases, suggesting that AGNES primarily excels in situations where gradients are harder to estimate accurately. For the sake of completeness, we include Adam with default hyperparameters as a comparison.
 
    In congruence with convergence guarantees from convex optimization, grid search suggests that $\alpha$ is the primary learning rate and $\eta$ should be chosen larger than $\alpha$. We tried NAG and Adam with higher learning rates $10^{-2}$ and $10^{-1}$ as well to ensure a fair comparison with AGNES, but found that they become unstable or perform worse for larger learning rates in our experiments. The AGNES default parameters $\alpha = 10^{-3}, \eta = 10^{-2}, \rho = 0.99$ in Algorithm \ref{descent algorithm} give consistently strong performance on different models but can be further tuned to improve performance. While the numerical experiments we performed support our theoretical predictions, we acknowledge that our focus lies on theoretical guarantees and we did not test these predictions over a broad set of benchmark problems.

    We present a more thorough comparison of NAG and AGNES with various parameter selections in Figure \ref{figure hyperparameter exploration} in Appendix \ref{appendix simulations}. With default parameters or minimal parameter tuning, AGNES reliably achieves superior performance compared to NAG (training loss) and smoother curves, suggesting more stable behavior (test accuracy).

    \subsection{Hyperparameter comparison}\label{hyperparameter comparison}
           \begin{figure*}    	
    	\includegraphics[clip = true, trim = 8mm 8mm 15mm 12mm, width = 0.33\textwidth]{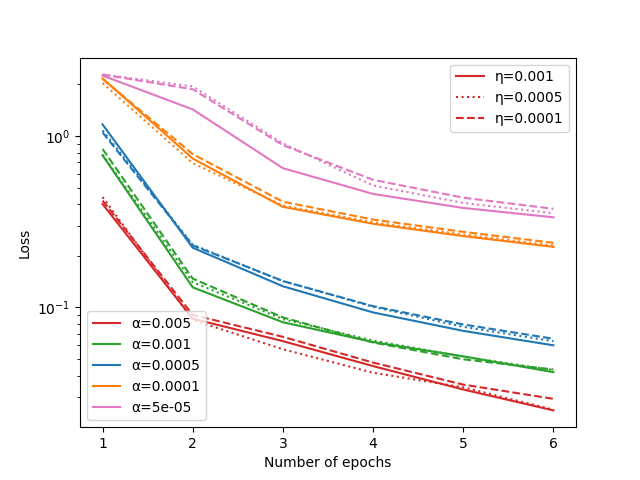}
    	\includegraphics[clip = true, trim = 8mm 8mm 15mm 12mm, width = 0.33\textwidth]{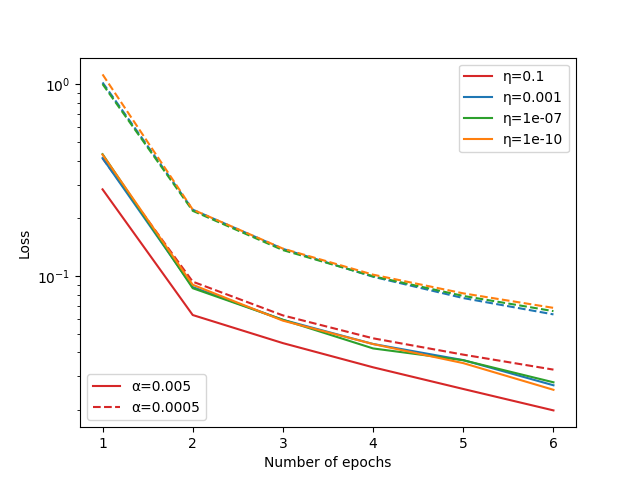}
    	\includegraphics[clip = true, trim = 8mm 8mm 15mm 12mm, width = 0.33\textwidth]{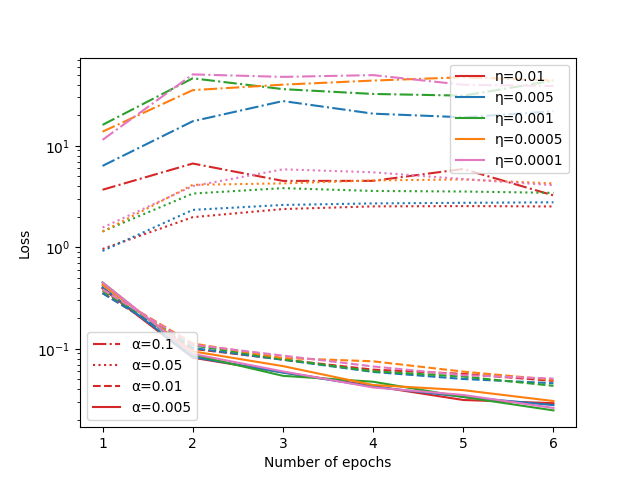} 
    	\caption{
    	\label{figure alpha vs eta} We report the average training loss after each epoch for six epochs for training LeNet-5 on MNIST with AGNES for various combinations of the hyperparameters $\alpha$ and $\eta$ to illustrate that $\alpha$ is the algorithm's primary learning rate. \textbf{Left:} 
      For a given $\alpha$ (color coded), the difference in the trajectory for the three values of $\eta$ (line style) is marginal.
      On the other hand, choosing $\alpha$ well significantly affects performance.
      \textbf{Middle:} 
      For any given $\alpha$, the largest value of $\eta$ performs much better than the other three values which have near-identical performance. Nevertheless, the worst performing value of $\eta$ with well chosen $\alpha = 5\cdot 10^{-3}$ performs better than the best performing value of $\eta$ with $\alpha = 5\cdot 10^{-4}$. 
      \textbf{Right:} 
      When $\alpha$ is too large, the loss increases irrespective of the value of $\eta$.
    	}
    	
    \end{figure*}
    We tried various combinations of AGNES hyperparameters $\alpha$ and $\eta$ to train LeNet-5 on the MNIST dataset to determine which hyperparameter has a greater impact on training. With a fixed batch size of 60 and a momentum value $\rho=0.99$, we trained independent copies of the model for 6 epochs for each combination of the hyperparameters. The average training loss over the epoch was recorded after each epoch. The results are reported in Figure \ref{figure alpha vs eta}. We see that $\alpha$ has the largest impact on the rate of decay of the loss, which establishes it as the `primary learning rage'. 
    If $\alpha$ is too small, the algorithm converges slowly and if $\alpha$ is too large, it diverges. If $\alpha$ is chosen correctly, a good choice of the correction step size $\eta$ (which can be orders of magnitude larger than $\alpha$) further accelerates convergence, 
    but $\eta$ cannot compensate for a poor choice of $\alpha$. 

    \newpage
    \subsection*{Acknowledgements}
    Portions of this research were conducted with the advanced computing resources provided by Texas A\&M High Performance Research Computing. This research was also supported in part by the University of Pittsburgh Center for Research Computing, RRID:SCR\_022735, through the resources provided. Specifically, this work used the H2P cluster, which is supported by NSF award number OAC-2117681.
    \bibliographystyle{plainnat}
    \bibliography{Agnes.bib}

    \newpage
	\appendix
    \part*{Appendix}
 \startcontents[appendix]
  \printcontents[appendix]{l}{1}{\setcounter{tocdepth}{2}}
	\newpage
\section{Additional experiments}\label{appendix simulations}
	\subsection{Numerical experiments for AGNES in smooth strongly convex optimization}
    We compare SGD and AGNES for the family of objective functions
    \[
    f_{\mu,L}:\R^2\to\R, \qquad f_{\mu,L}(x) = \frac\mu2\,x_1^2 + \frac L2\,x_2^2.
    \]
    We considered several stochastic estimators with multiplicative gradient scaling such as 
    \begin{itemize}
        \item collinear noise $g = (1+ \sigma N) \nabla f(x)$, where $N$ is one-dimensional standard normal.
        \item isotropic noise $g= \nabla f(x) + \frac{\sigma\,\|\nabla f(x)\|}{\sqrt d}\, N$, where $N$ is a $d$-dimensional unit Gaussian.
        \item Gaussian noise with standard variation $\sigma\|\nabla f(x)\|$ only in the direction orthogonal to $\nabla f(x)$. 
        \item  Gaussian noise with standard variation $\sigma\|\nabla f(x)\|$ only in the direction of the fixed vector $v=(1,1)/\sqrt2$.
        \item Noise of the form $\nabla f(x)+ \sqrt{1+ \sigma^2\,\|\nabla f(x)\|^2}\, N\,v$ where $v= (1, 1)/\sqrt 2$ and a variable $N$ which takes values $1$ or $-1$ with probability  $\frac12\,\frac{\sigma^2\,\|\nabla f(x)\|^2}{1+ \sigma^2\,\|\nabla f(x)\|^2}$ each; $N=0$ otherwise. In this setting, the noise remains macroscopically large at the global minimum, but the probability of encountering noise becomes small.
    \end{itemize}
    Numerical results were found to be comparable for all settings on a long time-scale, but the geometry of trajectories may change in the early stages of optimization depending on the noise structure.
    
    For collinear and isotropic noise, the results obtained for $f_{\mu,L}$ on $\R^2$ were furthermore found comparable (albeit not identical) to simulations with  a quadratic form on $\R^d$ with $d=10$ and
    \begin{itemize}
        \item $(d-1)$ eigenvalues $=\mu$ and one eigenvalue $=L$
        \item $(d-1)$ eigenvalues $=L$ and one eigenvalue $=\mu$
        \item eigenvalues equi-spaced between $\mu$ and $L$.
    \end{itemize}
     The evolution of $\E[f(x_n)]$ for different values of $\sigma$ and $L\geq \mu \equiv 1$ is considered for both SGD and AGNES in Figure \ref{figure strongly convex}. 

    The objective functions are $\mu=1$-convex and $L$-smooth. We use the optimal parameters $\alpha,\eta,\rho$ derived above for AGNES and the optimal step size $\eta = \frac{1}{L(1+\sigma^2)}$ for SGD. The mean of the objective value at each iteration is computed over 1,000 samples for each optimizer.

    \begin{figure}
    \includegraphics[width = 0.49\textwidth]{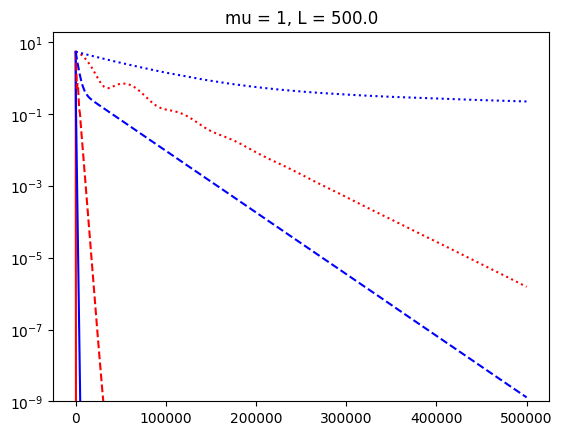}\hfill
    \includegraphics[width = 0.49\textwidth]{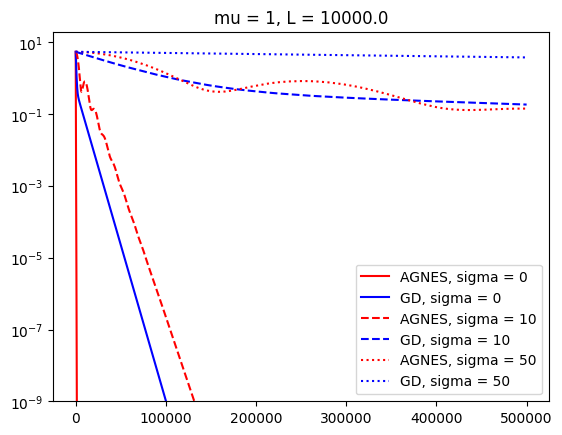}

    \caption{\label{figure strongly convex}
    We compare AGNES (red) and SGD (blue) for the optimization of $f_{\mu,L}$ with $\mu=1$ and $L=500$ (left) / $L=10^4$ (right) for different noise levels $\sigma=0$ (solid line), $\sigma=10$ (dashed) and $\sigma=50$ (dotted). In all cases, AGNES improves significantly upon SGD. The noise model used is isotropic Gaussian, but comparable results are obtained for different versions of multiplicatively scaling noise.
    }
    \end{figure}

    \subsection{Extensively comparing against NAG}
    We ran additional experiments testing a much wider range of hyperparameters for NAG for the task of classifying images from CIFAR-10. The results, presented in Figure \ref{figure hyperparameter exploration} indicate that AGNES outperforms NAG with little fine-tuning of the hyperparameters.

    We trained ResNet-34 using batch size of 50 for 40 epochs using NAG with learning rate in \{$8\cdot 10^{-5}, 10^{-4}, 2\cdot 10^{-4}, 5\cdot 10^{-4}, 8\cdot 10^{-4}, 10^{-3}, 2\cdot 10^{-3}, 5\cdot 10^{-3}, 8\cdot 10^{-3}, 10^{-2}, 2\cdot 10^{-2}, 5\cdot 10^{-2}, 8\cdot 10^{-2}, 10^{-1}, 2 \cdot 10^{-1}, 5\cdot 10^{-1}$\} and momentum value in \{$0.2, 0.5, 0.8, 0.9, 0.99$\}. These 80 combinations of hyperparameters for NAG were compared against AGNES with the default hyperparameters suggested $\alpha = 10^{-3}$ (learning rate), $\eta = 10^{-2}$ (correction step), and $\rho = 0.99$ (momentum) as well as AGNES with a slightly smaller learning rate $5\cdot 10^{-4}$ (with the other two hyperparameters being the same).
    
    AGNES consistently achieved a lower training loss as well as a better test accuracy faster than any combination of NAG hyperparameters tested. The same random seed was used each time to ensure a fair comparison between the optimizers. Overall, AGNES remained more stable and while other versions of NAG occasionally achieved a higher classification accuracy in certain epochs.

    \begin{figure}
    \centering
    \includegraphics[clip = true, trim = 15mm 8mm 15mm 3mm, width=0.65\linewidth]{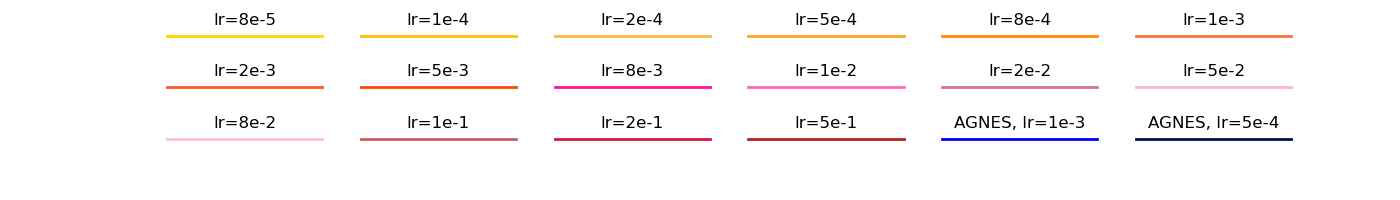}
    
    \includegraphics[clip = true, trim = 10mm 8mm 15mm 8mm, width=0.32\linewidth]{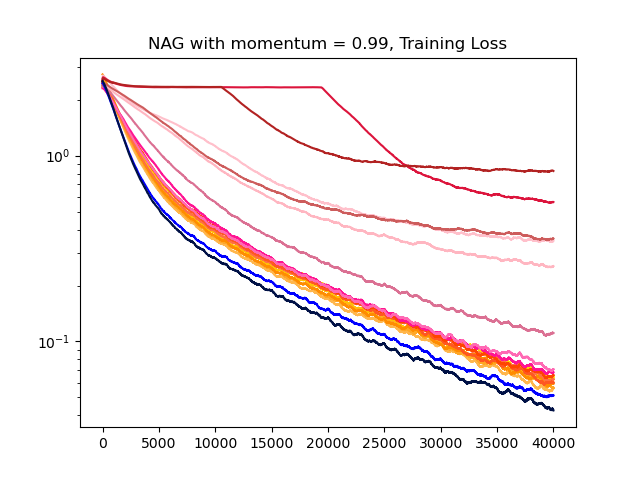}\hspace{3mm}
    \includegraphics[clip = true, trim = 10mm 8mm 15mm 8mm, width=0.32\linewidth]{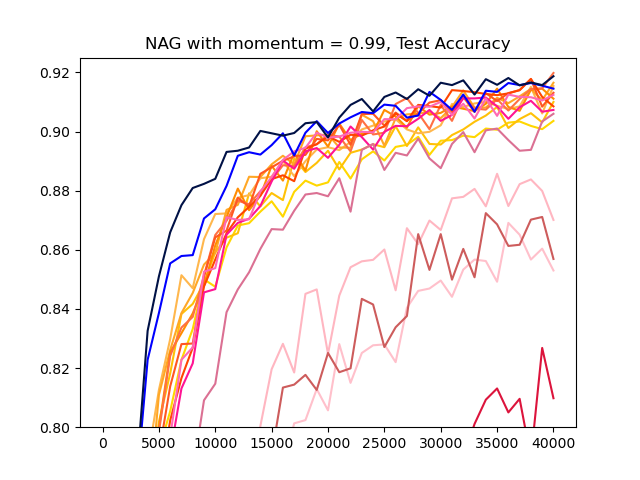}
    \vspace{2mm}
    
    \includegraphics[clip = true, trim = 10mm 8mm 15mm 8mm, width=0.32\linewidth]{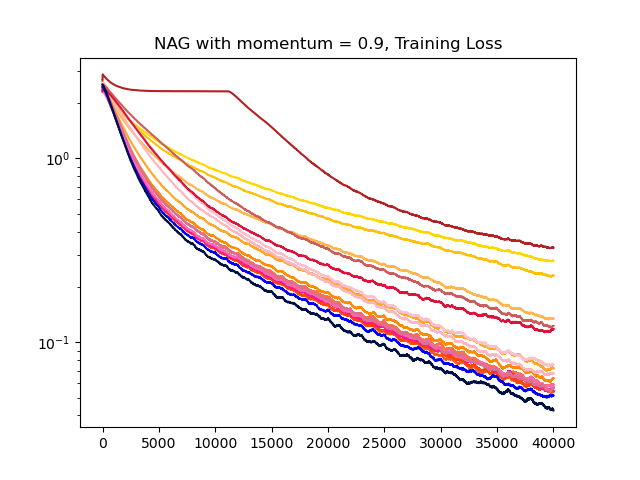}
    \hspace{3mm}
    \includegraphics[clip = true, trim = 10mm 8mm 15mm 8mm, width=0.32\linewidth]{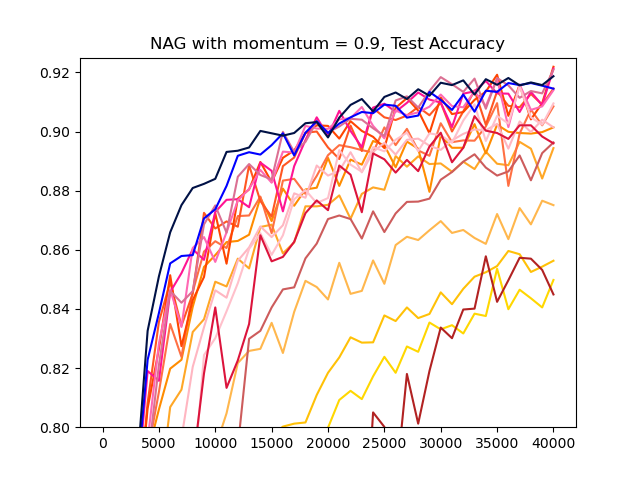}
    \vspace{2mm}
    
    \includegraphics[clip = true, trim = 10mm 8mm 15mm 8mm, width=0.32\linewidth]{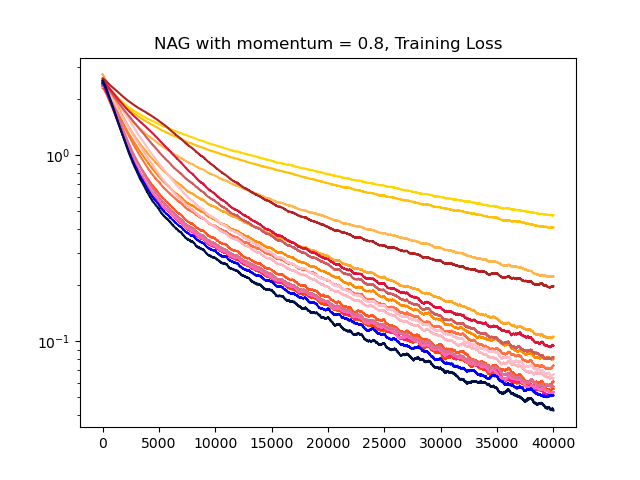}
    \hspace{3mm}
    \includegraphics[clip = true, trim = 10mm 8mm 15mm 8mm, width=0.32\linewidth]{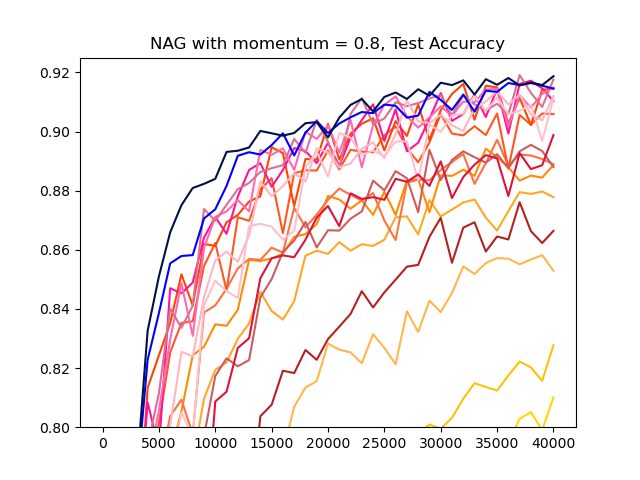}
    \vspace{2mm}
    
    \includegraphics[clip = true, trim = 10mm 8mm 15mm 8mm, width=0.32\linewidth]{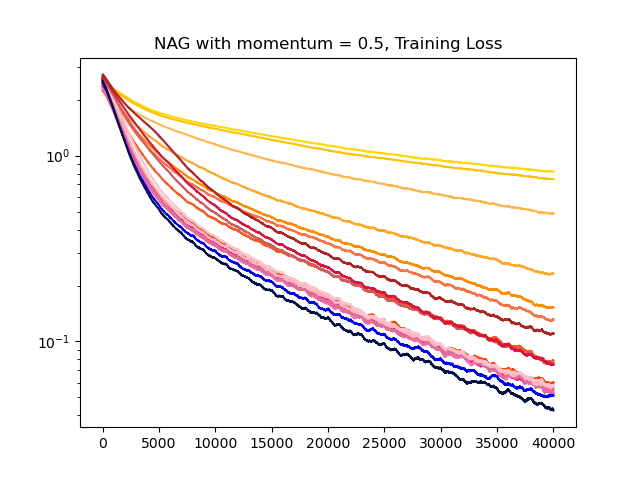}
    \hspace{3mm}
    \includegraphics[clip = true, trim = 10mm 8mm 15mm 8mm, width=0.32\linewidth]{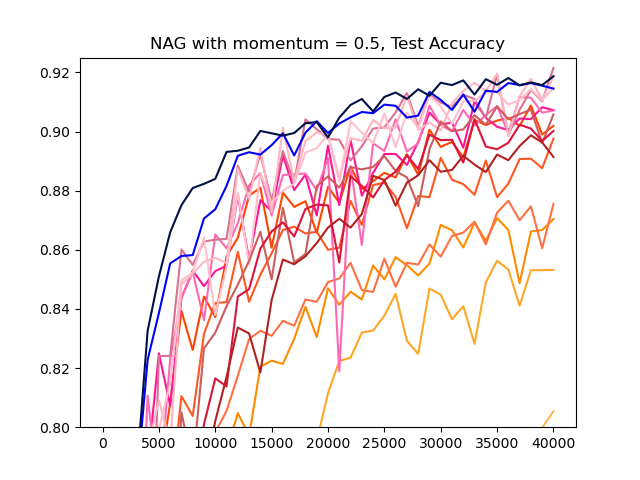}
    \vspace{2mm}
    
    \includegraphics[clip = true, trim = 10mm 8mm 15mm 8mm, width=0.32\linewidth]{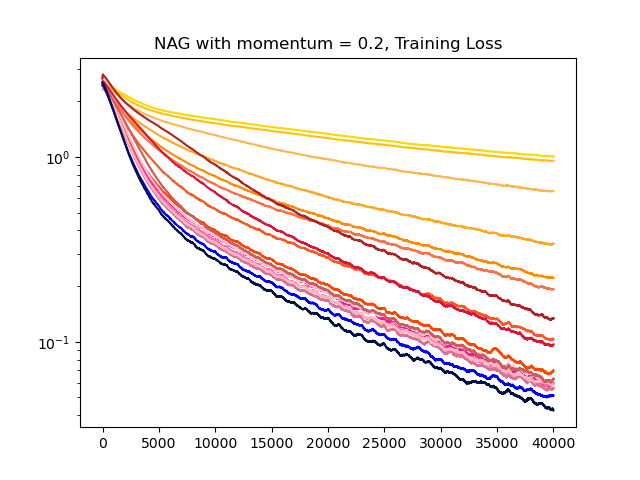}
    \hspace{3mm}
    \includegraphics[clip = true, trim = 10mm 8mm 15mm 8mm, width=0.32\linewidth]{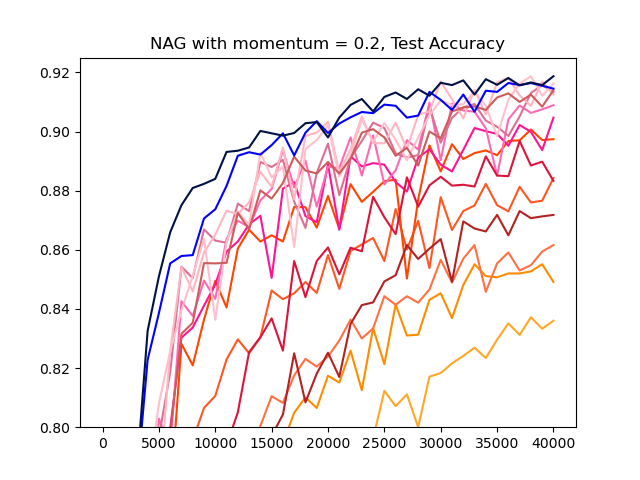}
    \caption{\label{figure hyperparameter exploration}
    We trained ResNet34 on CIFAR-10 with batch size 50 for 40 epochs using NAG. Training losses are reported as a running average with decay rate 0.999 in the left column and test accuracy after every epoch is reported in the right column. Each row represents a specific value of momentum used for NAG (from top to bottom: 0.99, 0.9, 0.8, 0.5, and 0.2) with learning rates ranging from $8\cdot 10^{-5}$ to $0.5$.  These hyperparameter choices for NAG were compared against AGNES with the default hyperparameters suggested $\alpha = 10^{-3}$ (learning rate), $\eta = 10^{-2}$ (correction step), and $\rho = 0.99$ (momentum) as well as AGNES with a slightly smaller learning rate $5\cdot 10^{-4}$ (with $\rho = 0.99$, $\eta = 10^{-2}$ as well). The same two training trajectories with AGNES are shown in all the plots in shades of blue. The horizontal axes represent the number of optimizer steps.}
\end{figure}

    \section{Multiple versions of the AGNES scheme}
	\subsection{Equivalence of the two formulations of AGNES}\label{appendix equiv}
	\begin{lemma}
		The two formulations of the AGNES time-stepping scheme (\ref{eq agnesterov}) and (\ref{eq agnes}) produce the same sequence of points.
	\end{lemma}
	\begin{proof}
		We consider the three-step formulation (\ref{eq agnes}),
		\[
		v_0 = 0, \quad
		x_n' = x_n +\alpha v_n, \qquad
		x_{n+1} = x_n' - \eta g'_n, \qquad
		v_{n+1} = \rho_n(v_n - g'_n),
		\] and use it to derive (\ref{eq agnesterov}) by eliminating the velocity variable $v_n$. If $v_0=0$, then $x'_0 = x_0$. From the definition $x'_n$, we get $\alpha v_n = x'_n - x_n$. Substituting this into the definition of $v_{n+1}$,
		\[ v_{n+1} = \rho_n \left(\frac{x'_n - x_n}{\alpha} - g'_n\right). \]
		Then using this expression for $v_{n+1}$ to compute $x'_{n+1}$,
		 \begin{align*}
		 	x'_{n+1} &= x_{n+1} + \alpha v_{n+1}\\
		 	& = x_{n+1} + \alpha\rho_n\left( \frac{x'_n - x_n}{\alpha} - g'_n \right) \\
		 	& = x_{n+1} + \rho_n(x'_n - \alpha g'_n - x_n).
		 \end{align*}
		Together with the definition of $x_{n+1}$ and the initialization $x'_0 = x_0$, this is exactly the two-step formulation (\ref{eq agnesterov}) of AGNES.		
	\end{proof}

    \subsection{Equivalence of AGNES and MaSS}
    After the completion of this work, we learned of 
    \cite{liu2018accelerating}'s Momentum-Added Stochastic Solver (MaSS) method, which generates sequences according to the iteration
\[
x_{n+1} = x_n' - \eta_1 g'_n, \qquad x_{n+1}' = (1+\gamma) x_{n+1} - \gamma x_{n} +\eta_2 g'_n
\]
where $g'_n$ is an estimate for $\nabla f(x_n')$. This is a version of AGNES with the choice $\eta=\eta_1$ and the momentum step
\begin{align*}
x_{n+1}' &= x_{n+1} + \rho(x_n' - \alpha g'_n - x_n)\\
    &= x_{n+1} + \rho(x_{n+1} + \eta g'_n -\alpha g'_n - x_n) \\
    &= (1+\rho)x_{n+1} - \rho x_n + (\eta-\alpha)\rho g'_n,
\end{align*}
i.e.\ MaSS coincides with AGNES for $\gamma = \rho$ and $\eta_2 = (\eta-\alpha)\rho$.

 \section{Continuous time interpretation of AGNES}\label{appendix continuous time}
	
	For better interpretability, we consider the continuous time limit of the AGNES algorithm. Similar ODE analyses of accelerated first order methods have been considered by many authors, including \cite{su2014differential,siegel2019accelerated,wilson2021lyapunov,attouch2022first,aujol2,zhang2018direct,dambrine2022stochastic}.

    Consider the time-stepping scheme
    \begin{equation}\label{eq agnes general}
	 	v_0 = 0, \quad
		x_n' = x_n + \gamma_1 v_n, \qquad
		x_{n+1} = x_n' - \eta g'_n, \qquad
		v_{n+1} = \rho_n(v_n - \gamma_2g'_n),
	\end{equation} 
    which reduces to AGNES as in \eqref{eq agnes} with the choice of parameters $\gamma_1 = \alpha, \gamma_2=1$. For the derivation of continuous time dynamics, we show that the same scheme arises with the choice $\gamma_1 = \gamma_2 = \sqrt\alpha$.

    \begin{restatable}{lemma}{invariance}\label{lemma alpha gamma}
	Let $\rho\in(0,1)$ and $\eta>0$ parameters.
	Assume that $\gamma_1, \gamma_2$ and $\tilde\gamma_1, \tilde\gamma_2$ are parameters such that $\tilde\gamma_1\tilde\gamma_2= \gamma_1\gamma_2$. Consider the sequences $(\tilde x_n, \tilde x_n', \tilde v_n)$ and $(x_n, x_n', v_n)$ generated by the time stepping scheme \eqref{eq agnes general} with parameters $(\rho, \eta,\tilde \gamma_1, \tilde \gamma_2)$ and $(\rho, \eta, \gamma_1, \gamma_2)$ respectively. If $x_0 = \tilde x_0$ and $\gamma_1 v_0 = \tilde \gamma_1\,\tilde v_0$, then $x_n = \tilde x_n$, $x_n' = \tilde x_n'$ and $\tilde\gamma_1\,\tilde v_n = \gamma_1\,v_n$ for all $n\in \N$.
	\end{restatable}

    \begin{proof}
	We proceed by mathematical induction on $n$. For $n=0$, the claim holds by the hypotheses of the lemma. For the inductive hypothesis, we suppose that $x_n = \tilde x_n$ and $\gamma_1 v_n = \tilde \gamma_1 \tilde v_n$ and prove the claim for $n+1$. Note that since $x'_n = x_n + \gamma_1 v_n$, it automatically follows that $x'_n = \tilde x'_n$. This implies that \[ x_{n+1} = x'_n - \eta g'_n = \tilde x'_n - \eta g'_n  = \tilde x_{n+1}.\] Considering the velocity term,
	\begin{align*}
	    \gamma_1 v_{n+1} &= \rho_{n} (\gamma_1 v_n - \gamma_1 \gamma_2 g'_n) 
	                   = \rho_n (\tilde \gamma_1 \tilde v_n - \tilde \gamma_1 \tilde \gamma_2 g'_n) 
	                   = \tilde \gamma_1 \rho_n (\tilde v_n - \tilde \gamma_2 g'_n)
	                   = \tilde \gamma_1 \tilde v_{n+1}.
	\end{align*}
	Thus $x_{n+1} = \tilde x_{n+1}$ and $\gamma_1 v_{n+1} = \tilde\gamma_1 \tilde v_{n+1}$. The induction can therefore be continued.
	\end{proof}

    Consider the choice of parameters in Theorem \ref{theorem strongly convex} by
    \[
    \eta = \frac1{L(1+\sigma^2)},\quad 
    \alpha = \frac{1 -\sqrt{\mu/L}}{1-\sqrt{\mu/L}+\sigma^2}\eta \approx \frac \eta{1+\sigma^2}, \quad 
    \rho = \frac{\sqrt{L}(1+\sigma^2) - \sqrt\mu}{\sqrt{L}(1+\sigma^2)+\sqrt\mu}
    \]
    if $\mu \ll L$.
    We denote $h:= \frac1{\sqrt L(1+\sigma^2)}$ and note that 
    \[
    \gamma_1 = \gamma_2 = \sqrt\alpha \approx h, \qquad \eta = (1+\sigma^2) h^2 = \frac{h}{\sqrt L}
    \]
    and\[
    \rho = 1 - 2 \frac{\sqrt \mu}{\sqrt{L}(1+\sigma^2) +\sqrt\mu} = 1- 2\sqrt\mu\,\frac{\sqrt L(1+\sigma^2)}{\sqrt L(1+\sigma^2) + \sqrt\mu} h\approx 1- 2\sqrt\mu h.
    \]
    Depending on which interpretation of $\eta$ we select, we obtain a different continuous time limit. First, consider the deterministic case $\sigma = 0$. Then
 
    \begin{align*}
    \begin{pmatrix}
        x_{n+1}\\ v_{n+1}
    \end{pmatrix}
    &= \begin{pmatrix} x_n\\ v_n\end{pmatrix} + h \begin{pmatrix} v_n - h\nabla f(x_n+hv_n) \\ -2\sqrt\mu\,v_n - (1-\sqrt\mu h) \nabla f(x_n+hv_n)\end{pmatrix}\\
    &= \begin{pmatrix} x_n\\ v_n\end{pmatrix} + h \begin{pmatrix} v_n \\ -2\sqrt\mu \,v_n - \nabla f(x_n)\end{pmatrix} + O(h^2)
    \end{align*}
    Keeping $f$ fixed and taking $h\to 0$, this is a time-stepping scheme for the coupled ODE system
    \[
    \begin{pmatrix} \dot x \\ \dot v\end{pmatrix}
        = \begin{pmatrix}
             v\\ -2\sqrt\mu\,v - \nabla f(x)
        \end{pmatrix}.
    \]
    Differentiating the first equation and using the system ODEs to subsequently eliminate $v$ from the expression, we observe that
    \[
    \ddot x = \dot v = - 2\sqrt\mu\,v - \nabla f(x) = - 2\sqrt\mu\,\dot x - \nabla f(x),
    \]
    i.e.\ we recover the heavy ball ODE. The alternative interpretation $\eta = h/\sqrt{L}$ can be analized equivalently and leads to a system
    \[
    \begin{pmatrix} \dot x \\ \dot v\end{pmatrix}
        = \begin{pmatrix}
             v - \frac1{\sqrt{L}}\,\nabla f(x)\\ -2\sqrt\mu\,v - \nabla f(x)
        \end{pmatrix}.
    \]
    which corresponds to a second order ODE 
    \begin{align*}
        \ddot x &= - \dot v - \frac1{\sqrt L} D^2f(x) \,\dot x\\
        &= - 2\sqrt\mu v - \nabla f(x) - \frac1{\sqrt L} D^2f(x) \,\dot x\\
        &= -2\sqrt\mu\,\left(\dot x + \frac1{\sqrt L}\nabla f(x)\right) - \nabla f(x)  - \frac1{\sqrt L} D^2f(x) \,\dot x\\
        &= - \left(2\sqrt\mu\,I_{m\times m} + \frac1{\sqrt L}\,D^2f(x)\right)\dot x - \left(1+ 2\,\sqrt{\frac\mu L}\right)\nabla f(x).
    \end{align*}
    This continuum limit is not a simple heavy-ball ODE, but rather a system with adaptive friction modelled by Hessian damping. A related Newton/heavy ball hybrid dynamical system was studied in greater detail by \cite{alvarez2002second}. For $L$-smooth functions, the $\ell^2$-operator norm of $D^2f(x)$ satisfies $\|D^2f(x)\|\leq L$, i.e.\ the additional friction term can be as large as $\sqrt L$ in directions corresponding to high eigenvalues of the Hessian. This provides significant regularization in directions which would otherwise be notably underdamped. 

    {Following Appendix \ref{appendix momentum parameters}, we maintain that the scaling 
     \[
            \frac{\eta(1-\rho)}{\alpha} = 2\,\sqrt{\frac \mu L}
    \]
    is `natural' as we vary $\eta,\alpha,\rho$. The same fixed ratio is maintained for the scaling choice $\eta = h/\sqrt L$ as
    \[
    \frac{\eta(1-\rho)}{\alpha} = \frac{h/\sqrt L\cdot 2\sqrt\mu\,h}{h^2} = 2\sqrt{\frac\mu L}.
    \]
    Indeed, numerical experiments in Section \ref{appendix simulations} suggest that such regularization may be observed in practice as high eigenvalues in the quadratic map do not exhibit `underdamped' behavior. We therefore believe that Hessian dampening is the potentially more instructive continuum description of AGNES.
    }
    A similar analysis can be conducted in the stochastic case with the scaling
    \[
    \gamma_1 = \gamma_2 = h, \quad \eta\in \left\{(1+\sigma^2)h^2, \:\frac h{\sqrt L}\right\}, \quad \rho = 1- 2\sqrt\mu h
    \]
    for large $\sigma$. We incorporate noise as 
    \[
    g'_n = (1+ \sigma N_n) \nabla f(x_n')
    \]
    and write
    \begin{align*}
    \begin{pmatrix}
        x_{n+1}\\ v_{n+1}
    \end{pmatrix}
    &= \begin{pmatrix} x_n\\ v_n\end{pmatrix} 
    + h \begin{pmatrix} v_n - (1+\sigma^2)h\nabla g'_n \\ -2\sqrt\mu\,v_n - (1-\sqrt\mu h)  \big(1+\sigma N_n)\nabla f(x_n+hv_n)\end{pmatrix}\\
    &= \begin{pmatrix} x_n\\ v_n\end{pmatrix} + h \begin{pmatrix} v_n \\ -2\sqrt\mu \,v_n - \nabla f(x_n) - \sqrt{h}\,\frac{\sigma\sqrt h}2\,N_n\nabla f(x)\end{pmatrix} + O(h^2)
    \end{align*}
    which can be viewed as an approximation of the coupled ODE/SDE system
    \[
    \begin{pmatrix} \d x\\ \d v\end{pmatrix} = \begin{pmatrix} v\,\d t\\ \big(-2\sqrt\mu \,v - \nabla f(x)\big)\, \d t + \sigma \sqrt h\,\d B \cdot \nabla f(x)\end{pmatrix}
    \]
    under moment bounds on the noise $N_n$. The precise noise type depends on the assumptions on the covariance structure of $N_n$ -- noise can point only in gradient direction or be isotropic on the entire space. For small $h$, the dynamics become deterministic. Again, an alternative continuous time limit is
    \[
    \begin{pmatrix} \d x\\ \d v\end{pmatrix} = \begin{pmatrix} (v-\nabla f(x)/\sqrt L)\,\d t + \frac{\sigma\sqrt h}{\sqrt L} \,\d B\cdot \nabla f(x) \\ \big(-2\sqrt\mu\, v - \nabla f(x)\big)\, \d t + \sigma \sqrt h\,\d B \cdot \nabla f(x)\end{pmatrix}
    \]
    if $\eta$ is scaled towards zero as $h/\sqrt L$. The first limiting structure is recoverd in the limit $L\to\infty$. Notably, the noise in the first equation is expected to be non-negligible if $\sigma\gg \sqrt L$. A similar analysis can be conducted in the convex case, noting that
    \[
    \frac{n+n_0}{n+n_0+3} = 1 - \frac{3}{n+n_0+3} = 1 - \frac{3}{(n+n_0+3)h} \,h
    \]
    where $(n+n_0+3)h$ roughly corresponds to the time $t$ in the continuous time setting.

	\section{Background material and auxiliary results}\label{appendix auxiliary}
	
	In this appendix, we gather a few auxiliary results that will be used in the proofs below. We believe that these will be familiar to the experts and can be skipped by experienced readers. 

    \subsection{A brief review of L-smoothness and (strong) convexity}
	Recall that if a function $f$ is $L$-smooth, then
	\begin{equation}\label{L-smoothness-equation}
	f(y) \leq f(x) + \nabla f(x)\cdot (y-x) + \frac L2\|x-y\|^2.
	\end{equation}
    For convex functions, this is in fact equivalent to $\nabla f$ being $L$-Lipschitz.
    \begin{lemma}
        If $f$ is convex and differentiable and satisfies \eqref{L-smoothness-equation}, then $\|\nabla f(x) - \nabla f(y)\| \leq L\|x - y\|$ for all $x$ and $y$.
    \end{lemma}
    \begin{proof}
        Setting $y = x - \frac{1}{L}\nabla f(x)$ in \eqref{L-smoothness-equation} implies that $f(x) - \inf_z f(z) \geq \frac{1}{2L}\|\nabla f(x)\|^2$. Applying this to the modified function $f_y(x) = f(x) - \nabla f(y) \cdot (x - y)$, which is still convex and satifies \eqref{L-smoothness-equation}, we get 
        \[
        \begin{split}
            f_y(x) - \inf_z f_y(z) = f_y(x) - f_y(y) &= f(x) - \nabla f(y)\cdot (x-y) - f(y)\\
            &\geq \frac{1}{2L}\|\nabla f_y(x)\|^2
             = \frac{1}{2L}\|\nabla f(x) - \nabla f(y)\|^2.
        \end{split}
        \]
        Note that here we have used the convexity to conclude that $\inf_z f_y(z) = f_y(y)$, i.e. that $f_y$ is minimized at $y$, since by construction $\nabla f_y(y) = 0$ (this is the only place where we use convexity!). Swapping the role of $x$ and $y$, adding these inequalities, and applying Cauchy-Schwartz we get
        \[
            \frac{1}{L}\|\nabla f(x) - \nabla f(y)\|^2 \leq (\nabla f(x) - \nabla f(y))\cdot (x - y) \leq \|\nabla f(x) - \nabla f(y)\|\|x - y\|,
        \]
        which implies the result.
    \end{proof}
	From the first order strong convexity condition, 
	\[
	f(y) \geq f(x) + \nabla f(x)\cdot (y-x) + \frac\mu2\|x-y\|^2,
	\]
	we deduce the more useful formulation $\nabla f(x) \cdot (x-y)\geq f(x) - f(y) + \frac\mu2\,\|x-y\|^2$. The convex case arises as the special case $\mu=0$. We note a special case of these conditions when one of the points is a minimizer of $f$.
	
	\begin{lemma}\label{lemma mu-L bound}
		If $f$ is an $L$-smooth function and $x^*$ is a point such that $f(x^*)=\inf_{x\in\R^m} f(x)$ then for any $x\in \R^m$,
		\[ f(x) - f(x^*) \leq \frac L2 \norm{x - x^*}^2. \] 
		Similarly, if $f$ is differentiable and $\mu$-strongly convex then for any $x\in \R^m$,
		\[ \frac \mu2 \norm{x - x^*}^2 \leq f(x) - f(x^*). \] 
	\end{lemma}
	\begin{proof}
		This follows from the two first order conditions stated above by noting that $\nabla f(x^*) = 0$ if $x^*$ is a minimizer of $f$.
	\end{proof}

    Additionally, $L$-smooth functions which are bounded from below satisfy the inequality
    \[
    \|\nabla f\|^2 \leq 2L\,(f - \inf f).
    \]
    Intuitively, if the gradient is large at a point, then we reduce $f$ quickly by walking in the gradient direction. The $L$-smoothness condition prevents the gradient from decreasing quickly along our path. Thus if the gradient is larger than a threshold at a point where $f$ is close to $\inf f$, then the inequality $f\geq \inf f$ would be violated.

	Let us record a modified gradient descent estimate, which is used only in the non-convex case. The difference to the usual estimate is that the gradient is evaluated at the terminal point of the interval rather than the initial point.
	
	\begin{lemma}\label{lemma momentum decrease}
		For any $x, v$ and $\alpha$: If $f$ is $L$-smooth, then
		\[
		f(x+\alpha v) \leq f(x) +\alpha \nabla f(x+\alpha v)\cdot v + \frac{L\alpha^2}2 \|v\|^2.
		\]
	\end{lemma}
	Note that if $f$ is convex, this follows immediately from \eqref{L-smoothness-equation} and the convexity condition $(\nabla f(y) - \nabla f(x))\cdot (y-x) \geq 0$.
	\begin{proof}
		The proof is essentially identical to the standard decay estimate. We compute
		\begin{align*}
			f(x) &= f(x+\alpha v) - \int_0^\alpha \frac{d}{dt} f(x+tv)\dt\\
			&= f(x+\alpha v) - \int_0^\alpha \big[\nabla f(x+\alpha v) +  \big\{\nabla f(x+tv)- \nabla f(x+\alpha v)\big\}\big] \cdot v\,\dt\\
			&\geq f(x+\alpha v) - \nabla f(x+\alpha v)\cdot v - \int_0^\alpha L\,(\alpha-t)\|v\|^2\,\dt\\
			&= f(x+\alpha v) - \alpha \, \nabla f(x+\alpha v)\cdot v - \frac{L\alpha^2}2 \|v\|^2.\qedhere
		\end{align*}
	\end{proof}

    \subsection{Stochastic processes, conditional expectations, and a decrease property for SGD}\label{appendix probability background}
	Now, we turn towards a very brief review of the stochastic process theory used in the analysis of gradient descent type algorithms. Recall that $(\Omega, \mathcal A, \P)$ is a probablity space from which we draw elements $\omega_n$ for gradient estimates $g(x_n', \omega_n)$ (AGNES) or $g(x_n,\omega_n)$ (SGD). We consider $x_0$ as a random variable on $\R^m$ with law $\mathbb Q$. Let us introduce the probability space $(\widehat \Omega, \widehat{\mathcal A}, \widehat \P)$ where
    \begin{enumerate}
        \item $\widehat\Omega = \R^d\times \prod_{n\in\N}\Omega$,
        \item $\widehat{\mathcal A}$ is the cylindrical/product $\sigma$-algebra on $\widehat \Omega$, and
        \item $\widehat \P = \mathbb Q \times \bigotimes \P$.
    \end{enumerate}
    The product $\sigma$-algebra and product measure are objects suited to events which are defined using only {\em finitely many} variables in the product space. A more detailed introduction can be found in \cite[Example 1.63]{klenke2013probability}. We furthermore define the filtration $\{\F_n\}_{n\in\N}$ where $\F_n$ is the $\sigma$-algebra generated by sets of the form
        \[
        B\times\prod_{i=1}^nA_i \times \prod_{i\in\N} \Omega, \qquad B\subseteq \R^m\text{ Borel},\quad A_i \in \mathcal A.
        \]
    In particular, $\bigcup_{n\in\N}\F_n \subseteq \sigma\left(\bigcup_{n\in\N} \F_n\right) =  \widehat{\mathcal{A}}$ and, examining the time-stepping scheme, it is immediately apparent that $x_n, x_n', v_n$ are $\F_n$-measurable random variables on $\widehat\Omega$. In particular, they are $\mathcal A$-measurable. {Alternatively, we can consider $\F_n$ as the $\sigma$-algebra generated by the random variables $x_1,x'_1,\dots,x_n,x'_n$, i.e.\ all the information that is known after intialization and taking $n$ gradient steps.} All probabilities in the main article are with respect to $\widehat\P$.
    
	Recall that conditional expectations are a technical tool to capture the stochasticity in a random variable $X$ which can be predicted from another random quantity $Y$. This allows us to quantify the randomness in the gradient estimators $g'_n$ which comes from the fact that $x_n$ is a random variable (not known ahead of time) and which randomness comes from the fact that on top of the inherent randomness due to e.g.\ initialization, we do not compute exact gradients. In particular, even at run time when $x_n$ is known, there is additional noise in the estimators $g'_n$ in our setting due to the selection of $\omega_{n}$.
		
	In the next Lemma, we recall two important properties of conditional expectations.
	
	\begin{lemma}\cite[Theorem 8.14]{klenke2013probability} \label{conditional expectaion properties}
		Let $g$ and $h$ be $\mathcal A$-measurable random variables on a probability space $(\Omega,\mathcal A, \P)$ and $\F\subseteq \mathcal A$ be a $\sigma-$algebra. Then the conditional expectations $\E[g\mid\F]$ and $\E[h\mid\F]$ satisfy the following properties:
		\begin{enumerate}
			\item (linearity) $\E[\alpha g + \beta h \mid \F] = \alpha \,\E[g] + \beta\, \E[h\mid\F]$ for all $\alpha,\beta\in \R$ 
			\item (tower identity) $\E[\E[g\mid\F]] = \E[g]$ 
			\item If $g$ is $\F-$measurable then $\E[gh\mid\F] = g\,\E[h\mid\F]$. In particular, $\E[g\mid\F] = g$
		\end{enumerate}
	\end{lemma}

    For a more thorough introduction to filtrations and conditional expectations, see e.g.\ \cite[Chapter 8]{klenke2013probability}.
    $\E[g'_n|\F_n]$ is the mean of $g'_n$ if all previous steps are already known. 
 
	\begin{lemma}\label{sufficient decrease}
		Suppose $g'_n, x_n,$ and $x'_n$ satisfy the assumptions laid out in Section \ref{section assumptions}, then the following statements hold
		\begin{enumerate}
            \item $\E\big[ g'_n |\F_n\big] = \nabla f(x_n')$
            \item $\E\big[ \|g'_n - \nabla f(x_n')\|^2\big] \leq \sigma^2\,\E\big[\|\nabla f(x_n')\|^2\big]$.
			\item $\E[\norm{g'_n}^2] = (1+\sigma^2)\E[\norm\df^2]$
			\item $\E[\df\cdot g'_n] = \E[\norm\df^2]$
		\end{enumerate}
	\end{lemma}
	
	\begin{proof}
    {\bf First and second claim.} This follows from Fubini's theorem.
 
    {\bf Third claim.} The third result then follows by an application of the tower identity with $\F_n$, expanding the square of the norm as a dot product, and then using the linearity of conditional expectation:
		\begin{align*} 
			\E \left[\norm{g'_n}^2\right] &= \E \left[ \E\left[\norm{g'_n}^2\mid \F_n \right]\right]\\
			& = \E\left[\E\left[ \norm{g'_n - \df}^2 + 2 g \cdot \df - \norm\df^2 \mid \F_n \right]\right]\\
			& = \E\left[ \E \left[\norm{g'_n - \df}^2 \mid \F_n \right] + 2 \E \left[ g \cdot \df \mid \F_n \right] - \E \left[ \norm{\df}^2 \mid \F_n\right] \right]\\
			& \leq \E \left[ \sigma^2 \df + 2 \norm{\df}^2 - \norm{\df}^2  \right]\\
			& = (1+\sigma^2)\E\left[ \norm{\df}^2  \right].
		\end{align*} 
	{\bf Fourth claim.} For the fourth result, we observe that since $f$ is a deterministic function and $x'_n$ is $\F_n$-measurable, $\df$ is also measurable with respect to the $\sigma$-algebra. Then using the tower identity followed by the third property in Lemma \ref{conditional expectaion properties},
		\begin{align*}
			\E \left[ \df \cdot g'_n \right] & = \E\left[\E\left[\df \cdot g'_n \mid \F_n\right]\right]\\
			& = \E \left[ \df\cdot \E\left[g'_n \mid \F_n\right]\right]\\
			& = \E\left[\df\cdot \df\right]\\
			& = \E\left[\norm{\df}^2\right].
		\end{align*}
	\end{proof}
	
	As a consequence, we note the following decrease estimate.
	
	\begin{lemma}\label{lemma gd decrease}
		Suppose that $f, x'_n$, and $g'_n=g(x'_n, \omega_n)$ satisfy the conditions laid out in Section \ref{section assumptions}, then
		\[
		\E\big[ f(x'_n-\eta g'_n)\big] \leq \E\big[ f(x'_n)\big] - \eta\left(1-\frac{L(1+\sigma^2)\eta}2\right)\,\E\left[\norm\df^2\right].
		\]
	\end{lemma}
	
	\begin{proof}
		Using $L$-smoothness of $f$,
		\begin{align*}
			f(x'_n - \eta g'_n) &\leq f(x'_n) - \eta g'_n\cdot\df + \frac{L\eta^2}{2}\norm{g'_n}^2.
		\end{align*} Then taking the expectation and using the results of the previous lemma,
		\begin{align*}
			\E\left[f(x'_n - \eta g'_n)\right] &\leq \E[f(x'_n)] - \eta \E\left[ \norm\df^2 \right] + \frac{L\eta^2}{2}(1+\sigma^2)\E\left[ \norm\df^2 \right] \\
			&\leq \E\big[ f(x'_n)\big] - \eta\left(1-\frac{L(1+\sigma^2)\eta}2\right)\,\E\left[\norm\df^2\right]
		\end{align*}
	\end{proof}
	
	In particular, if $\eta \leq \frac1{L(1+\sigma^2)}$, then
	\[
	\E\big[ f(x'_n-\eta g'_n)\big] \leq \E\big[f(x'_n)\big] - \frac\eta2 \,\E\left[\|\nabla f(x'_n)\|^2\right].
	\]
	
	\section{Convergence proofs: convex case}\label{appendix convex}
	\subsection{Gradient Descent (GD)}
	
	We first present a convergence result for stochastic gradient descent for convex functions with multiplicative noise scaling. To the best of our knowledge, convergence proofs for this type of noise which degenerates at the global minimum have been given by \cite{DBLP:journals/corr/abs-1811-02564, wojtowytsch2021stochasticdiscrete} under a Polyak-Lojasiewicz (or PL) condition (which holds automatically in the strongly convex case), but not for functions which are merely convex. We note that, much like AGNES, SGD achieves the same rate of convergence in stochastic convex optimization with multiplicative noise as in the deterministic case (albeit with a generally much larger constant). In particular, SGD with multiplicative noise is more similar to deterministic gradient descent than to SGD with additive noise in this way.
	
	Analyses of SGD with non-standard noise under various conditions are given by \cite{10.5555/3455716.3455953, 2019arXiv190704232S}.
	
	\begin{theorem}[GD, convex case]\label{theorem convex sgd}
		Assume that $f$ is a convex function and that the assumptions laid out in Section \ref{section assumptions} are satisfied. If the sequence $x_n$ is generated by the gradient descent scheme
		\[
		g_n= g(x_n,\omega_n), \qquad x_{n+1} = x_n - \eta g_n, \qquad \eta \leq \frac1{L(1+\sigma^2)},
		\]
		then for any $x^*\in \R^m$ and any $n_0\geq 1+\sigma^2$, 
		\[
		\E[f(x_n) - f(x^*)] \leq \frac{\eta n_0 \,\E[f(x_0) - f(x^*)] + \frac12\,\E\big[\|x_0-x^*\|^2\big]}{\eta (n+n_0)}.
		\]
		In particular, if $\eta = \frac1{L(1+\sigma^2)}$, $n_0 = 1+\sigma^2$, and $x^*$ is a point such that $f(x^*)=\inf_{x\in\R^m} f(x)$, then
		\[
		\E[f(x_n) - f(x^*)] \leq \, \frac{L(1+\sigma^2) \,\E\big[\|x_0-x^*\|^2\big]}{2( n+1+\sigma^2)}.
		\]
	\end{theorem}
	
	\begin{proof}
		Let $n_0\geq 0$ and consider the Lyapunov sequence
		\[
		\L_n =  \E\left[ \eta (n+n_0)\big( f(x_n) - \inf f \big) + \frac12\,\|x_n-x^*\|^2\right]
		\]
		We find that
		\begin{align*}
			\L_{n+1} &= \E\left[ \eta(n+n_0+1) \big\{f(x_n - \eta g_n) - \inf f\big\} + \frac12\,\|x_n - \eta g_n-x^*\|^2\right]\\
			&\leq \E\bigg[ \eta (n+n_0+1)\left\{f(x_n) - \frac \eta2 \|\nabla f(x_n)\|^2 - \inf f\right\}\\
            &\qquad+ \frac12\,\|x_n -x^*\|^2 - \eta\,(x_n-x^*)\cdot g_n + \frac{\eta^2}2\,\|g_n\|^2\bigg]\\
			&= \E\bigg[\eta (n+n_0)\big\{f(x_n) -\inf f\big\} + \frac12\|x_n-x^*\|^2 + f(x_n) - \inf f +\eta \,\nabla f(x_n) \cdot (x^*-x_n) \\
			&\qquad- \frac{\eta^2(n+n_0)}2\|\nabla f(x_n)\|^2 + \frac{\eta^2}2 \,\|g_n\|^2\bigg]\\
			&\leq \L_n + 0 - \frac{\eta^2}2 \big(n+n_0 - (1+\sigma^2)\big) \,\E\big[\|\nabla f(x_n)\|^2\big]
		\end{align*}
		by the convexity of $f$. The result therefore holds if $n_0$ is chosen large since
		\[
		\E[f(x_n) - f(x^*)] \leq \frac{\L_n}{\eta (n+n_0)} \leq \frac{\L_0}{\eta (n+n_0)} = \frac{\eta n_0 \,\E[f(x_0) - f(x^*)] + \frac12\,\E\big[\|x_0-x^*\|^2\big]}{\eta (n+n_0)}.
		\]
		{If $x^*$ is a minimizer of $f$ then the last claim in the theorem follows by using the upper bound $f(x_0) - f(x^*) \leq \frac L2 \|x_0-x^*\|^2 $ from Lemma \ref{lemma mu-L bound} and substituting $\eta = \frac{1}{L(1+\sigma)^2}, n_0 = 1+\sigma^2$.}
	\end{proof}

    \subsection{AGNES and NAG}
     The proofs of Theorems \ref{theorem nesterov convex} and \ref{theorem convex} in this section are constructed in analogy to the simplest setting of deterministic continuous-time optimization. As noted by \cite{su2014differential}, Nesterov's time-stepping scheme can be seen as a non-standard time discretization of the heavy ball ODE
     \[
     \begin{pde}\ddot x &= - \frac3t\,\dot x - \nabla f(x)&t>0\\\dot x&=0&t=0\\ x&= x_0 &t=0\end{pde}
     \]
     with a decaying friction coefficient.
     The same is true for AGNES, which reduces to Nesterov's method in the determinstic case.
     Taking the derivative and exploiting the first-order convexity condition, we see that the {\em Lyapunov function}
     \begin{equation}\label{eq continuous time lyapunov convex}
     \L(t) := t^2\big(f(x(t)) - f(x^*)\big) + \frac12\,\big\|t\dot x + 2\big(x(t) - x^*\big)\big\|^2
     \end{equation}
     is decreasing in time along the heavy ball ODE, see e.g.\ \cite[Theorem 3]{su2014differential}. Here $x^*$ is a minimizer of the convex function $f$. In particular
     \[
     f(x(t)) - f(x^*) \leq \frac{\L(t)}{t^2} \leq \frac{\L(0)}{t^2} = \frac{2\,\|x_0-x^*\|^2}{t^2}.
     \]
     To prove Theorems \ref{theorem nesterov convex} and \ref{theorem convex}, we construct an analogue to $\L$ in \eqref{eq continuous time lyapunov convex}. Note that $\alpha v_n = x_n' - x_n$ is a discrete analogue of the velocity $\dot x$ in the continuous setting. Both the proofs follow the same outline. Since Nesterov's algorithm is a special case of AGNES, we first prove Theorem \ref{theorem convex}.
	We present the Lyapunov sequence in a fairly general form, which allows us to reuse calculations for both proofs and suggests the optimality of our approach for Nesterov's original algorithm. 
	
For details on the probalistic set-up and useful properties of gradient estimators, see Appendix \ref{appendix probability background}. Let us recall the two-step formulation of AGNES, which we use for the proof,
\begin{equation}\tag{\ref{eq agnesterov}}
	x_0 = x'_0, \qquad x_{n+1} = x'_n - \eta g'_n, \qquad x'_{n+1} = x_{n+1} + \rho_n \big(x'_n -\alpha g'_n -x_n\big).
\end{equation} 
We first prove the alternative version mentioned after Theorem \ref{theorem convex} in the main text. Both proofs proceed initially identically and only diverge in Step 3. The reader interested mainly in Theorem \ref{theorem convex} is invited to read the first two steps of the proof of Theorem \ref{theorem convex appendix} and then skip ahead to the proof of Theorem \ref{theorem convex} below.

\begin{theorem}[AGNES, convex case, $n_0$ version]\label{theorem convex appendix}
	Suppose that $x_n$ and $x'_n$ are generated by the time-stepping scheme (\ref{eq agnes}), $f$ and $g'_n = g(x'_n, \omega_n)$ satisfy the conditions laid out in Section \ref{section assumptions}, $f$ is convex, and $x^*$ is a point such that $f(x^*) = \inf_{x\in\R^m} f(x)$. If the parameters are chosen such that
	 \[
	    \eta \leq \frac1{L(1+\sigma^2)}, \qquad \alpha < \frac{\eta}{1+\sigma^2}, \qquad n_0 \geq \frac{2\sigma^2\eta}{\eta - \alpha(1+\sigma^2)}, \qquad \rho_n = \frac{n+n_0}{n+n_0+3},
	 \]
	 then
	 \[
	 \E\big[ f(x_n) - f(x^*)\big] \leq \frac{(\alpha n_0+2\eta)n_0\,\E\big[ f(x_0) - \inf f\big] + 2\,\E\big[\,\|x_0-x^*\|^2\big]}{\alpha\,(n+n_0)^2}.
	 \]
	 In particular, if $\alpha \leq \frac{\eta}{1+2\sigma^2}$ then it suffices to choose $n_0 \geq 2\eta/\alpha \geq 2(1+2\sigma^2)$.
\end{theorem}

\begin{proof}
	{\bf Set-up.} Mimicking the continuous time model in \eqref{eq continuous time lyapunov convex}, we consider the Lyapunov sequence given by
	\[ \L_n = P(n) \,\E\left[f(x_n) - f(x^*)\right] + \frac12 \E \left[\norm{b(n)(x'_n - x_n) + a(n)(x'_n - x^*)}^2\right] \] 
	where $P(n)$ some function of $n$, $a(n) = a_0 + a_1 n$, and $b(n) = b_0 + b_1 n$ for some coefficients $a_0,a_1,b_0,b_1$. Our goal is to choose these in such a way that $\L_n$ is a decreasing sequence.

	    \textbf{Step 1.} If we denote the first half of the Lyapunov sequence as \(\L^1_n = P(n) \E\left[f(x_n) - f(x^*)\right] \), then
	\begin{align*}
		\L^1_{n+1} - \L^1_n &= P(n+1)\E[f(x_{n+1}) - f(x^*)] - P(n)\E[f(x_{n}) - f(x^*)] \\
		&\leq (P(n+1) + k)\E[f(x_{n+1}) - f(x^*)] - P(n)\,\E[f(x_{n}) - f(x^*)],
	\end{align*}
	where $k$ is a positive constant that can be chosen later to balance out other terms.
	Using Lemma \ref{sufficient decrease},
	\begin{align*}
		\L^1_{n+1} - \L^1_n
		&\leq  (P(n+1)+k)\E\left[f(x'_n) - c_{\eta,\sigma,L} \norm{\nabla f(x'_n)}^2 - f(x^*)\right] - P(n) \E[f(x_n) - f(x^*)]\\
		&= P(n) \E\left[f(x'_n) - f(x_n)\right] + \left(P(n+1)+k-P(n)\right)\E[f(x'_n)-f(x^*)] \\ & \quad- (P(n+1)+k)c_{\eta,\sigma,L} \E[\norm{\nabla f(x'_n)}^2]
	\end{align*} where $c_{\eta,\sigma,L} = \eta\left(1-\frac{L(1+\sigma^2)\eta}2\right)$. Using convexity,
	\begin{align}
		\L^1_{n+1} - \L^1_n &\leq P(n) \E\left[\nabla f(x'_n)\cdot (x'_n - x_n)\right] + \left(P(n+1)+k-P(n)\right)\E[\nabla f(x'_n)\cdot (x'_n - x^*)] \nonumber \\  
		& \quad- (P(n+1)+k)c_{\eta,\sigma,L} \E[\norm{\nabla f(x'_n)}^2]. \label{nest_first_half}
	\end{align}

	\textbf{Step 2.} We denote
	\[
	w_n = b(n)(x'_n - x_n) + a(n)(x'_n - x^*)
	\]
	and use the definition of \(x'_{n+1}\) from (\ref{eq agnesterov}),
	\begin{align*}
		w_{n+1} &=  b(n+1)(x'_{n+1} - x_{n+1}) + a(n+1)(x'_{n+1} - x^*) \\
		&=  b(n+1)\rho_n \left(x'_{n} - \alpha g'_n - x_n\right) + a(n+1)\left(x_{n+1} + \rho_n (x'_{n} - \alpha g'_n - x_n) - x^*\right) \\
		&= (b(n+1)+a(n+1))\rho_n \left(x'_{n} - \alpha g'_n - x_n\right) + a(n+1)(x_{n+1} - x^*).
	\end{align*}
	We will choose \[ \rho_n = \frac{b(n)}{b(n+1)+a(n+1)}, \] 
    such that the expression becomes
	\begin{align*}
		w_{n+1} &= b(n)\left(x'_{n} - \alpha g'_n - x_n\right) + a(n+1)(x_{n+1} - x^*) \\
		&= b(n)\left(x'_{n} - \alpha g'_n - x_n\right) + (a_0 + a_1n+a_1)(x'_{n} - \eta g'_n - x^*) \\
		&= w_n + a_1 (x'_n - x^*) - (\alpha b(n) + \eta a(n+1))g'_n.
	\end{align*}
	Then
	\begin{align*}
		\frac12 \norm{w_{n+1}}^2 - \frac12 \norm{w_n}^2 &= w_n\cdot (w_{n+1}-w_n) +  \frac12\norm{w_{n+1}-w_n}^2 \\
		&= w_n\cdot \left(a_1 (x'_n - x^*) - (\alpha b(n) + \eta a(n+1))g'_n\right) \\&\qquad + \frac 12 \norm{a_1 (x'_n - x^*) - (\alpha b(n) + \eta a(n+1))g'_n}^2.
	\end{align*}
	We want the terms in this expression to balance the terms in $\L^1_{n+1} - \L^1_n$, so we choose $a_1 = 0$, i.e.\ $a(n) = a_0$ is a constant. This implies,
	\begin{align}
		\E\bigg[\frac12 \norm{w_{n+1}}^2 - &\frac12 \norm{w_n}^2\bigg] = \E \left[- (\alpha b(n) + \eta a_0) w_n\cdot g'_n + \frac 12 (\alpha b(n) + \eta a_0)^2 \norm{g'_n}^2\right]  \nonumber\\
		&\leq - (\alpha b(n) + \eta a_0) \E[w_n\cdot \nabla f(x'_n)] + \frac 12 (\alpha b(n) + \eta a_0)^2(1+\sigma^2) \E[\norm{\nabla f(x'_n)}^2] \nonumber\\
		&= - (\alpha b(n) + \eta a_0)b(n) \E[(x'_n-x_n)\cdot \nabla f(x'_n)] \nonumber
		\\& \quad - (\alpha b(n) + \eta a_0)a_0 \E[(x'_n - x^*)\cdot \nabla f(x'_n)] \nonumber\\& \quad
		+ \frac 12 (\alpha b(n) + \eta a_0)^2(1+\sigma^2) \E[\norm{\nabla f(x'_n)}^2]. \label{nest_second_half}
	\end{align}

	\textbf{Step 3.} Combining the estimates (\ref{nest_first_half}) and (\ref{nest_second_half}) from the last two steps,
	\begin{align*}
		\L_{n+1} - \L_n &\leq \left(P(n) - (\alpha b(n) + \eta a_0)b(n)\right)\E\left[\nabla f(x'_n)\cdot (x'_n - x_n)\right] \\
		& \quad + \left(P(n+1)+k-P(n) - (\alpha b(n) + \eta a_0)a_0 \right)\E[\nabla f(x'_n)\cdot (x'_n - x^*)] \nonumber \\  
		& \quad + \left( \frac12 (\alpha b(n) + \eta a_0)^2(1+\sigma^2) - (P(n+1)+k)c_{\eta,\sigma,L}\right) \E[\norm{\nabla f(x'_n)}^2].
	\end{align*}
    Since $\norm{\nabla f(x'_n)}^2\geq 0$ and $\nabla f(x'_n)\cdot (x'_n - x^*) \geq f(x'_n) - f(x^*) \geq 0 $, we require the coefficients of these two terms to be non-positive and the coefficient of $\nabla f(x'_n)\cdot (x'_n - x_n)$ to be zero. That gives us the following system of inequalities,
	\begin{align}
		P(n) &= (\alpha b(n) + \eta a_0)b(n) \label{convex_coeff_1}\\
		P(n+1)+k-P(n) &\leq (\alpha b(n) + \eta a_0)a_0 \label{convex_coeff_2}\\
		\frac12 (\alpha b(n) + \eta a_0)^2(1+\sigma^2) &\leq (P(n+1)+k)\eta\left(1-\frac{L(1+\sigma^2)\eta}2\right). \label{convex_coeff_3}
	\end{align}

	\textbf{Step 4.} Now we can choose values that will satisfy the above system of inequalities. We substitute $a_0 = 2, b_1 = 1, b_0 = n_0$, and $k=2\eta-\alpha $.
	From (\ref{convex_coeff_1}), we get \(P(n) = \left(\alpha(n+n_0) + 2\eta\right)(n+n_0) \). Next, we observe that \[P(n+1) = P(n) + \alpha + 2\alpha(n+n_0) + 2 \eta.\] Then (\ref{convex_coeff_2}) holds because
	\begin{align*}
		P(n+1) + k - P(n) &= \alpha + 2\alpha(n+n_0) + 2 \eta + 2\eta - \alpha \\
		&= 2(\alpha(n+n_0) + 2 \eta) \\
		&= (\alpha b(n) + \eta a_0)a_0.
	\end{align*}
	We now choose $\eta$ to satisfy $\eta \leq \frac{1}{L(1+\sigma^2)}$, which ensures that $\frac \eta2 \leq \eta\left( 1 - \frac{L(1+\sigma^2)\eta}{2} \right)$. Consequently, for (\ref{convex_coeff_3}), it suffices to ensure that 
	\[(\alpha b(n) + \eta a_0)^2(1+\sigma^2) \leq (P(n+1)+k)\eta,\]
	which is equivalent to showing that the polynomial,
	\begin{align*}
		q(z) &= \left( \alpha z^2 + 2\eta z + \alpha + 2\alpha z + 2\eta + 2\eta - \alpha \right)\eta\\
		&\quad - \left(\alpha^2  z^2 + 4\eta^2 + 4\alpha\eta z\right)(1+\sigma^2),
	\end{align*}
	is non-negative for all $z\geq n_0$. $q(z)$ simplifies to
	\begin{align*}
		q(z) = \alpha(\eta - \alpha(1+\sigma^2))z^2 + 2\eta(\eta + \alpha - 2\alpha(1+\sigma^2))z - 4\eta^2\sigma^2.
	\end{align*}
    To guarantee that $q$ is non-negative for $z = n+n_0 \geq n_0$, we require that 
    \begin{enumerate}
        \item the leading order coefficient is strictly positive\footnote{\ In principle, it would suffice if the quadratic coefficient vanished and the linear coefficient were strictly positive. However, if $\eta - \alpha(1+\sigma^2) =0$, then the coefficient of the linear term is negative since $\eta +\alpha - 2\alpha(1+\sigma^2) = - \alpha \sigma^2<0$. We therefore require the coefficient of the quadratic term to be positive.} and
        \item $n_0\geq0$, $q(n_0)\geq 0$.
    \end{enumerate}
    Since $q(0) < 0$ and $q$ is quadratic, this suffices to guarantee that $q$ is increasing on $[n_0,\infty)$. The first condition reduces to the fact that
    \[
    \eta - \alpha(1+\sigma^2) >0. 
    \]
    We can find the minimal admissible value of $n_0$ by the quadratic formula. We first consider the term {\em outside} the square root:
    \begin{align*}
        -&\frac{2\eta(\eta + \alpha - 2\alpha(1+\sigma^2))}{2\alpha(\eta - \alpha(1+\sigma^2))} 
        = -\frac\eta\alpha \left(1 - \frac{\alpha\,\sigma^2}{\eta-\alpha(1+\sigma^2)}\right)
    \end{align*}
    and thus
    \begin{align*}
    &n_0 \geq -\frac\eta\alpha \left(1 - \frac{\alpha\,\sigma^2}{\eta-\alpha(1+\sigma^2)}\right) + \sqrt{\frac{\eta^2}{\alpha^2} \left(1 - \frac{\alpha\,\sigma^2}{\eta-\alpha(1+\sigma^2)}\right)^2 + \frac{4\eta^2\sigma^2}{\alpha(\eta-\alpha(1+\sigma^2))}}\\
    &= \frac\eta\alpha\left\{\sqrt{1 - 2\frac{\alpha\sigma^2}{\eta-\alpha(1+\sigma^2)} + \left(\frac{\alpha\sigma^2}{\eta-\alpha(1+\sigma^2)}\right)^2 + 4\,\frac{\alpha\sigma^2}{\eta-\alpha(1+\sigma^2)}} -\left(1 - \frac{\alpha\,\sigma^2}{\eta-\alpha(1+\sigma^2)}\right) \right\}\\
    &= \frac\eta\alpha\left\{\sqrt{\left(1 + \frac{\alpha\,\sigma^2}{\eta-\alpha(1+\sigma^2)}\right)^2 }
    -\left(1 - \frac{\alpha\,\sigma^2}{\eta-\alpha(1+\sigma^2)}\right) \right\}\\
    &= \frac\eta\alpha \,\frac{2\alpha\sigma^2}{\eta-\alpha(1+\sigma^2)}\\
    &= \frac{2\eta\,\sigma^2}{\eta-\alpha(1+\sigma^2)}.
    \end{align*}
    In particular, in the deterministic case $\sigma=0$, the choice $n_0=0$ is admissible. Notably, we require $n_0 \geq 2\sigma^2\,\frac\eta\eta = 2\sigma^2$. Furthermore, if \( \alpha \leq \frac{\eta}{1+2\sigma^2}, \)
    then
    \[ \frac{2\sigma^2\eta}{\eta - \alpha (1+\sigma^2)} \leq \frac{2\sigma^2 \eta}{\alpha \sigma^2} = \frac{2\eta}{\alpha}, \]
    so it suffices to choose $n_0 \geq 2\eta/\alpha$ in this case.

	\textbf{Step 5.} We have shown that the Lyapunov sequence,
	\[ \L_n = ((n+n_0)\alpha + 2\eta)(n+n_0) \E\left[f(x_n) - f(x^*)\right] + \frac12 \E \left[\norm{(n+n_0)(x'_n - x_n) + 2(x'_n - x^*)}^2\right], \] is monotone decreasing. It follows that
	\[
	\E\big[ f(x_n) - f(x^*)\big] \leq \frac{\L_n}{P(n)} \leq \frac{\L_0}{P(n)} \leq \frac{\E\big[ (n_0\alpha+2\eta)n_0\,\big( f(x_0) - f(x^*)\big) + 2\,\|x_0-x^*\|^2\big]}{\alpha\,(n+n_0)^2}.
	\]
	If $2\eta \leq \alpha n_0$, we get 
	\[
	\E\big[ f(x_n) - f(x^*)\big] \leq \frac{ 2\alpha n_0^2\,\E\big[ f(x_0) - \inf f\big] + 2\,\E\big[\,\|x_0-x^*\|^2\big]}{\alpha\,(n+n_0)^2}.
	\]
	Finally, if $\eta = \frac1{L(1+\sigma^2)},\; \alpha = \frac{1}{L(1+\sigma^2)(1+2\sigma^2)},$ and $n_0 = 2(1+2\sigma^2)$, then using Lemma \ref{lemma mu-L bound}, the expression above simplifies to
	\[ \E\big[ f(x_n) - f(x^*)\big] \leq \frac{2L(1+2\sigma^2)(3+5\sigma^2)\E\left[\norm{x_0 - x^*}^2\right]}{n^2}. \qedhere\] 
	\end{proof}

    \begin{remark}
        Note that SGD arises as a special case of this analysis if we consider $\alpha=0, n_0 \geq 2\sigma^2$ since $P(n)$ is a {\em linear} polynomial in this case.
    \end{remark}

    \begin{remark}\label{remark eventual decrease}
        Note that the proof of Theorem \ref{theorem convex appendix} implies more generally that
        \[
        \L_{n+1} \leq \L_n - q(n+n_0) \,E\big[\|\nabla f(x_n')\|^2\big],
        \]
        even if $n_0$ is not chosen such that $q(n+n_0)\geq 0$ for all $n$. However, $q(n+n_0)\geq 0$ for all {\em sufficiently large} $n\in\N$, i.e.\ $\L_n$ decreases eventually (assuming that $\L_n<\infty$ for all finite $n$) . More precisely, for given $\eta,\alpha$ if 
        \[
        n+ n_0 \geq n^* := \left\lceil\frac{\eta\sigma^2}{\eta - \alpha (1+\sigma^2)}\right\rceil, \qquad\text{then }\quad \L_{n} \leq \frac{\L_{n^*}}{\alpha\,(n+n_0)^2}.
        \]
        Thus a poor choice of $n_0$ will not prevent convergence, but it may delay it.
    \end{remark}

    We now prove the version of this result stated in the main text, for which $\rho_n = \frac{n}{n+a_0+1}$ with $a_0>2$, i.e.\ with slightly more friction.

\convex*
\begin{proof}
The proof for this version of Theorem \ref{theorem convex} is identical to the proof of Theorem \ref{theorem convex appendix} until Step 3, after which we take an alternate approach. Let us recall the expression we got in the beginning of Step 3.

    \textbf{Step 3.} We want to show that the bound on the right hand side of the inequality
	\begin{align}\label{eq convex lyapunov second version}
	\L_{n+1} - \L_n &\leq \left(P(n) - (\alpha b(n) + \eta a_0)b(n)\right)\E\left[\nabla f(x'_n)\cdot (x'_n - x_n)\right]\nonumber \\
		& \quad + \left(P(n+1)+k-P(n) - (\alpha b(n) + \eta a_0)a_0 \right)\E[\nabla f(x'_n)\cdot (x'_n - x^*)]  \\  \nonumber
		& \quad + \left( \frac12 (\alpha b(n) + \eta a_0)^2(1+\sigma^2) - (P(n+1)+k)c_{\eta,\sigma,L}\right) \E[\norm{\nabla f(x'_n)}^2] 
	\end{align}
    is non-positive.
    Using convexity and $L$-smoothness in the form of \cite[Lemma B.1]{wojtowytsch2021stochasticdiscrete}, we get the inequality
    \[
        \nabla f(x'_n)\cdot (x'_n - x^*) \geq  f(x'_n) - f(x^*) \geq  \frac{1}{2L}\norm{\df}^2,
    \]
    which allows us to combine the second and third line in \eqref{eq convex lyapunov second version}, assuming that the coefficient in the second line is non-positive. If this is the case, then the entire right hand side of \eqref{eq convex lyapunov second version} is bounded from above by
    \begin{align*}
	&\left(P(n) - (\alpha b(n) + \eta a_0)b(n)\right)\E\left[\nabla f(x'_n)\cdot (x'_n - x_n)\right]\\
    &\quad + \bigg\{\left\{P(n+1)+k-P(n) - (\alpha b(n) + \eta a_0)a_0 \right\}\frac{1}{2L}\\
    &\qquad+ \left( \frac12 \big(\alpha b(n) + \eta a_0\big)^2(1+\sigma^2) - (P(n+1)+k)c_{\eta,\sigma,L}\right) \bigg\}\E[\norm{\nabla f(x'_n)}^2].
	\end{align*}
    As $\E\left[\nabla f(x'_n)\cdot (x'_n - x_n)\right]$ does not have a sign, we choose to set its coefficient to zero, and we require both the coefficient in the second line of \eqref{eq convex lyapunov second version} and the coefficient of $\E[\norm{\df}^2]$ in the combined version to be non-positive.
    Noting that $c_{\eta,\sigma,L} \geq \eta/2$ if $\eta\leq 1/L(1+\sigma^2)$, this leads to the  system of inequalities
        \begin{align}
		P(n) &= (\alpha b(n) + \eta a_0)b(n) \label{convex_coeff_01}\\
		P(n+1)+k-P(n) &\leq (\alpha b(n) + \eta a_0)a_0 \label{convex_coeff_02}\\
		P(n+1)+k-P(n) - (\alpha b(n) + \eta a_0)a_0 &\leq L\big( (P(n+1)+k)\eta - (\alpha b(n) + \eta a_0)^2(1+\sigma^2) \big). \label{convex_coeff_03}
	\end{align}
    In principle, this approach is more general than that of Theorem \ref{theorem convex} as we do not require two terms to be individually non-positive, but only one of them and their weighted sum. In the proof of Theorem \ref{theorem convex}, a similar role was played by the parameter $k$, which allowed to shift a small positive term between expressions.
 
 \textbf{Step 4.} Now we can choose the parameters and variables so as to satisfy the inequalities above. We begin by setting $b_1 = 1, b_0 = 0, k=0$, and choosing $\alpha, \eta, a_0$ as in the theorem statement. Using \eqref{convex_coeff_01} as the definition of $P(n)$, we note that
 \begin{align*}
 P(n+1) - P(n) &= 2\alpha n + \alpha + \eta a_0.
 \end{align*}
 Thus, \eqref{convex_coeff_02} simplifies to
 \[ 2\alpha n + \alpha + \eta a_0 \leq a_0(\alpha n + \eta a_0), \]
 which holds since $a_0 \geq 2$ and $\eta \geq \alpha$.
The right hand side of \eqref{convex_coeff_03} simplifies to
 \begin{align*}
  L\Big( \eta(n+1)&(\alpha (n+1) + \eta a_0) - L\big( \alpha n + \eta a_0)^2(1+\sigma^2) \Big) \\
     & = L\left( \big\{\eta \alpha - (1+\sigma^2)\alpha^2 \big\}\,n^2 + \big\{\eta(2\alpha +\eta a_0) - 2\eta a_0\alpha(1+\sigma^2) \big\}\,n -\eta^2a_0^2(1+\sigma^2) + \eta(\alpha +\eta a_0)\right)\\
     & = L\left(\big\{\eta(2\alpha +\eta a_0) - 2\eta a_0\alpha(1+\sigma^2) \big\}\,n -\eta^2a_0^2(1+\sigma^2) + \eta(\alpha +\eta a_0)\right),
\end{align*}
where the last equality holds since $\alpha = \eta/(1+\sigma^2)$. Thus for \eqref{convex_coeff_03} to hold, it suffices that
\begin{align*}
    2\alpha n + \alpha + \eta a_0 - a_0(\alpha n+\eta a_0) \leq L\left(\big\{\eta(2\alpha +\eta a_0) - 2\eta a_0\alpha(1+\sigma^2) \big\}\,n -\eta^2a_0^2(1+\sigma^2) + \eta(\alpha +\eta a_0)\right),
\end{align*} 
which is equivalent to
\begin{align}
    \big\{\alpha(2-a_0) - L\eta (2\alpha +\eta a_0) + 2L\eta  a_0\alpha(1+\sigma^2) \big\}\; n  + \big\{ \alpha + \eta a_0 - a_0^2\eta +  L\eta^2 a_0^2(1+\sigma^2) - L\eta (\alpha +\eta a_0) \big\} \leq 0. \label{decrease_1}
\end{align}
A linear polynomial is non-negative for all $n\geq 0$ if and only if both of its coefficients are.
The leading order coefficient in \eqref{decrease_1} is
\begin{align*}
    \alpha(2-a_0) - L\eta (2\alpha +\eta a_0) + 2L\eta a_0\alpha(1+\sigma^2)
    &=  \alpha(2-a_0) - L\eta (2\alpha +\eta a_0) + 2L\eta^2 a_0 \\
    &=  2\alpha - 2L\eta\alpha + a_0(-\alpha + L \eta^2) \\
    &=  \frac{\eta}{1+\sigma^2}\big( 2(1-L\eta) + a_0(L \eta (1+\sigma^2) - 1) \big),
\end{align*}
which is non-positive if and only if $a_0 \geq \frac{2(1-L\eta)}{1-L\eta (1+\sigma^2)}$. We remark that it is this part of the computation that forces us to choose $\eta$ strictly smaller than $\frac{1}{L(1+\sigma^2)}$. In the deterministic case $\sigma=0$, we would encounter no such limitation as the term would be automatically zero for $\eta = 1/L$. 
Finally, we consider the constant term in \eqref{decrease_1} and use the fact that $1< a_0$ and $\alpha \leq \eta$
\begin{align*}
    \alpha + \eta a_0 - a_0^2\eta +  L\eta^2 a_0^2(1+\sigma^2) - L\eta (\alpha +\eta a_0) &= 
    (\alpha + \eta a_0)(1-L\eta) + a_0^2\eta(L\eta(1+\sigma^2)-1) \\
    &\leq 2\eta a_0 (1-L\eta) + a_0^2\eta(L\eta(1+\sigma^2)-1) \\
    &\leq \eta a_0 \left( 2(1-L\eta) + a_0(L\eta(1+\sigma^2)-1)  \right) \\
    &\leq 0,
\end{align*}
using again that $a_0 \geq 2\frac{1-L\eta}{1-L\eta(1+\sigma^2)}$.
This shows that $\L_{n+1} \leq \L_{n}$.

{\bf Step 5.} The conclusion again follows as in the proof of Theorem \ref{theorem convex appendix}.
\end{proof}

	In addition to convergence in expectation, we get almost sure convergence as well.
    \almostsurely*

    The same is of course true for Theorem \ref{theorem convex appendix}.
    
	\begin{proof}
	The conclusion follows by standard arguments from the fact that the sequence of expectations $\E[f(x_n) - \inf f]$ is summable: By the previous argument, the estimate 
	\[
	\E\big[ \big|f(x_n) - f(x^*)\big|\big] = \E\big[ f(x_n) - f(x^*)\big] \leq \frac{C}{n^2}
	\]
	holds for some $C>0$. Since
	\begin{align*}
		\P\left( \lim_{n\to\infty} f(x_n) \neq \inf f\right) 
		&= \P\left( \limsup_{n\to\infty} |f(x_n) - \inf f|>0\right)\\
		&= \P\left( \bigcup_{k=1}^\infty \left\{\limsup_{n\to\infty} |f(x_n) - \inf f|>\frac1k\right\}\right)\\
		&\leq \sum_{k=1}^\infty \P\left( \limsup_{n\to\infty} |f(x_n) - \inf f|>\frac1k\right),
	\end{align*}
	it suffices to show that $\P\left( \limsup_{n\to\infty} |f(x_n) - \inf f|>\eps\right) = 0$ for any $\eps>0$. We further note that for any $N\in \N$ we have
	\begin{align*}
		\P\left( \limsup_{n\to\infty} |f(x_n) - \inf f|>\eps\right)         &\leq \P\left( \exists\ n\geq N \text{ s.t. }|f(x_n) - \inf f|>\eps\right)\\
		&= \P\left(\bigcup_{n=N}^\infty \left\{ |f(x_n) - \inf f| >\eps\right\}\right)\\
		&\leq \sum_{n=N}^\infty \P\left(|f(x_n) - \inf f| >\eps\right)\\
		&\leq \sum_{n=N}^\infty \frac{\E\big[|f(x_n) - \inf f|\big]}\eps\\
		&\leq \frac{C}\eps \sum_{n=N}^\infty \frac1{n^2}
	\end{align*}
	by Markov's inequality. 
	As the series over $n^{-2}$ converges, the expression on the right can be made arbitrarily small by choosing $N$ sufficiently large. Thus the quantity on the left must be zero, which concludes the proof. In the strongly convex case, the series $\sum_{n=1}^\infty \left(1 - \sqrt{\frac\mu L} \frac1{1+\sigma^2}\right)^n$ converges and thus the same argument applies there as well.
\end{proof}

Next we turn to NAG. Let us recall the statement of Theorem \ref{theorem nesterov convex}.

\nesterovconv*

\begin{proof}

   We consider a Lyapunov sequence of the same form as before,
    \[ \L_n = P(n) \E\left[f(x_n) - f(x^*)\right] + \frac12 \E \left[\norm{b(n)(x'_n - x_n) + a(n)(x'_n - x^*)}^2\right] \] 
    where $P(n)$ is some function of $n$, $a(n) = a_0 + a_1 n$, and $b(n) = b_0 + b_1 n$.
    
	Since Nesterov's algorithm is a special case of AGNES, after substituting $\alpha=\eta$, the analysis in steps 1, 2, and 3 of the proof of Theorem \ref{theorem convex} remains valid. With that substitution, we get the following system of inequalities corresponding to step 3,
	\begin{align}
		P(n) &= \eta (b(n) + a_0)b(n) \label{nest_coeff_1}\\
		P(n+1)+k-P(n) &\leq \eta (b(n) + a_0)a_0 \label{nest_coeff_2}\\
		\frac{\eta^2}{2} (b(n) + a_0)^2(1+\sigma^2) &\leq (P(n+1)+k)\eta\left(1-\frac{L(1+\sigma^2)\eta}2\right). \label{nest_coeff_3}
	\end{align}
	Using the definition of $P(n)$ from (\ref{nest_coeff_1}), (\ref{nest_coeff_3}) is equivalent to
	\begin{align*}
		 (1+\sigma^2) &\leq \frac{2(b_1n + b_1 + b_0 + a_0 + k)(b_1n + b_0)\left(1-\frac{L(1+\sigma^2)\eta}2\right)}{(b_1n + b_0 + a_0)^2}
	\end{align*}
	which should still hold in limit as $n \rightarrow \infty$,
	\begin{align*}
		(1+\sigma^2) &\leq \lim_{n\rightarrow \infty} \frac{2(b_1n + b_1 + b_0 + a_0 + k)(b_1n + b_0)\left(1-\frac{L(1+\sigma^2)\eta}2\right)}{(b_1n + b_0 + a_0)^2}\\
		& = 2\left(1-\frac{L(1+\sigma^2)\eta}2\right).
	\end{align*}
	This implies
	\[ \eta \leq \frac{1-\sigma^2}{L(1+\sigma^2)}. \]
	We can choose $a_0 = 2$, $b(n) = n$, and $k=\eta$. Then (\ref{nest_coeff_1}) implies that $P(n) = \eta n(n+2)$. (\ref{nest_coeff_2}) holds because
	\begin{align*}
		P(n+1) + k - P(n) = \eta(2n + 4) = \eta (b(n) + a_0)a_0
	\end{align*} and (\ref{nest_coeff_3}) holds because
	\begin{align*}
		\frac{\eta}{2} (b(n) + a_0)^2(1+\sigma^2) &= \frac{\eta(n+2)^2(1+\sigma^2)}{2} \\
		&\leq \frac{\eta((n+1)(n+3) + 1)(1+\sigma^2)}{2} \\
		&= (P(n+1)+k)\left(1-\frac{L(1+\sigma^2)\eta}2\right).
	\end{align*}
	We have shown that the Lyapunov sequence
	\[ \L_n = \eta n(n+2)\E[f(x_n) - f(x^*) ] + \frac1{2}\E[\norm{n(x'_n - x_n) + 2(x'_n - x^*)}^2], \]
	where $\eta \leq \frac{1-\sigma^2}{L(1+\sigma^2)}$, is monotonically decreasing. It follows that 
	\[ \eta n(n+2)\E[f(x_n) - f(x^*) ] \leq \L_n \leq \L_0 = 2\E[\norm{x_0 - x^*}^2].\qedhere\]
\end{proof}

We emphasize again that this analysis works only if $\sigma < 1$. The condition that $\eta \leq \frac{1-\sigma^2}{L(1+\sigma^2)}$ is imposed by (\ref{nest_coeff_3}) and does not depend on any specific choice of $a_0, b_0$, or $b_1$. On the other hand, (\ref{nest_coeff_1}) forces the rate of convergence to be inversely proportional to $\eta$. This means that as $\sigma$ approaches 1, the step size $\eta$ decreases to zero, and the rate of convergence blows up to infinity. On the other hand, as the proof of Theorem \ref{theorem convex} shows, AGNES does not suffer from this problem. Having an additional parameter enables AGNES to converge even if the noise $\sigma$ is arbitrarily large.

Let us point out how the same techniques used in Theorem \ref{theorem convex} can be adapted to prove convergence $f(x_n)\to \inf f$, even if a global minimizer does not exist. We recall the main statement.

\weaklyconvex*

\begin{proof}
	The first step follows along the same lines as the proof of Theorem \ref{theorem convex appendix} with minor modifications. Note that we did not use the minimizing property of $x^*$ except for Step 5.2. Assume for the moment that $\inf f>-\infty$.
	
	Assume first that $\eps:= \liminf_{n\to \infty} \E[f(x_n)] - \inf f>0$. Select $x^*$ such that $f(x^*) < \inf f + \eps/4$ and define the Lyapunov sequence $\L_n$ just as in the proof of Theorem \ref{theorem convex} with the selected point $x^*$.
	
	We distinguish between two situations. First, assume that $n$ satisfies $\E \big[f(x_n')\big] \geq f(x^*)$. In this case we find that also $\E\big[ f(x_{n+1}) \leq \E \big[f(x_n')\big] \leq f(x^*)$.
	
	On the other hand, assume that $\E \big[f(x_n')\big] \geq f(x^*)$ for $n= 0, \dots, N$. In that case, the proof of Theorem \ref{theorem convex} still applies, meaning that $\E[f(x_N)]$ cannot remain larger than $f(x^*) + \eps/2 $ indefinitely. In either case, we find that there exists $N\in \N$ such that $\E\big[ f(x_N)\big] \leq f(x^*) + \eps/2 < \liminf_{n\to\infty} \E[f(x_n)]$.
	
	Note that the proof of Theorem \ref{theorem convex} applies with $n'\geq n_0$ as a starting point and a non-zero initial velocity $v_n$. The argument therefore shows that, for every $n'\in \N$ there exists $N\in \N$ such that $\E[f(x_N)] \leq \liminf_{n\to\infty} \E[f(x_n)]$. Inserting the definition of the lower limit, we have reached a contradiction.
\end{proof}

    We conjecture that the statement holds with the limit in place of the lower limit, but that it is impossible to guarantee a rate of convergence $O(n^{-\beta})$ {for any $\beta>0$} in this setting. 
When following this strategy, the key question is how far away the point $x^*$ must be chosen. For very flat functions such as
\[
f_\alpha:\R\to \R, \qquad f_\alpha(x) = \begin{cases}
	x^{-\alpha}& x>1\\
	1 +\alpha(1-x) &x\leq 1,
\end{cases}
\]
$x^*$ may be very far away from the initial point $x_0$, and the rate of decay can be excrutiatingly slow if minimizers do not exist. For an easy example, we turn to the continuous time model. The solution to the heavy ball ODE
\[
\begin{pde} x'' &= - \frac 3t\,x' - f_\alpha'(x) &t>1\\ x&= 1 &t=1\\ x' &=-\beta & t=1\end{pde}
\]
is given by 
\[
x(t) = \left(\frac{ 4\,(3+\alpha)}{\alpha(2+\alpha)^2}\right)^\frac{2}{2+\alpha}\, t^{\frac{2}{2+\alpha}}
\]
for $\beta= \frac{2}{2+\alpha} \left(\frac{ 4\,(3+\alpha)}{\alpha(2+\alpha)^2}\right)^\frac{2}{2+\alpha}>0$.
Ignoring the complicated constant factor, we see that 
\[
f_\alpha(x(t)) = x(t) ^{-\alpha} \sim t^{- \frac{2\alpha}{2+\alpha}},
\]
the decay rate can be as close to zero as desired for $\alpha$ close to zero, and indeed \citet{siegel2023qualitative} show that no rate of decay can be guaranteed even beyond the situation of algebraic rates. For comparison, the solution of the gradient flow equation
\[\begin{pde}
	z' &= - f_\alpha'(z) &t>0\\ z&=1 &t=0
\end{pde}\qquad\text{is given by } z(t) = \big(1+ \alpha(2+\alpha)t\big)^{\frac1{2+\alpha}}\quad \Ra\quad f_\alpha(z(t)) \sim t^{- \frac{\alpha}{2+\alpha}}.
\]
Thus, while both the heavy ball ODE and the gradient flow can be made arbitrarily slow in this setting, the heavy ball remains much faster in comparison.

\section{Convergence proofs: strongly convex case}\label{appendix strongly convex}

\subsection{Gradient Descent}

\cite{DBLP:journals/corr/abs-1811-02564, wojtowytsch2021stochasticdiscrete} analyze stochastic gradient descent under the PL condition
\begin{equation}\label{eq pl condition}
    \mu \big( f(x) - \inf f\big) \leq \frac12\,\|\nabla f(x)\|^2\qquad\forall\ x\in \R^m
\end{equation}
and the noise scaling assumption
\[
    \E_\omega\big[ \|g(x,\omega) - \nabla f(x)\|^2\big]\leq \sigma \big( f(x) - \inf f\big)
\]
motivated by Lemma \ref{lemma noise scaling}. The assumption is equivalent to multiplicative noise scaling within a constant since every $L$-smooth function which satisfies a PL condition satisfies
\[
    2\mu\, \big( f(x) - \inf f\big) \leq \|\nabla f(x)\|^2 \leq 2L\, \big( f(x) - \inf f\big).
\]
For completeness, we provide a statement and proof directly in the multiplicative noise scaling regime which attains the optimal constant.

Additionally, we note that strong convexity implies the PL condition. The PL condition holds in many cases where convexity is false, e.g.\
\[
f(x,y) = (y-\sin x)^2, \qquad \|\nabla f\|^2 \geq |\partial_yf|^2 = 4\,f.
\]
The set of minimizers $\{(x,y): y=\sin x\}$ is non-convex, so $f$ cannot be convex. While this result is well-known to the experts, we have been unable to locate a reference and hence provide a proof.

\begin{lemma}
    Assume that $f:\R^m\to\R$ is $\mu$-strongly convex and $C^2$-smooth. Then $f$ satisfies the PL-condition with constant $\mu>0$.
\end{lemma}

\begin{proof}
    Let $x,y\in \mathbb{R}^d$. Strong convexity combined with the Cauchy-Schwartz inequality means that
    \begin{equation}
    \begin{split}
        f(x) - f(y) \leq -\langle \nabla f(x),y-x\rangle - \frac{\mu}{2}\|x - y\|^2 &\leq \|\nabla f(x)\|\|y-x\| - \frac{\mu}{2}\|x - y\|^2\\
        &\leq \max_{z\in\mathbb{R}} \|\nabla f(x)\|z - \frac{\mu}{2}z^2\\
        &= \frac{1}{2\mu} \|\nabla f(x)\|^2.
    \end{split}
    \end{equation}
    Since this is true for $y = x^*$, the result follows.
\end{proof}

Several results in this vein are also collected in \cite[Theorem 2]{karimi2016linear} together with additional generalizations of convexity, but with a suboptimal implication ($\mu$-strongly convex \& $L$-smooth) $\Ra$ $\mu/L$-PL. The additional implication (convexity \& PL) $\Ra$ strong convexity can also be found there.

\begin{theorem}[GD, PL condition]\label{theorem gd strongly convex}
		Assume that $f$ satisfies the PL-condition \eqref{eq pl condition} and that the assumptions laid out in Section \ref{section assumptions} are satisfied. Let $x_n$ be the sequence generated by the gradient descent scheme
		\[
		g_n= g(x_n,\omega_n), \qquad x_{n+1} = x_n - \eta g'_n, \qquad \eta \leq \frac1{L(1+\sigma^2)},
		\]
		where $\omega_1, \omega_2,\dots$ are elements of $\Omega$ which are drawn independently of each other and the initial condition $x_0$.
		Then the estimate 
		\[
		\E\left[f(x_n) - \inf_{x\in\R^m}f(x)\right] \leq \big(1-\mu\eta\big)^n \,\E\big[f(x_0) - \inf f\big]
		\]
		holds for any $n\in\N$. Additionally, the sequence $x_n$ converges to a limiting random variable $x_\infty$ almost surely and in $L^2$ such that $f(x_\infty)\equiv \inf f$ almost surely. 
\end{theorem}

\begin{proof}
    We denote
    \[
    \L_n := \E\big[ f(x_n) - \inf f\big]
    \]
    and compute by Lemma \ref{lemma gd decrease} that
    \begin{align*}
    \L_{n+1} &\leq \E\left[f(x_n) - \frac\eta2 \,\|\nabla f(x_n)\|^2 - \inf f\right]\\
        &\leq \E\left[f(x_n) - \mu\eta \,\big(f(x_n) - f(x^*)\big) - \inf f\right]\\
        &= \big(1-\mu\eta\big) \L_n.
    \end{align*}
    The proof of almost sure convergence is identical to the corresponding argument in \citep[Theorem 2.2]{wojtowytsch2021stochasticdiscrete} and similar in spirit to that of Corollary \ref{corollary almost surely}.
\end{proof}

As usual, the optimal step-size is $\eta = \frac1{L(1+\sigma^2)}$ as used in Figure \ref{figure time complexity}.

\subsection{AGNES and NAG}
Just like the convex case, we first prove Theorem \ref{theorem strongly convex} and set up the Lyapunov sequence with variable coefficients that can be chosen as per the time-stepping scheme. The continuous time analogue in this case is the heavy-ball ODE
     \[
     \begin{pde}\ddot x &= - 2\sqrt\mu\,\dot x - \nabla f(x)&t>0\\\dot x&=0&t=0\\ x&= x_0 &t=0\end{pde}
     \]
     For $\mu$-strongly convex $f$, a simple calculation shows that the Lyapunov function
     \[
     \L(t) = f(x(t)) - f(x^*) + \frac12 \left\|\dot x + \sqrt\mu\big(x(t) - x^*\big)\right\|^2
     \]
     satisfies $\L'(t) \leq -\sqrt\mu \,\L(t)$ and thus
     \[
     f(x(t)) - f(x^*) \leq \L(t) \leq e^{-\sqrt\mu\,t} \L(0) = e^{-\sqrt\mu\,t} \left(f(x_0) - f(x^*) + \frac\mu2 \,\|x_0-x^*\|^2\right).
     \]
     See for instance \cite[Theorem 1]{siegel2019accelerated} for details.

Here, we state and prove a slightly generalized version of Theorem \ref{theorem strongly convex} in the main text. While we assumed an optimal choice of parameters in the main text, we allow for a suboptimal selection here.

\begin{theorem4}[AGNES, strongly convex case -- general version]
In addition to the assumptions in Theorem 3, suppose that $f$ is $\mu$-strongly convex and that
\[
0<\eta \leq \frac1{L(1+\sigma^2)}, \qquad 0 < \psi \leq \sqrt{\frac{\eta}{1+\sigma^2}}, \qquad \rho =  \frac{1-\sqrt\mu\psi}{1+\sqrt\mu\psi}, \qquad\alpha =  \frac{\psi - \eta \sqrt\mu}{1 - \sqrt\mu \psi}\,\psi,
\]
then
\[
\E\big[f(x_n) - f(x^*)\big] \leq \big(1-\sqrt\mu\,\psi\big)^n \,\E\left[f(x_0)-f(x^*) + \frac\mu2 \norm{x_0 - x^*}^2 \right].
\]
\end{theorem4}

	Note that \( \E\left[f(x_0)-f(x^*) + \frac\mu2 \norm{x_0 - x^*}^2 \right] \leq 2 \E\left[f(x_0)-f(x^*) \right] \) due to Lemma \ref{lemma mu-L bound}. A discussion about the set of admissible parameters is provided after the proof. We note several special cases here.
\begin{enumerate}
	\item If $\psi$ is selected optimally as $\sqrt{\eta/ (1+\sigma^2)}$ for $\eta$, the order of decay is $1- \sqrt{\frac{\mu \eta}{1+\sigma^2}}$, strongly resembling Theorem \ref{theorem convex}.
	
	\item If additionally $\eta= 1/(L(1+\sigma^2))$ is chosen optimally, then we recover the decay rate $1- \sqrt{\mu/L} \,/ (1+\sigma^2)$ claimed in the main text.
	
	\item We recover the gradient descent algorithm with the choice $\alpha = 0$ which is achieved for $\psi = \eta\sqrt\mu$. This selection is admissible in our analysis since
	\[
	\sqrt{\mu\eta} \leq \sqrt{\frac\mu L \,\frac{1}{1+\sigma^2}}\leq \sqrt{\frac1{1+\sigma^2}} \quad \Ra\quad
	\eta \sqrt\mu \leq \sqrt{\frac\eta{1+\sigma^2}}. 
	\]
	As expected, the constant of decay is $1- \sqrt\mu\,\psi = 1- \mu\eta$, as achieved in Theorem \ref{theorem gd strongly convex}. In this sense, our analysis of AGNES interpolates fully between the optimal AGNES scheme (a NAG-type scheme in the deterministic case) and (stochastic) gradient descent. However, this proof only applies in the strongly convex setting, but not under a mere PL assumption.
	
	\item If $\mu < L$ -- i.e.\ if $f(x)\not\equiv A + \mu\|x-x^*\|^2$ for some $A\in\R$ and $x_0\in\R^m$ -- then we can choose $0<\psi < \sqrt\mu\,\eta$, corresponding to $\alpha<0$. In this case, the gradient step is sufficiently strong to compensate for momentum taking us in the wrong direction. Needless to say, this is a terrible idea and the rate of convergence is worse than that of gradient descent.
\end{enumerate}	

\begin{proof}
	{\bf Set-up.} Consider the Lyapunov sequence
	\[
	\L_n = \E\big[ f(x_n) - f(x^*)\big] + \frac12\,\E\left[\big\| b(x_n'-x_n) + a(x_n'-x^*)\|^2\right]
	\]
	for constants $b, a$ to be chosen later. We want to show that there exists some decay factor $0<\delta<1$ such that $\L_{n+1} \leq \delta \L_n$.
	
	{\bf Step 1.} Let us consider the first term. Note that
	\begin{align*}
		\E\big[ f(x_{n+1})\big] &= \E\big[ f(x_n' - \eta g'_n)\big] \\
		&\leq \E\big[ f(x_n')\big] - c_{\eta,\sigma,L} \E\big[\|\nabla f(x_n')\|^2\big]
	\end{align*}
	where $c_{\eta,\sigma,L} = \eta \left(1-\frac{L\eta(1+\sigma^2)}2\right)\geq \eta/2$ if $\eta\leq \frac1{L(1+\sigma^2)}$.
	
	{\bf Step 2.} We now turn to the second term and use the definition of $x'_{n+1}$ from (\ref{eq agnesterov}),
\begin{align*}
	b(x'_{n+1} - x_{n+1}) + a(x'_{n+1} - x^*) &= b\rho (x'_{n} - \alpha g'_n - x_n) + a( x_{n+1} + \rho (x'_n - \alpha g'_n - x_n) - x^*)\\
	&= (b+a)\rho (x'_n - \alpha g'_n - x_n) + a(x'_n - \eta g'_n - x^*) \\
	&= (b+a)\rho(x'_n - x_n) + a(x'_n - x^*) -  ((b+a)\rho\alpha + \eta a)g'_n.
\end{align*}
	To simplify notation, we introduce two new dependent variables:
	\[
	c := (b+a)\rho, \qquad \psi:= (b+a)\rho\alpha+\eta a = \alpha c +\eta a.
	\]
	With these variables, we have
	\[
	b\big(x_{n+1}' - x_{n+1}\big) + a \big(x_{n+1}' - x^*\big)  =  c\,(x'_n - x_n) + a(x_n'-x^*) - \psi g'_n.
	\]
	Taking expectation of the square, we find that
	\begin{align*}
		\E&\left[\big\|b\big(x_{n+1}' - x_{n+1}\big) + a \big(x_{n+1}' - x^*\big)\big\|^2\right] \\
		&= c^2\, \E\big[\|x'_n - x_n\|^2\big] + 2a c \,\E\big[ (x'_n - x_n) \cdot (x_n'-x^*)\big] + a^2\E\big[\|(x_n'-x^*)\|^2\big]\\
		&\quad - 2 c\psi\, \E\big[ g'_n\cdot (x_n'-x_n)\big] - 2a\psi \,\E\big[g'_n\cdot (x_n'-x^*)\big] + \psi^2\,\E\big[\|g'_n\|^2\big]\\
		&\leq  c^2\, \E\big[\|x'_n - x_n\|^2\big] + 2 ac \,\E\big[ (x'_n - x_n) \cdot (x_n'-x^*)\big] + a^2\E\big[\|(x_n'-x^*)\|^2\big]\\
		&\quad - 2 c\psi\, \E\big[ \nabla f(x_n')\cdot (x_n'-x_n)\big] - 2a\psi \,\E\big[\nabla f(x_n')\cdot (x_n'-x^*)\big] + \psi^2(1+\sigma^2)\,\E\big[\|\nabla f(x_n')\|^2\big]
	\end{align*}
	
	{\bf Step 3.} We now use strong convexity to deduce that
	\begin{align*}
		\E&\left[\big\|b\big(x_{n+1}' - x_{n+1}\big) + a \big(x_{n+1}' - x^*\big)\big\|^2\right] \\
		& \leq  c^2\, \E\big[\|x'_n - x_n\|^2\big] + 2 ac \,\E\big[ (x'_n - x_n) \cdot (x_n'-x^*)\big] + a^2\E\big[\|(x_n'-x^*)\|^2\big]\\
		&\quad - 2 c\psi \E\left[ f(x_n') - f(x_n) + \frac\mu2 \,\|x_n'-x_n\|^2\right] - 2a\psi \,\E\left[ f(x_n') - f(x^*) + \frac\mu2 \,\|x_n'-x^*\|^2\right]\\
		&\quad + \psi^2(1+\sigma^2)\,\E\big[\|\nabla f(x_n')\|^2\big]\\
		& = ( c^2- c\psi\mu)\,\E\big[\|x'_n - x_n\|^2\big] +2 ac\,\E\big[(x'_n - x_n)\cdot (x_n'-x^*)\big] + \left(a^2- a\psi\mu\right)\,\E\big[\|x_n'-x^*\|^2\big]\\
		&\quad - 2 c\psi \E\left[ f(x_n') - f(x_n) \right] - 2a\psi \,\E\left[ f(x_n') - f(x^*) \right] + \psi^2(1+\sigma^2)\,\E\big[\|\nabla f(x_n')\|^2\big].
	\end{align*}
	
	{\bf Step 4.} We now add the estimates of Steps 1 and 3:
	\begin{align*}
		\L_{n+1} &= \E\left[ f(x_{n+1}) - f(x^*) + \frac12 \,\|b(x'_{n+1} - x_{n+1}) + a(x_{n+1}'-x^*)\|^2\right]\\
		&\leq \left(1 - c\psi - a\psi\right)\,\E\big[ f(x_n')\big] +  c\psi\,\E\big[f(x_n)\big] - (1-a\psi)\,\E\big[f(x^*)\big] \\
		&\quad + \frac12 ( c^2- c\psi\mu)\,\E\big[\|x'_n - x_n\|^2\big] + ac\,\E\big[(x'_n - x_n)\cdot (x_n'-x^*)\big]  \\
		&\quad + \frac12 \left( a^2 - a\psi\mu \right)\,\E\big[\|x_n'-x^*\|^2 \big] + \left(\frac{\psi^2(1+\sigma^2)}2 - c_{\eta,\sigma,L}\right) \,\E\big[ \|\nabla f(x_n')\|^2\big]
	\end{align*}
	We require the coefficient of $\E[f(x'_n)]$ to be zero, i.e.\ \(1 - a\psi = c\psi\), so the inequality simplifies to
	\begin{align*}
		\L_{n+1} &\leq  c\psi\,\E\big[f(x_n) - f(x^*) \big] + \frac12 ( c^2- c\psi\mu)\,\E\big[\|x'_n - x_n\|^2\big]\\
		&\quad + ac\,\E\big[(x'_n - x_n)\cdot (x_n'-x^*)\big]  + \frac12 \left( a^2 - a\psi\mu \right)\,\E\big[\|x_n'-x^*\|^2 \big]\\
		&\quad + \left(\frac{\psi^2(1+\sigma^2)}2 - c_{\eta,\sigma,L}\right) \,\E\big[ \|\nabla f(x_n')\|^2\big].
	\end{align*}
	The smallest decay factor we can get at this point is the coefficient of $\E [f(x_n) - f(x^*)]$. So we hope to show that $\L_{n+1} \leq c\psi \L_n$, which leads to the following system of inequalities on comparing it with the coefficients in the upper bound that we obtained in the previous step,
	\begin{align}
		c &= (b+a)\rho \label{str1}\\
		\psi &= \alpha c + \eta a \label{str2}\\
		(c+a)\psi &= 1 \label{str3}\\
		c^2 -  c\psi\mu &\leq  c\psi b^2 \label{str4}\\
		ac &=  c\psi ab \label{str5}\\
		a^2 - a\psi\mu &\leq  c\psi a^2 \label{str6}\\
		\frac{(1+\sigma^2)\psi^2}{2} &\leq \eta\left(1 - \frac{L(1+\sigma^2)\eta}{2}\right) \label{str7}
	\end{align}
	\textbf{Step 5.} Now we try to choose constants such that the system of inequalities holds. We assume that $\eta \leq \frac{1}{L(1+\sigma^2)}$. Then since $\frac\eta2 \leq \eta\left(1 - \frac{L(1+\sigma^2)\eta}{2}\right)$, for \eqref{str7} it suffices that $(1+\sigma^2)\psi^2 \leq \eta$, i.e.\ 
	\[\psi \leq \sqrt{\frac{\eta}{1+\sigma^2}}. \] 
	Note that \eqref{str5} implies $\psi = 1/b$ and substituting that into \eqref{str3}, we get $c=b-a$. Using this, \eqref{str6} is equivalent to
	\begin{equation*}
		a^2 - a\psi \mu \leq c\psi a^2 = (\frac1\psi - a)\psi a^2 = a^2 - a^2\psi,
	\end{equation*}
	which holds with equality if $a = \sqrt\mu$. \eqref{str4} holds because
	\begin{equation*}
		c - \psi\mu = b - a - \psi \mu \leq b = \psi b^2,
	\end{equation*}
	if $\mu, \psi > 0$. Finally \eqref{str1} implies
	\[ \rho = \frac{b-a}{b+a} = \frac{1-\sqrt\mu\psi}{1+\sqrt\mu\psi}, \]
	and \eqref{str2} implies 
	\[ \alpha = \frac{\frac1b - \eta a}{b-a} = \frac{\psi^2 - \eta \sqrt\mu \psi}{1 - \sqrt\mu \psi}
	 .\]
	With these choices of parameters, $	\L_{n+1} \leq c\psi \L_{n} = (1 - \sqrt\mu\psi)\L_{n}$, and thus
	 \begin{align*}
	  \E[f(x_n) - f(x^*)] &\leq (1 - \sqrt\mu\psi)^n \L_0 \\
	 	&= (1 - \sqrt\mu\psi)^n \E \left[ f(x_0) - f(x^*) + \frac\mu2\norm{x_0 - x^*}^2 \right] \\
	 	&\leq  2(1 - \sqrt\mu\psi)^n \E \left[ f(x_0) - f(x^*)\right],
	 \end{align*} where we have used Lemma \ref{lemma mu-L bound} for strong convexity in the last step. When the parameters are chosen optimally, i.e.\ $\eta = \frac{1}{L(1+\sigma^2)}$ and $\psi = \sqrt{\frac{\eta}{1+\sigma^2}} = \frac{1}{\sqrt{L}(1+\sigma^2)}$, we get $\rho, \alpha$ and the convergence rate as stated in the theorem.
	\end{proof}

	We focus on the meaningful case in which $\sqrt\mu\psi>0$. 
	As discussed in Section \ref{section algorithm}, for given $f,g$ we can replace $L, \sigma$ by larger values $L', \sigma'$ and $\mu$ by a smaller value $\mu'$. Let us briefly explore the effect of these substitutions. 
	The parameter range described in this version of Theorem \ref{theorem strongly convex} can be understood as a three parameter family of AGNES parameters $\eta,\alpha,\rho$ parametrized by $\eta,\psi,\mu'$ and constraints given by $L, \mu,\sigma$ as
	\[
	D:= \left\{(\eta,\psi, \mu') \:\bigg|\: 0 < \eta \leq \frac1{L(1+\sigma^2)}, \:\: 0< \psi \leq \sqrt{\frac{\eta}{1+\sigma^2}}, \:\: 0 < \mu' \leq \mu\right\}.
	\]
	The parameter map is given by
	\[
	(\eta,\psi,\mu') \mapsto (\eta,\rho, \alpha) = \left(\eta, \:\frac{1- \sqrt{\mu'}\psi}{1+\sqrt{\mu'}\psi}, \:\frac{\psi - \eta\sqrt{\mu'}}{1-\sqrt{\mu'}\,\psi}\,\psi\right).
	\]
	We can converesely obtain $\sqrt{\mu'\,}\psi$ from $\rho$ since the function $z\mapsto (1-z)/(1+z)$ is its own inverse and thus $\sqrt{\mu'\,}\psi = \frac{1-\rho}{1+\rho}$. In particular, in terms of the algorithms parameters, the decay rate is
	\[
	1 - \sqrt{\mu'}\psi = 1- \frac{1-\rho}{1+\rho} = \frac{2\rho}{1+\rho}.
	\]
	Furthermore, we see that
	\[
	\alpha = \frac{\psi^2 - \eta\,\sqrt{\mu'}\psi}{1-\sqrt{\mu'}\psi} = \frac{\psi^2 - \eta\,\frac{1-\rho}{1+\rho}}{1- \frac{1-\rho}{1+\rho}} = \frac{1+\rho}{2\rho}\left(\psi^2 - \eta\,\frac{1-\rho}{1+\rho}\right) \quad\LRa \quad \psi = \sqrt{\frac{2\rho}{1+\rho}\alpha + \eta\,\frac{1-\rho}{1+\rho}}
	\]
	since $\psi>0$. Thus, at the cost of a more complicated representation, we could work directly in the parameter variables rather than using the auxiliary quantities $\psi, \mu'$. In particular, both the parameter map and its inverse are continuous on $D$ and its image respectively. Hence,despite the rigid appearance of the parameter selection in Theorem \ref{theorem strongly convex}, there exists an open set of admissible parameters $\eta,\alpha,\rho$ for which we obtain exponentially fast convergence.

	We provide a more general version of Theorem \ref{theorem nesterov strongly} as well. Just as in the convex case, as $\sigma \nearrow 1$, the step size $\eta$ decreases to zero and the theorem fails to guarantee convergence for $\sigma>1$.
\nesterovstrongly*
\begin{proof}
	Consider the Lyapunov sequence
	\[ \L_n = \E\left[f(x_n) - f(x^*)\right] + \frac12 \E\left[\norm{b(x'_n - x_n) + a(x'_n - x^*)}^2\right], \]
	where $a$ and $b$ are to be determined later. Since NAG is a special case of AGNES with $\alpha=\eta$, the first four steps are identical to the proof of Theorem \ref{theorem strongly convex}. We get the following system of inequalities,	
	\begin{align}
		c &= (a+b)\rho \label{nestr1}\\
		\psi &= \eta(a+c) \label{nestr2}\\
		(c+a)\psi &= 1 \label{nestr3}\\
		c^2 - c\psi\mu &\leq b^2 \psi \label{nestr4}\\
		a^2 - a\psi\mu & \leq a^2 \psi \label{nestr5}\\
		ac &= ab c\psi \label{nestr6}\\
		\frac{\psi^2(1+\sigma^2)}2 &\leq \eta \left(1-\frac{L\eta(1+\sigma^2)}2 \right) \label{nestr7}
	\end{align}	
	
	Substituting \eqref{nestr2} into \eqref{nestr3}, we get $(a+c)^2 = \frac1\eta$ and $\psi = \eta/\sqrt\eta = \sqrt\eta$. Thus (\ref{nestr6}) simplifies to
	\[ \frac{1-\sigma^2}{2} \leq 1 - \frac{L\eta(1+\sigma^2)}{2}, \]
	which is equivalent to
	\[ \eta \leq \frac{1-\sigma^2}{L(1+\sigma^2)}. \]
	From \eqref{nestr6}, $b=1/\psi = 1/\sqrt\eta$. The rest of the inequalities can be verified to work with $a=\sqrt\mu, c = b-a, \rho = \frac{b-a}{b+a}$. This shows that $\L_{n+1} \leq c\psi \L_n = (1-\sqrt{\mu\eta})\L_n$. Finally, we get 
	 \begin{align*}
	 	\E[f(x_n)-f(x^*)] &\leq (1-\sqrt{\mu\eta})^n\, \E \left[ f(x_0) - f(x^*) + \frac\mu2\norm{x_0 - x^*}^2 \right] \\
	 	&\leq  2(1 - \sqrt{\mu\eta})^n\, \E \left[ f(x_0) - f(x^*)\right].
	 \end{align*}
\end{proof}

    \subsection{On the role of momentum parameters}\label{appendix momentum parameters}

    Two different AGNES parameters are associated with momentum: $\alpha$ and $\rho$. In this section, we disentangle their respective contributions to keeping AGNES stable for highly stochastic noise.

    For simplicity, first consider the case $f:\R\to\R$, $f(x) = x$ and $g(x) = (1+\sigma N)\,f'(x)$ where $N$ is a standard normal random variable. Then
    \[
    v_{n+1} = \rho(v_n - g'_n) = \dots = -\rho \sum_{i=0}^{n}\rho^{n-i}g'_i.
    \]
    since $v_0=0$. In particular, we note that
    \[
    \E[v_{n+1}] = -\rho\sum_{i=1}^{n}\rho^{n-i}\E[g'_i] = -\rho \sum_{i=1}^{n}\rho^{n-i} = -\rho \frac{1- \rho^{n+1}}{1-\rho}. 
    \]
    and
    \begin{align*}
    \E\left[\left|v_{n+1} - \left(-\rho \frac{1- \rho^{n+1}}{1-\rho}\right)\right|^2\right] & = \rho^2\E\left[\left|\sum_{i=0}^n \rho^{n-i}(g'_i - 1)\right|^2\right] = \sigma^2\rho^2\sum_{i=0}^n \rho^{2(n-i)}\E\big[ |g'_i -1|^2\big]\\
    &= \sigma^2\rho^2\sum_{i=0}^n \rho^{2(n-i)} = \sigma^2\rho^2\,\frac{1- \rho^{2(n+1)}}{1-\rho^2}
    \end{align*}
    due to the independence of different gradient estimators between time steps. In particular, we see that
    \begin{enumerate}
        \item as $\rho$ becomes closer to $1$, the eventual magnitude of the velocity variable increases as $\lim_{n\to\infty}\E\|v_n\|= \frac\rho{1-\rho}$. 

        \item as $\rho$ becomes closer to $1$, the eventual variance of the velocity variable increases as $\lim_{n\to\infty}\E\big[\|v_n - \E[v_n]\|^2\big]= \frac{\rho^2}{1-\rho^2}$. 

        \item the noise in the normalized velocity estimate asymptotically satisfies
        \[
        \lim_{n\to\infty}\E\left[\left\|\frac{v_n - \E[v_n]}{\E[\|v_n\|]}\right\|^2\right] = \sigma^2\frac{(1-\rho)^2}{1-\rho^2} = \sigma^2\frac{(1-\rho)^2}{(1-\rho)(1+\rho)} = \sigma^2\frac{1-\rho}{1+\rho}
        \]
    \end{enumerate}

    Thus, if $\rho$ is closer to $1$, both the magnitude and the variance of the velocity variable increase, but the the relative importance of noise approaches zero as $\rho\to 1$. This is not surprising -- if $\rho$ is close to $1$, the sequence $\rho^n$ decays much slower than if $\rho$ is small. Gradient estimates from different times enter at a similar scale and cancellations can occur easily. As the influence of past gradients remains large, we say that the momentum variable has a `long memory'.
    
    Of course, when minimizing a non-linear function $f$, the gradient is not constant, and we face a trade-off:
    \begin{enumerate}
        \item A long memory allows us to cancel random oscillations in the gradient estimates more easily.
        \item A long memory also means we compute with more out-of-date gradient estimates from points much further in the past along the trajectory.
    \end{enumerate}

    Naturally, the relative importance of the first point increases with the stochasticity $\sigma$ of the gradient estimates. Even if the gradient evaluations are deterministic, we benefit from integrating historic information gained throughout the optimization process, but the rate at which we `forget' outdated information is much higher.

    Thus the parameter $\rho$ corresponds to the rate at which we forget old information. It also impacts the magnitude of the velocity variable. The parameter $\alpha$ compensates for the scaling of $v_n$ with $1/(1-\rho)$. We can think of $\rho$ as governing the rate at which we forget past gradients, and $\alpha$ as a measure of the confidence with which we integrate past gradient information into time-steps for $x$.
    
    Let us explore this relationship in strongly convex optimization. In Theorem \ref{theorem strongly convex}, the optimal choice of hyper-parameters is given by $\eta = \frac1{L(1+\sigma^2)}$ and
    \[
     \alpha = \frac{1-\sqrt{\mu/L}}{1-\sqrt{\mu/L}+ \sigma^2}\eta, \qquad \rho = \frac{\sqrt L\,(1+\sigma^2) -\sqrt\mu}{\sqrt L(1+\sigma^2) + \sqrt\mu}= 1 - \frac{2\sqrt\mu}{\sqrt{L}(1+\sigma^2)+\sqrt\mu}.
    \]
    Let us consider the simplified regime $\mu\ll L$ in which
    \[
    \alpha\approx \frac{\eta}{1+\sigma^2}, \qquad \rho \approx 1 - 2\,\sqrt{\frac\mu L}\,\frac{1}{1+\sigma^2} \quad \Ra\quad \frac{\alpha}{1-\rho} = \frac{\eta}{2\,\sqrt{\mu/L}}.
    \]
    In particular, we note: The larger $\sigma$, the closer $\rho$ is to $1$, i.e.\ the longer the memory we keep. The relative importance of the momentum step compared to the gradient step, on the other hand, remains constant, depending only on the `condition number' $L/\mu$.

    We note that also in the convex case, high stochasticity forces $n_0$ to be large, meaning that $\rho_n$ is always close to $1$.
    Notably for generic non-convex objective functions, it is unclear that past gradients along the trajectory would carry useful information, as there is no discernible geometric relationship between gradients at different points. This mirrors an observation of Appendix \ref{appendix non-convex}, just after Theorem \ref{theorem non-convex}.

	\section{AGNES in non-convex optimization}\label{appendix non-convex}

    We consider the case of non-convex optimization. In the deterministic setting, momentum methods for non-convex optimization have recently been studied by \cite{diakonikolas2021generalized}. We note that the algorithm may perform worse than stochastic gradient descent, but that for suitable parameters, the performance is comparable to that of SGD within a constant factor. 

\begin{restatable}[Non-convex case]{theorem}{nonconvex}\label{theorem non-convex}
	Assume that $f$ satisfies the assumptions laid out in Section \ref{section assumptions}.
	Let $\eta, \alpha,\rho$ be such that 
	\[
	\eta \leq \frac1{L(1+\sigma^2)}, \qquad \alpha < \frac{\eta}{1+\sigma^2}, \qquad (L\alpha+1)\rho^2\leq 1.
	\]
	Then
	\[
	\min_{0\leq i\leq n} \E\big[\|\nabla f(x_i)\|^2\big] \leq \frac{ 2\,\E\left[ f(x_0) - \inf f + \frac1{\alpha\rho^2} \,\|v_0\|^2\right]}{(n+1)\,\left(\eta - \alpha (1+\sigma^2)\right)}.
	\]
\end{restatable}
{If $v_0=0$, the bound is minimal for gradient descent (i.e.\ $\alpha=0$) since the decay factor $\eps = \eta - \alpha(1+\sigma^2)$ is maximal.}

\begin{proof}
	Consider
	\[
	\L_n =\E\left[ f(x_n) + \frac \lambda2 \,\|x'_n - x_n\|^2\right].
	\]
	for a parameter $\lambda>0$ to be fixed later. We have
	\begin{align*}
		\E\big[f(x_{n+1})\big] &\leq \E\big[f(x_n')\big] - \frac\eta2 \,\E\big[\|\nabla f(x_n')\|^2\big]\\
		&\leq \E\left[ f(x_n) + \nabla f(x_n') \cdot(x'_n - x_n) + \frac{L\,\alpha^2}2\|v_n\|^2 - \frac\eta2 \,\|\nabla f(x_n')\|^2\right]\\
		\E\big[ \|x_{n+1}' - x_{n+1}\|^2\big] &= \rho^2 \E\big[ \| (x'_n - x_n)\|^2 - 2\alpha\,(x'_n - x_n)\cdot \,g'_n + \alpha^2\,\|g'_n\|^2\big]
	\end{align*}
	by Lemmas \ref{lemma momentum decrease} and \ref{lemma gd decrease}. We deduce that
	\begin{align*}
		\L_{n+1} &\leq \E\big[f(x_n)\big] + \left(1-\lambda\alpha\rho^2\right) \E\big[\nabla f(x_n') \cdot (x'_n - x_n)\big] + \frac{L + \lambda\rho^2}{2} \,\E\big[\|x'_n - x_n\|^2\big]\\
		&\qquad \qquad + \frac{\lambda \rho^2\alpha \cdot\alpha (1+\sigma^2)-\eta}2 \,\E\big[\|\nabla f(x_n')\|^2\big]\\
		&\leq \L_n  + \frac{\lambda \rho^2\alpha \cdot\alpha (1+\sigma^2) -\eta}2 \,\E\big[\|\nabla f(x_n')\|^2\big]
	\end{align*}
	under the conditions 
	\[
	1-\lambda\alpha\rho^2 = 0, \qquad L + \lambda\rho^2 \leq \lambda.
	\]
	The first condition implies that $\lambda = (\alpha \rho^2)^{-1}$, so the second one reduces to
	\[
	(1-\rho^2)\lambda = \frac{1-\rho^2}{\rho^2\alpha } \geq L \quad\LRa\quad 1- \rho^2 \geq L\rho^2\alpha \quad\LRa\quad 1 \geq (1+L\alpha)\rho^2.
	\]
	Finally, we consider the last equation. If
	\[
	\eps:= \eta - \lambda \rho^2\alpha \cdot\alpha (1+\sigma^2) = \eta - \alpha (1+\sigma^2)>0,
	\] 
	then we find that
	\[
	\E\left[ f(x_0) + \frac1{\alpha\rho^2} \,\|v_0\|^2- \inf f\right] \geq \L_1 - \L_{n+1} = \sum_{i=0}^n (\L_i - \L_{i+1}) \geq \frac\eps2 \sum_{i=1}^n \E\big[\|\nabla f(x_i)\|^2\big]
	\]
	and hence
	\[
	\min_{0\leq i\leq n} \E\big[\|\nabla f(x_i)\|^2\big] \leq \frac1{n+1} \sum_{i=1}^n \E\big[\|\nabla f(x_i)\|^2\big] \leq \frac{ 2\E\left[ f(x_0) - \inf f + \frac1{\alpha\rho^2} \,\|v_0\|^2\right]}{\eps (n+1)}.\qedhere
	\]
\end{proof}

    \section{Proof of Lemma \ref{lemma noise scaling}: Scaling intensity of minibatch noise}\label{appendix noise}

    In this appendix, we provide theoretical justification for the multiplicative noise scaling regime considered in this article. Recall our main statement:

    \noise*

    \begin{proof}
    Since $\nabla \Risk = \frac1n\sum_{i=1}^n \nabla \ell_i$, we observe that
    \begin{align*}
    \frac1n\sum_{i=1}^n\left\|\nabla \ell_i - \nabla \Risk\right\|^2 &\leq \frac1n\sum_{i=1}^n\left\|\nabla \ell_i\right\|^2
    \end{align*}
    as the average of a quantity is the unique value which minimizes the mean square discrepancy: $\E X = \argmin_{a\in \R} \E\big[ |X-a|^2\big]$. We further find by H\"older's inequality that
    \begin{align*}
    \frac1n\sum_{i=1}^n\left\|\nabla \ell_i\right\|^2 &= \frac1n \sum_{i=1}^n \left\|\sum_{j=1}^k 2\big(h_j(w, x_i) - y_{i,j}\big)\nabla_wh_j(w,x_i)\right\|^2\\
    &\leq \frac4n \sum_{i=1}^n \left(\sum_{j=1}^k 
    \big(h_j(w, x_i) - y_{i,j}\big)^2\right)\left( \sum_{j=1}^k\big\|\nabla_w h_j(w,x_i)\big\|_2^2\right)\\
    &= \frac4n \sum_{i=1}^n \|h(w,x_i)-y_i\|_2^2\,\|\nabla_w h(w,x_i)\|^2\\
    &\leq 4C^2\big(1+\|w\|^2\big)^{2p}\,\frac1n \sum_{i=1}^n \|h(w,x_i) - y_i\|^2_2\\
    &=  4C^2\big(1+\|w\|^2\big)^{2p}\,\Risk(w).
    \end{align*}
    \end{proof}

    \section{Implementation aspects}\label{appendix implementation}

    We discuss some implementation in this section. All the code used for the experiments in the paper has been provided in the supplementary materials.
    The experiments in section Section \ref{section numerical} and Appendix \ref{appendix simulations} were run on Google Colab for compute time less than an hour. The experiments in Section \ref{section regression} were run on a laptop CPU with compute time less than an hour. The experiments in Sections \ref{section image classification} and \ref{hyperparameter comparison} were run on a single current generation GPU in a local cluster for up to 50 hours.  An additional compute of no more than 200 hours on a single GPU was used for experiments which were ultimately not used in the submitted version.
    
    \subsection{The last iterate}
    
    All neural-network based experiments were performed using the PyTorch library. Gradient-based optimizers in PyTorch and TensorFlow are implemented in such a way that gradients are computed outside of the optimizer and the point returned by an optimizer step is the point for the next gradient evaluation. This strategy facilitates the manual manipulation of gradients by scaling, clipping or masking to train only a subset of the network parameters.

    The approach is theoretically justified for SGD. Guarantees for NAG and AGNES, on the other hand, are given for $f(x_n)$ rather than $f(x_n')$, i.e.\ not at the point where the gradient is evaluated. A discrepancy arises between theory and practice.\footnote{\ For instance, the implementations of NAG in PyTorch and Tensorflow return $x_n'$ rather than $x_n$.}\ In Algorithm \ref{descent algorithm}, this discrepancy is resolved by taking a final gradient descent step in the last time step and returning the sequence $x_n'$ at intermediate steps. In our numerical experiments, we did not include the final gradient descent step. Skipping the gradient step in particular allows for an easier continuation of simulations beyond the initially specified stopping time, if so desired. We do not anticipate major differences under realistic circumstances. This can be justified analytically in convex and strongly convex optimization, at least for a low learning rate.

        \begin{lemma}
        If $\eta< \frac1{3L}$, then
        \[
        \E\big[ f(x_n') - f(x^*)\big]\leq  \frac{\E[f(x_{n+1})-f(x^*)]}{1-3L\eta}.
        \]
        \end{lemma}
        
        \begin{proof}
        By essentially the same proof as Lemma \ref{lemma gd decrease}, we have
        \[
        \E\left[f(x_n') - \frac{3\eta}2 \|\nabla f(x_n')\|^2 \right]\leq\E\big[ f(x_{n+1})\big] \leq \E\left[f(x_n') - \frac\eta2 \|\nabla f(x_n')\|^2\right],
        \]
        since the correction term to linear approximation is bounded by the $L$-Lipschitz continuity of $\nabla f$ both from above and below. Recall furthermore that
        \[
        \|\nabla f(x)\|^2 \leq 2L\,\big(f(x) - f(x^*)\big)
        \]
        for all $L$-smooth functions. Thus 
        \[
        (1-3L\eta) \E \big[f(x_n')-f(x^*)\big] \leq \E\left[ f(x_n') - \frac{3\eta}2 \|\nabla f(x_n')\|^2 \right]\leq \E\big[f(x_{n+1})\big].
        \]
        In particular, if $1-3L\eta>0$, then
        \[
        \E\big[ f(x_n') -f(x^*)\big] \leq \frac1{1-3L\eta}\,\E[f(x_{n+1})-f(x^*)]. \qedhere
        \]
        \end{proof}

        The condition $\eta<1/(3L)$ is guaranteed if the stochastic noise scaling satisfies $\sigma> \sqrt 2$ since then $1-3L\eta \geq 1-\frac3{1+\sigma^2}$. For $\eta = 1/((1+\sigma^2)L$, we than find that
        \[
        \E\big[ f(x_n') - f(x^*)\big]\leq  \frac{\E[f(x_{n+1})-f(x^*)]}{1-\frac{3}{1+\sigma^2}} = \frac{\sigma^2 + 1}{\sigma^2-2} \, \E[f(x_{n+1})-f(x^*)].
        \]

    \subsection{Weight decay}\label{appendix weight decay}

    Weight decay is a machine learning tool which controls the magnitude of the coefficients of a neural network. In the simplest SGD setting, weight decay takes the form of a modified update step
    \[
    x_{n+1} = (1-\lambda\eta)x_n - \eta g_n
    \]
    for $\lambda>0$. 
    A gradient flow is governed by (1) an energy to be minimized and (2) an energy dissipation mechanism \citep{peletier2014variational}. It is known that different energy/dissipation pairings may induce the same dynamics -- for instance, \cite{doi:10.1137/S0036141096303359} show that the heat equation is both the $L^2$-gradient flow of the Dirichlet energy and the Wasserstein gradient flow of the entropy function. 

    In this language, weight decay can be interpreted in two different ways:
    \begin{enumerate}
        \item We minimize a modified objective function $x\mapsto f(x) + \frac{\lambda}{2} \,\|x\|^2$ which includes a Tikhonov regularizer. The gradient estimates are stochastic for $f$ and deterministic for the regularizer. This perspective corresponds to including weight decay as part of the {\em energy}.

        \item We dynamically include a confinement into the optimizer which pushes back against large values of $x_n$. This perspective corresponds to including weight decay as part of the {\em dissipation}.
    \end{enumerate}    
    In GD, both perspectives lead to the same optimization algorithm.
    In advanced minimizers, the two perspectives no longer coincide. For Adam, \cite{loshchilov2018fixing,loshchilovdecoupled} initiated a debate on the superior strategy of including weight decay. We note that the two strategies do not coincide for AGNES, but do not comment on the superiority of one over the other:

\begin{enumerate}
    \item Treating weight decay as a dynamic property of the optimizer leads to an update rule like
    \[
    x_n' = x_n +\alpha v_n, \qquad v_{n+1} = \rho\big(v_n -  g'_n\big), \qquad x_{n+1} = (1-\lambda\eta) x_n' - \eta g'_n.
    \]

    \item Treating weight decay as a component of the objective function to be minimized leads to the update rule
    \[
        x_n' = x_n +\alpha v_n, \qquad v_{n+1} = \rho\left(v_n -  g'_n - {\lambda}\, \,x_n'\right), \qquad x_{n+1} = (1-\lambda\eta) x_n' - \eta g'_n.
    \]
\end{enumerate}
In our numerical experiments, we choose the second approach, viewing weight decay as a property of the objective function rather than the dissipation. This coincides with the approach taken by the SGD (and SGD with momentum) optimizer as well as Adam (but not AdamW).

\newpage
\section*{NeurIPS Paper Checklist}

\begin{enumerate}

\item {\bf Claims}
    \item[] Question: Do the main claims made in the abstract and introduction accurately reflect the paper's contributions and scope?
    \item[] Answer: \answerYes{} 
    \item[] Justification: We wrote the abstract and introduction with the goal to summarize our main contributions accurately and precisely.
    \item[] Guidelines:
    \begin{itemize}
        \item The answer NA means that the abstract and introduction do not include the claims made in the paper.
        \item The abstract and/or introduction should clearly state the claims made, including the contributions made in the paper and important assumptions and limitations. A No or NA answer to this question will not be perceived well by the reviewers. 
        \item The claims made should match theoretical and experimental results, and reflect how much the results can be expected to generalize to other settings. 
        \item It is fine to include aspirational goals as motivation as long as it is clear that these goals are not attained by the paper. 
    \end{itemize}

\item {\bf Limitations}
    \item[] Question: Does the paper discuss the limitations of the work performed by the authors?
    \item[] Answer: \answerYes{} 
    \item[] Justification: We compare the algorithm proposed to commonly used methods both in convex optimization and deep learning. We dedicate Section 3 to the derivation of the noise modelling assumption and illustrate the heuristics which are being made.
    \item[] Guidelines:
    \begin{itemize}
        \item The answer NA means that the paper has no limitation while the answer No means that the paper has limitations, but those are not discussed in the paper. 
        \item The authors are encouraged to create a separate "Limitations" section in their paper.
        \item The paper should point out any strong assumptions and how robust the results are to violations of these assumptions (e.g., independence assumptions, noiseless settings, model well-specification, asymptotic approximations only holding locally). The authors should reflect on how these assumptions might be violated in practice and what the implications would be.
        \item The authors should reflect on the scope of the claims made, e.g., if the approach was only tested on a few datasets or with a few runs. In general, empirical results often depend on implicit assumptions, which should be articulated.
        \item The authors should reflect on the factors that influence the performance of the approach. For example, a facial recognition algorithm may perform poorly when image resolution is low or images are taken in low lighting. Or a speech-to-text system might not be used reliably to provide closed captions for online lectures because it fails to handle technical jargon.
        \item The authors should discuss the computational efficiency of the proposed algorithms and how they scale with dataset size.
        \item If applicable, the authors should discuss possible limitations of their approach to address problems of privacy and fairness.
        \item While the authors might fear that complete honesty about limitations might be used by reviewers as grounds for rejection, a worse outcome might be that reviewers discover limitations that aren't acknowledged in the paper. The authors should use their best judgment and recognize that individual actions in favor of transparency play an important role in developing norms that preserve the integrity of the community. Reviewers will be specifically instructed to not penalize honesty concerning limitations.
    \end{itemize}

\item {\bf Theory Assumptions and Proofs}
    \item[] Question: For each theoretical result, does the paper provide the full set of assumptions and a complete (and correct) proof?
    \item[] Answer: \answerYes{} 
    \item[] Justification: All assumptions are summarized in Section \ref{section assumptions}. Wherever additional assumptions are made, they are stated clearly in the proof. Complete and correct proofs for all the lemmas and theorems are provided in the appendices.
    \item[] Guidelines:
    \begin{itemize}
        \item The answer NA means that the paper does not include theoretical results. 
        \item All the theorems, formulas, and proofs in the paper should be numbered and cross-referenced.
        \item All assumptions should be clearly stated or referenced in the statement of any theorems.
        \item The proofs can either appear in the main paper or the supplemental material, but if they appear in the supplemental material, the authors are encouraged to provide a short proof sketch to provide intuition. 
        \item Inversely, any informal proof provided in the core of the paper should be complemented by formal proofs provided in appendix or supplemental material.
        \item Theorems and Lemmas that the proof relies upon should be properly referenced. 
    \end{itemize}

    \item {\bf Experimental Result Reproducibility}
    \item[] Question: Does the paper fully disclose all the information needed to reproduce the main experimental results of the paper to the extent that it affects the main claims and/or conclusions of the paper (regardless of whether the code and data are provided or not)?
    \item[] Answer: \answerYes{} 
    \item[] Justification: All code from experiments is provided in the supplementary materials.
    \item[] Guidelines:
    \begin{itemize}
        \item The answer NA means that the paper does not include experiments.
        \item If the paper includes experiments, a No answer to this question will not be perceived well by the reviewers: Making the paper reproducible is important, regardless of whether the code and data are provided or not.
        \item If the contribution is a dataset and/or model, the authors should describe the steps taken to make their results reproducible or verifiable. 
        \item Depending on the contribution, reproducibility can be accomplished in various ways. For example, if the contribution is a novel architecture, describing the architecture fully might suffice, or if the contribution is a specific model and empirical evaluation, it may be necessary to either make it possible for others to replicate the model with the same dataset, or provide access to the model. In general. releasing code and data is often one good way to accomplish this, but reproducibility can also be provided via detailed instructions for how to replicate the results, access to a hosted model (e.g., in the case of a large language model), releasing of a model checkpoint, or other means that are appropriate to the research performed.
        \item While NeurIPS does not require releasing code, the conference does require all submissions to provide some reasonable avenue for reproducibility, which may depend on the nature of the contribution. For example
        \begin{enumerate}
            \item If the contribution is primarily a new algorithm, the paper should make it clear how to reproduce that algorithm.
            \item If the contribution is primarily a new model architecture, the paper should describe the architecture clearly and fully.
            \item If the contribution is a new model (e.g., a large language model), then there should either be a way to access this model for reproducing the results or a way to reproduce the model (e.g., with an open-source dataset or instructions for how to construct the dataset).
            \item We recognize that reproducibility may be tricky in some cases, in which case authors are welcome to describe the particular way they provide for reproducibility. In the case of closed-source models, it may be that access to the model is limited in some way (e.g., to registered users), but it should be possible for other researchers to have some path to reproducing or verifying the results.
        \end{enumerate}
    \end{itemize}

\item {\bf Open access to data and code}
    \item[] Question: Does the paper provide open access to the data and code, with sufficient instructions to faithfully reproduce the main experimental results, as described in supplemental material?
    \item[] Answer: \answerYes{} 
    \item[] Justification: All code as well as the synthetically generated data used for the regression experiments are provided in the supplementary materials.
    \item[] Guidelines:
    \begin{itemize}
        \item The answer NA means that paper does not include experiments requiring code.
        \item Please see the NeurIPS code and data submission guidelines (\url{https://nips.cc/public/guides/CodeSubmissionPolicy}) for more details.
        \item While we encourage the release of code and data, we understand that this might not be possible, so “No” is an acceptable answer. Papers cannot be rejected simply for not including code, unless this is central to the contribution (e.g., for a new open-source benchmark).
        \item The instructions should contain the exact command and environment needed to run to reproduce the results. See the NeurIPS code and data submission guidelines (\url{https://nips.cc/public/guides/CodeSubmissionPolicy}) for more details.
        \item The authors should provide instructions on data access and preparation, including how to access the raw data, preprocessed data, intermediate data, and generated data, etc.
        \item The authors should provide scripts to reproduce all experimental results for the new proposed method and baselines. If only a subset of experiments are reproducible, they should state which ones are omitted from the script and why.
        \item At submission time, to preserve anonymity, the authors should release anonymized versions (if applicable).
        \item Providing as much information as possible in supplemental material (appended to the paper) is recommended, but including URLs to data and code is permitted.
    \end{itemize}

\item {\bf Experimental Setting/Details}
    \item[] Question: Does the paper specify all the training and test details (e.g., data splits, hyperparameters, how they were chosen, type of optimizer, etc.) necessary to understand the results?
    \item[] Answer: \answerYes{} 
    \item[] Justification: All experimental settings are described in the article and its supplementary materials. They can also be inferred in the code provided.
    \item[] Guidelines:
    \begin{itemize}
        \item The answer NA means that the paper does not include experiments.
        \item The experimental setting should be presented in the core of the paper to a level of detail that is necessary to appreciate the results and make sense of them.
        \item The full details can be provided either with the code, in appendix, or as supplemental material.
    \end{itemize}

\item {\bf Experiment Statistical Significance}
    \item[] Question: Does the paper report error bars suitably and correctly defined or other appropriate information about the statistical significance of the experiments?
    \item[] Answer: \answerYes{}{} 
    \item[] Justification: All experiments were repeated multiple times. We provide means and standard deviations over all runs.
    \item[] Guidelines:
    \begin{itemize}
        \item The answer NA means that the paper does not include experiments.
        \item The authors should answer "Yes" if the results are accompanied by error bars, confidence intervals, or statistical significance tests, at least for the experiments that support the main claims of the paper.
        \item The factors of variability that the error bars are capturing should be clearly stated (for example, train/test split, initialization, random drawing of some parameter, or overall run with given experimental conditions).
        \item The method for calculating the error bars should be explained (closed form formula, call to a library function, bootstrap, etc.)
        \item The assumptions made should be given (e.g., Normally distributed errors).
        \item It should be clear whether the error bar is the standard deviation or the standard error of the mean.
        \item It is OK to report 1-sigma error bars, but one should state it. The authors should preferably report a 2-sigma error bar than state that they have a 96\% CI, if the hypothesis of Normality of errors is not verified.
        \item For asymmetric distributions, the authors should be careful not to show in tables or figures symmetric error bars that would yield results that are out of range (e.g. negative error rates).
        \item If error bars are reported in tables or plots, The authors should explain in the text how they were calculated and reference the corresponding figures or tables in the text.
    \end{itemize}

\item {\bf Experiments Compute Resources}
    \item[] Question: For each experiment, does the paper provide sufficient information on the computer resources (type of compute workers, memory, time of execution) needed to reproduce the experiments?
    \item[] Answer: \answerYes{} 
    \item[] Justification: The information on compute resources is provided in the appendix.
    \item[] Guidelines:
    \begin{itemize}
        \item The answer NA means that the paper does not include experiments.
        \item The paper should indicate the type of compute workers CPU or GPU, internal cluster, or cloud provider, including relevant memory and storage.
        \item The paper should provide the amount of compute required for each of the individual experimental runs as well as estimate the total compute. 
        \item The paper should disclose whether the full research project required more compute than the experiments reported in the paper (e.g., preliminary or failed experiments that didn't make it into the paper). 
    \end{itemize}
    
\item {\bf Code Of Ethics}
    \item[] Question: Does the research conducted in the paper conform, in every respect, with the NeurIPS Code of Ethics \url{https://neurips.cc/public/EthicsGuidelines}?
    \item[] Answer: \answerYes{} 
    \item[] Justification: The work presented here is primarily theoretical. No human subjects were involved. The datasets used are standard benchmark datasets (MNIST, CIFAR-10) or purely synthetic. No direct social consequences are anticipated.
    \item[] Guidelines:
    \begin{itemize}
        \item The answer NA means that the authors have not reviewed the NeurIPS Code of Ethics.
        \item If the authors answer No, they should explain the special circumstances that require a deviation from the Code of Ethics.
        \item The authors should make sure to preserve anonymity (e.g., if there is a special consideration due to laws or regulations in their jurisdiction).
    \end{itemize}

\item {\bf Broader Impacts}
    \item[] Question: Does the paper discuss both potential positive societal impacts and negative societal impacts of the work performed?
    \item[] Answer: \answerNA{} 
    \item[] Justification: The main contribution of the work is an algorithm for smooth convex optimization. Foundational as the topic at large may be in various fields, it is impossible to link directly to societal impact.
    \item[] Guidelines:
    \begin{itemize}
        \item The answer NA means that there is no societal impact of the work performed.
        \item If the authors answer NA or No, they should explain why their work has no societal impact or why the paper does not address societal impact.
        \item Examples of negative societal impacts include potential malicious or unintended uses (e.g., disinformation, generating fake profiles, surveillance), fairness considerations (e.g., deployment of technologies that could make decisions that unfairly impact specific groups), privacy considerations, and security considerations.
        \item The conference expects that many papers will be foundational research and not tied to particular applications, let alone deployments. However, if there is a direct path to any negative applications, the authors should point it out. For example, it is legitimate to point out that an improvement in the quality of generative models could be used to generate deepfakes for disinformation. On the other hand, it is not needed to point out that a generic algorithm for optimizing neural networks could enable people to train models that generate Deepfakes faster.
        \item The authors should consider possible harms that could arise when the technology is being used as intended and functioning correctly, harms that could arise when the technology is being used as intended but gives incorrect results, and harms following from (intentional or unintentional) misuse of the technology.
        \item If there are negative societal impacts, the authors could also discuss possible mitigation strategies (e.g., gated release of models, providing defenses in addition to attacks, mechanisms for monitoring misuse, mechanisms to monitor how a system learns from feedback over time, improving the efficiency and accessibility of ML).
    \end{itemize}
    
\item {\bf Safeguards}
    \item[] Question: Does the paper describe safeguards that have been put in place for responsible release of data or models that have a high risk for misuse (e.g., pretrained language models, image generators, or scraped datasets)?
    \item[] Answer: \answerNA{} 
    \item[] Justification: The main contribution of the work is theoretical and no data or models with a high risk for misuse are produced.
    \item[] Guidelines:
    \begin{itemize}
        \item The answer NA means that the paper poses no such risks.
        \item Released models that have a high risk for misuse or dual-use should be released with necessary safeguards to allow for controlled use of the model, for example by requiring that users adhere to usage guidelines or restrictions to access the model or implementing safety filters. 
        \item Datasets that have been scraped from the Internet could pose safety risks. The authors should describe how they avoided releasing unsafe images.
        \item We recognize that providing effective safeguards is challenging, and many papers do not require this, but we encourage authors to take this into account and make a best faith effort.
    \end{itemize}

\item {\bf Licenses for existing assets}
    \item[] Question: Are the creators or original owners of assets (e.g., code, data, models), used in the paper, properly credited and are the license and terms of use explicitly mentioned and properly respected?
    \item[] Answer: \answerYes{} 
    \item[] Justification: We use the MNIST and CIFAR-10 datasets, which are cited accurately. We also use an implementation of ResNets, for which we cite the GitHub repository and reproduce the license terms in the code provided.
    \item[] Guidelines:
    \begin{itemize}
        \item The answer NA means that the paper does not use existing assets.
        \item The authors should cite the original paper that produced the code package or dataset.
        \item The authors should state which version of the asset is used and, if possible, include a URL.
        \item The name of the license (e.g., CC-BY 4.0) should be included for each asset.
        \item For scraped data from a particular source (e.g., website), the copyright and terms of service of that source should be provided.
        \item If assets are released, the license, copyright information, and terms of use in the package should be provided. For popular datasets, \url{paperswithcode.com/datasets} has curated licenses for some datasets. Their licensing guide can help determine the license of a dataset.
        \item For existing datasets that are re-packaged, both the original license and the license of the derived asset (if it has changed) should be provided.
        \item If this information is not available online, the authors are encouraged to reach out to the asset's creators.
    \end{itemize}

\item {\bf New Assets}
    \item[] Question: Are new assets introduced in the paper well documented and is the documentation provided alongside the assets?
    \item[] Answer: \answerNA{}{} 
    \item[] Justification: No new assets are released.
    \item[] Guidelines:
    \begin{itemize}
        \item The answer NA means that the paper does not release new assets.
        \item Researchers should communicate the details of the dataset/code/model as part of their submissions via structured templates. This includes details about training, license, limitations, etc. 
        \item The paper should discuss whether and how consent was obtained from people whose asset is used.
        \item At submission time, remember to anonymize your assets (if applicable). You can either create an anonymized URL or include an anonymized zip file.
    \end{itemize}

\item {\bf Crowdsourcing and Research with Human Subjects}
    \item[] Question: For crowdsourcing experiments and research with human subjects, does the paper include the full text of instructions given to participants and screenshots, if applicable, as well as details about compensation (if any)? 
    \item[] Answer: \answerNA{} 
    \item[] Justification: The paper does not involve crowdsourcing nor research with human subjects.
    \item[] Guidelines:
    \begin{itemize}
        \item The answer NA means that the paper does not involve crowdsourcing nor research with human subjects.
        \item Including this information in the supplemental material is fine, but if the main contribution of the paper involves human subjects, then as much detail as possible should be included in the main paper. 
        \item According to the NeurIPS Code of Ethics, workers involved in data collection, curation, or other labor should be paid at least the minimum wage in the country of the data collector. 
    \end{itemize}

\item {\bf Institutional Review Board (IRB) Approvals or Equivalent for Research with Human Subjects}
    \item[] Question: Does the paper describe potential risks incurred by study participants, whether such risks were disclosed to the subjects, and whether Institutional Review Board (IRB) approvals (or an equivalent approval/review based on the requirements of your country or institution) were obtained?
    \item[] Answer: \answerNA{} 
    \item[] Justification: The paper does not involve crowdsourcing nor research with human subjects.
    \item[] Guidelines:
    \begin{itemize}
        \item The answer NA means that the paper does not involve crowdsourcing nor research with human subjects.
        \item Depending on the country in which research is conducted, IRB approval (or equivalent) may be required for any human subjects research. If you obtained IRB approval, you should clearly state this in the paper. 
        \item We recognize that the procedures for this may vary significantly between institutions and locations, and we expect authors to adhere to the NeurIPS Code of Ethics and the guidelines for their institution. 
        \item For initial submissions, do not include any information that would break anonymity (if applicable), such as the institution conducting the review.
    \end{itemize}

\end{enumerate}

\end{document}